\documentclass{article}

\usepackage[final]{neurips_2022}

\setcitestyle{authoryear}

\usepackage[utf8]{inputenc} 
\usepackage[T1]{fontenc}    
\usepackage{hyperref}       
\usepackage{url}            
\usepackage{booktabs}       
\usepackage{amsfonts}       
\usepackage{nicefrac}       
\usepackage{microtype}      
\usepackage{xcolor}         

\usepackage{microtype}
\usepackage{graphicx}
\usepackage{subfigure}
\usepackage{booktabs} 

\usepackage[T1]{fontenc}

\usepackage{hyperref}

\usepackage{amsmath}
\usepackage{amssymb}
\usepackage{mathtools}
\usepackage{amsthm}
\usepackage{microtype}
\usepackage{graphicx}
\usepackage{subfigure}
\usepackage{booktabs} 

\usepackage{hyperref}

\usepackage{url}
\usepackage{url}            
\usepackage{booktabs}       
\usepackage{amsfonts}       
\usepackage{nicefrac}       
\usepackage{microtype}      
\usepackage{xcolor}         

\usepackage{amsmath, amssymb, amsthm, bm}
\usepackage[capitalize,noabbrev]{cleveref}
\usepackage{multirow}
\usepackage{graphicx}
\usepackage{makecell}
\usepackage{booktabs}
\usepackage{array}
\usepackage{algorithm}
\usepackage{algorithmic}
\usepackage{dsfont}
\usepackage{csquotes}
\usepackage{cancel}
\usepackage{thmtools} 
\usepackage{xfrac} 
\usepackage{thm-restate}
\usepackage{caption}
\usepackage{nicefrac}
\usepackage[disable]{todonotes}
\usepackage{indentfirst}
\usepackage{cleveref}
\usepackage{comment}

\usepackage[shortlabels]{enumitem}
\setlist[enumerate,1]{leftmargin=0.6cm}
\setlist[itemize]{leftmargin=0.6cm}

\usepackage{amsmath,amsfonts,bm}

\def\1{\bm{1}}

\def\vzero{{\bm{0}}}
\def\vone{{\bm{1}}}

\def\vtheta{{\bm{\theta}}}
\def\vdelta{{\bm{\delta}}}

\def\vxi{{\bm{\xi}}}

\def\vb{{\bm{b}}}

\def\vm{{\bm{m}}}

\makeatletter
\newcommand{\ve}{\@ifnextchar\bgroup{\velong}{{\bm{e}}}}
\newcommand{\velong}[1]{{\bm{#1}}}
\makeatother

\def\vg{{\bm{g}}}
\def\vh{{\bm{h}}}

\def\vm{{\bm{m}}}

\def\vu{{\bm{u}}}
\def\vv{{\bm{v}}}
\def\vw{{\bm{w}}}
\def\vx{{\bm{x}}}
\def\vy{{\bm{y}}}
\def\vz{{\bm{z}}}
\def\v1{{\bm{1}}}
\def\vX{{\bm{X}}}
\def\vW{{\bm{W}}}

\def\mI{{\bm{I}}}

\def\mP{{\bm{P}}}

\def\mU{{\bm{U}}}

\def\mSigma{{\bm{\Sigma}}}

\DeclareMathAlphabet{\mathsfit}{\encodingdefault}{\sfdefault}{m}{sl}
\SetMathAlphabet{\mathsfit}{bold}{\encodingdefault}{\sfdefault}{bx}{n}

\def\gA{{\mathcal{A}}}

\def\gC{{\mathcal{C}}}

\def\gG{{\mathcal{G}}}

\def\gO{{\mathcal{O}}}

\def\gX{{\mathcal{X}}}

\def\sR{{\mathbb{R}}}

\newcommand{\E}{\mathbb{E}}
\newcommand{\R}{\mathbb{R}}

\newcommand{\Cov}{\mathrm{Cov}}

\newcommand{\normlp}{L^p}

\DeclareMathOperator{\Tr}{Tr}

\newcommand{\diag}{\mathrm{diag}}

\newcommand{\dotp}[2]{\left<#1, #2\right>}
\newcommand{\dotpsm}[2]{\langle #1, #2\rangle}

\newcommand{\hvx}{\hat{\vx}}

\newcommand{\Normal}{\mathcal{N}}

\newcommand{\DatZ}{\mathcal{Z}} 
\newcommand{\DatXi}{\Xi} 

\newcommand{\abs}[1]{\left\lvert #1 \right\rvert}
\newcommand{\abssm}[1]{\lvert #1 \rvert}

\newcommand{\normsm}[1]{\| #1 \|}

\newcommand{\normtwosm}[1]{\normsm{#1}_2}

\newcommand{\dd}{\textup{\textrm{d}}}

\newcommand{\normPsm}[1]{\normsm{#1}_{\mathrm{P}}}

\newcommand{\normRsm}[1]{\normsm{#1}_{\mathrm{R}}}

\newcommand{\hvm}{\widehat{\vm}}
\newcommand{\hvv}{\widehat{\vv}}

\newcommand{\eeta}{\eta_{\mathrm{e}}}

\newcommand{\prdX}{\mathcal{P}_{X}}

\newcommand{\msigma}{{\bm{\sigma}}}
\newcommand{\vDelta}{{\bm{\Delta}}}

\newcommand{\kaifeng}[1]{{\color{red}{}}}
\newcommand{\sadhika}[1]{{\color{red}{}}}
\newcommand{\abhishek}[1]{{\color{blue}{}}}

\newcommand{\ito}{{It\^{o}}}
\newcommand{\levy}{{L\'{e}vy}}

\theoremstyle{plain}
\newtheorem{theorem}{Theorem}[section]

\newtheorem{lemma}[theorem]{Lemma}

\theoremstyle{definition}
\newtheorem{definition}[theorem]{Definition}

\theoremstyle{remark}
\newtheorem{remark}[theorem]{Remark}

\title{On the SDEs and Scaling Rules \\ for Adaptive Gradient Algorithms}
\author{%
	Sadhika Malladi\thanks{Equal Contribution.} \qquad Kaifeng Lyu\footnotemark[1] \qquad Abhishek Panigrahi \qquad Sanjeev Arora\\
	Department of Computer Science\\
	Princeton University\\
	\texttt{\{smalladi,klyu,ap34,arora\}@cs.princeton.edu} \\
}

\begin{document}

\maketitle
\sloppy

\begin{abstract}
Approximating Stochastic Gradient Descent (SGD) as a Stochastic Differential Equation (SDE) has allowed researchers to enjoy the benefits of studying a continuous optimization trajectory while carefully preserving the  stochasticity of SGD.
Analogous study of adaptive gradient methods, such as RMSprop and Adam, has been challenging because there were no rigorously proven SDE approximations for these methods.
This paper derives the SDE approximations for RMSprop and Adam, giving theoretical guarantees of their correctness as well as experimental validation of their applicability to common large-scaling vision and language settings.
A key practical result is the derivation of a \textit{square root scaling rule} to adjust the optimization hyperparameters of RMSprop and Adam when changing batch size, and its empirical validation in deep learning settings.

\end{abstract}

\section{Introduction}
Distributed synchronous optimization environments have enabled very rapid training of models (in terms of wall-clock time) by allowing a large batch size. 
Understanding large-batch stochastic optimization is crucial to enjoying the speed-up of this setting.
In this context, \cite{krizhevsky2014one,goyal2017accurate} empirically discovered the  {\em linear scaling rule (LSR)} for Stochastic Gradient Descent (SGD). It calls for scaling learning rate proportionately to the batch size while fixing the number of epochs. 
It was recognized that the validity of this scaling rule stems from the benefits to generalization due to noise from mini-batch gradient estimation. But naive analysis, as done in~\cite{hoffer2017train}, suggested
that the scaling rule for SGD ought to be {\em square root} instead of linear.
Subsequently, \cite{jastrzkebski2017three} pointed out that since the
phenomenon involves varying the LR even down to zero, the correct analysis
should invoke a continuous view, namely a stochastic differential equation
(SDE). The SDE view helps identify the correct scaling of the noise  and leads to a  derivation of the linear scaling rule (see Section~\ref{sec:sde-approx-scaling}).


However, extending the SDE approach---i.e., continuous-time approximations---to
popular adaptive optimization algorithms, like RMSprop~\citep{hinton2012rmsprop}
and Adam~\citep{kingma2014adam}, has been challenging due to their use of
coordinate-wise normalization. By ignoring gradient noise,
\cite{ma2020qualitative} derived intuitive continuous approximations for
full-batch RMSprop and Adam; however, this deterministic view precludes a scaling
rule. 

Recent papers have suggested a {\em square root} scaling rule for adaptive
algorithms: set the learning rate proportionately to the {\em square root} of
the batch size while fixing the number of epochs. Based on perturbation theory,
\cite{granziol2021learning} proposed the square root scaling rule for RMSprop
and Adam but could only reason about optimization behavior near convergence, not
along the entire trajectory.
A square root scaling rule was also
empirically discovered for another adaptive gradient method called \textit{LAMB} \citep{you2020lamb},
which is an optimization method with layerwise adaptive learning
rates, designed for better optimization and generalization in large-batch training.
Instead of tuning learning rates while increasing batch size, LAMB used
the square root scaling rule to automatically adjust the learning rate and
achieve strong performance on vision and language tasks.




In this paper, we make the following contributions.

\begin{enumerate}
\item We propose new SDE approximations for two popular
adaptive optimization algorithms, RMSprop and Adam (\Cref{def:rmsprop_sde,def:adam_sde})  in \Cref{thm:rmsprop_sde,thm:adam_sde}. We prove that these approximations are {\em 1st-order weak approximations} (\Cref{def:weak_approx}), providing a calculus-based guarantee of the approximation strength between the SDEs and their corresponding discrete processes as was done for SGD and its SDE in \cite{li2019stochastic}. 
\item Our SDE approximations immediately yield square-root scaling rules
(\Cref{def:rmsprop_scaling,def:adam_scaling}) for adjusting the optimization hyperparameters of Adam and RMSprop when changing batch size to ensure that the resulting discrete trajectories are all 1st-order weak approximations of the same SDE (\Cref{thm:scaling_rules}). Experiments (\Cref{fig:testaccplots_cifarmain,fig:largescaleexpts_mainpaper} and \Cref{sec:app_sqsr_scaling_exps}) validate the scaling rules in the vision and language modeling domains.
\item We provide efficient experimental verification of the validity of the SDE approximation for the adaptive algorithms in realistic deep nets (\Cref{def:rmsprop_scaling,def:adam_scaling}). Direct simulation of the SDE, e.g., Euler-Maruyama,
is prohibitively expensive because it requires computing the full gradient and noise covariance at fine-grained intervals. Hence we adapt (\Cref{def:svag_alg}) the new and efficient {\em SVAG simulation} for SGD \citep{li2021validity} for use with our proposed SDEs and rigorously prove its correctness (\Cref{thm:svag}).  
Using SVAG, we provide evidence that the proposed SDE approximations track the analogous discrete trajectories in many common large-scale vision and language settings (\Cref{fig:adamsvag} and \Cref{sec:app_exp_svag}). 
\end{enumerate} 

\section{Preliminaries}
\label{sec:prelim}

We use $\vv
\odot \vu$, $\vv^2$, $\sqrt{\vv}$ to denote coordinate-wise operators for multiplication, taking squares,
taking square roots.
For ease of exposition we modify RMSprop and Adam to use $\vv_k$ in the update for $\vtheta_k$ instead of using $\vv_{k+1}$.\footnote{Experiments in \Cref{sec:app_vk} show that this change does not significantly impact performance.}
We also assume that $\vv_0$ is nonzero if $\epsilon$ is 0 to avoid division by zero.

\begin{definition}\label{def:rmsprop}
	RMSprop~\citep{hinton2012rmsprop}
	is an algorithm that updates $\vtheta_k$ as follows,
	\begin{align*}
    	\vtheta_{k+1} = \vtheta_{k} - \eta \vg_{k} \odot (\sqrt{\vv_{k}} + \epsilon)^{-1},  \quad
    	\vv_{k+1}     = \beta \vv_k + (1 - \beta) \vg_k^2, 
	\end{align*}
	where $\vtheta_k$ is the parameter, $\vg_k$ is the stochastic gradient at step $k$, and $\vv_k$ is an estimate for the second moment of $\vg_k$.
\end{definition}


\begin{definition} \label{def:adam}
	Adam~\citep{kingma2014adam}
	is an algorithm that updates $\vtheta_k$ as follows,
	\begin{align*}
		&\vm_{k+1} = \beta_1 \vm_{k} + (1-\beta_1) \vg_k, \quad
		& &\vv_{k+1} = \beta_2 \vv_k + (1-\beta_2) \vg_k^2, \\
		&\hvm_{k+1} = \vm_{k+1}\odot (1-\beta_1^{k+1})^{-1}, \quad
		& &\hvv_{k+1} = \vv_{k+1}\odot (1-\beta_2^{k+1})^{-1}, \\
		&\vtheta_{k+1} = \vtheta_k - \eta \hvm_{k+1} \odot (\sqrt{\hvv_k} + \epsilon)^{-1},
	\end{align*}
	where $\vtheta_k$ is the parameter, $\vg_k$ is the stochastic gradient at step $k$, $\vm_k$ is the momentum, and $\vv_k$ is an estimate for the second moment of $\vg_k$.
\end{definition}

\subsection{Noisy Gradient Oracle with Scale Parameter}
We abstract the stochastic gradient as being provided by a {\em noisy} oracle for the full gradient.
We formulate the oracle to highlight the phenomenon of interest: changing the batch size affects the scale of the noise.


\begin{definition} \label{def:NGOS}
A {\em Noisy Gradient Oracle with Scale Parameter} (NGOS) is
characterized by a tuple $\gG_{\sigma} = (f, \mSigma, \DatZ_{\sigma})$. Given a (noise)
scale parameter $\sigma > 0$, $\gG_{\sigma}$ takes an input $\vtheta$ and returns $\vg
=\nabla f(\vtheta) + \sigma \vz$, where $\nabla f(\vtheta)$ is the gradient of
$f$ at $\vtheta$, $\vz$ is the gradient noise drawn from the probability distribution
$\DatZ_{\sigma}(\vtheta)$ with mean zero and covariance matrix $\mSigma(\vtheta)$.
We use $\gG_{\sigma}(\vtheta)$ to denote the distribution of $\vg$ given $\sigma$ and $\vtheta$.
The probability distribution
$\DatZ_{\sigma}(\vtheta)$ can change with the scale $\sigma$, but the
covariance matrix $\mSigma(\vtheta)$ is fixed across different noise scales.
\end{definition}

For all $\vg_k$ in \Cref{def:rmsprop,def:adam}, we assume that $\vg_k$ is drawn
from a noisy gradient oracle $\gG_{\sigma}$.
In our setting, as is common when batches are sampled with replacement, $\sigma$ is primarily controlled through the batch size; in particular,
$\sigma\sim 1/\sqrt{B}$ (see \Cref{sec:app_sigma_scaling} for a derivation).
For sampling with replacement on a finite dataset of size $n$,
where $f_1(\vtheta), \dots, f_n(\vtheta)$ are the loss functions for the $n$ data points (and the average of these $n$ functions is $f(\vtheta)$),
this covariance matrix for a given parameter $\vtheta$ can be explicitly written as
\begin{equation} \label{eq:sigma-training}
	\mSigma(\vtheta) = \frac{1}{n}\sum_{i=1}^{n} (\nabla
f_{i}(\vtheta) - \nabla f(\vtheta))(\nabla f_{i}(\vtheta) - \nabla
f(\vtheta))^\top.
\end{equation}

\subsection{SDE Approximation and Scaling Rules} \label{sec:sde-approx-scaling}
Under appropriate conditions  it becomes possible to approximate SGD via an
\ito~SDE, which uses Brownian motion to model the noise and has the following general form, where $W_t$ is a Wiener
process: $\dd \vX_t = \vb(\vX_t)\dd t + \msigma(\vX_t) \dd W_t$. 
The SGD update rule for a loss $f$ is $\vx_{k+1}=\vx_k-\eta\vg_k$, where $\eta$ is the learning rate and $\vg_k$ is given by the NGOS on input $\vx_k$.
The following is the canonical interpretation of SGD as an SDE:
\begin{equation} \label{eqn:canonicalSDE}
	\dd \vX_t = -\nabla f(\vX_t)\dd t + \sqrt{\eta}\mSigma^{1/2}(\vX_t)\dd \vW_t.
\end{equation} 
\Cref{eqn:canonicalSDE} hints at a relationship between learning rate and
gradient noise---specifically, the {\em linear scaling rule}---since scaling
batch size by factor $\kappa$ scales the noise covariance by $1/\kappa$, which can be
canceled by scaling $\eta$ by $\kappa$ as well~\citep{jastrzkebski2017three}.
Therefore, the linear scaling rule ensures the SDE approximation does not change when using a different batch size.
With the same methodology, the current paper studies the SDE approximations for adaptive gradient 
algorithms to derive the square root scaling rule for them.

\subsection{Quality of SDE Approximation and Theoretical Assumptions} 
\label{subsec:SDEquality}
The quality of the SDE approximation can be measured empirically (\Cref{sec:svag}) and bounded theoretically using a calculus-based guarantee, which was initiated in the context of deep learning in \cite{li2019stochastic}. In particular, the theoretical guarantee uses the following notion of approximation between discrete and continuous stochastic processes by
comparing iteration $k$ in the discrete process with continuous time $k \eeta$, where $\eeta > 0$ is the (effective) step size of the discrete process.
\begin{definition}[Order-1 Weak Approximation, \cite{li2019stochastic}]
	Let $\{\vX_t^{\eeta}: t\in[0, T]\}$ and $\{\vx_k^{\eeta}\}_{k=0}^{\lfloor
	T/\eeta \rfloor}$ be families of continuous and discrete stochastic processes
	parametrized by $\eeta$. We say $\{\vX_t^{\eeta}\}$ and $\{\vx_k^{\eeta}\}$
	are order-1 weak approximations of each other if for every test function $g$
	with at most polynomial growth (\Cref{def:app_polygrowth_fn}), there exists
	a constant $C>0$ independent of $\eeta$ such that
	\begin{equation*}
		\max_{k=0,...,\lfloor T/\eeta\rfloor} | \E g(\vx_k^{\eeta}) - \E g(\vX_{k\eeta}^{\eeta}) | \leq C\eeta
		\label{def:weak_approx}
	\end{equation*}
\end{definition}
A function $g \colon \sR^d\to\sR$ is said to have {\em polynomial growth} if there exist positive integers $\kappa_1,\kappa_2>0$ such that $|g(\vx)| \leq \kappa_1(1+\normtwosm{\vx}^{2\kappa_2})$ for all $\vx \in \R^d$ (\Cref{def:app_polygrowth_fn}).
The above definition measures the strength of the approximation by the closeness of a test function $g$ computed on the iterates of both trajectories.
The approximation becomes stronger in this sense as $\eeta$ becomes smaller.
In the SDE approximation of SGD, $\eeta = \eta$ and $k$ steps correspond to continuous time $k \eta$.
A key difference between SGD and adaptive algorithms is that $\eeta = \eta^2$ for both RMSprop and Adam, which means
$k$ steps correspond to continuous time $k \eta^2$.
We validate this time-scaling theoretically in \Cref{sec:SDE}.

Now we formalize the assumptions needed in the theory. 
Since our analysis framework is based upon calculus, it becomes necessary to assume 
differentiability conditions on the NGOS (\Cref{as:f}).
Similar differentiability conditions also appear in prior SDE works \citep{li2019stochastic,li2021validity},
and we note that lately it has become clear that restricting to
differentiable losses (via differentiable node activations such as Swish~\citep{ramachandran2017searching}) does not hinder good performance.
\begin{definition}[Well-behaved NGOS] \label{as:f}
	The loss function $f$ and covariance matrix function $\mSigma$ in a NGOS $\gG_{\sigma}$ are {\em well-behaved} if they satisfy\footnote{Note:
	$\gC^\infty$-smoothness can be relaxed using the mollification technique
	from \cite{li2019stochastic}.}:
	(1) $\nabla f(\vtheta)$ is Lipschitz and $\gC^\infty$-smooth;
	(2) The square root of covariance matrix $\mSigma^{1/2}(\vtheta)$ is
		bounded, Lipschitz, and $\gC^\infty$-smooth; and
	(3) All partial derivatives of $\nabla f(\vtheta)$ and $\mSigma^{1/2}(\vtheta)$ up to and including the $4$-th order have polynomial growth.
	We also say that the NGOS $\gG_{\sigma}$ is well-behaved if $f$ and $\mSigma$ are well-behaved.
\end{definition}

Deriving an SDE approximation also requires conditions on the noise distribution in the NGOS. It is allowed to be
non-Gaussian, but not heavy-tailed.
We require an upper bound
on the third moment of the noise so that the distribution is not very skewed.
For other higher order moments, we require 
$\E_{\vz \sim \DatZ_{\sigma}}[\normtwosm{\vz}^p]^{1/p}$, namely the $\normlp$-norm of random variable $\normtwosm{\vz}$,
to grow at
most linearly as a function of $\vtheta$ (which is implied by ensuring an upper bound on all even
order moments).
We note that the following conditions are standard in prior work using {\ito} SDEs to study SGD.
\begin{definition}[Low Skewness Condition] \label{def:low-skew}
	The NGOS $\gG_{\sigma}$ is said to satisfy the {\em low skewness condition} if  
	there is a function $K_3(\vtheta)$ of polynomial growth (independent of $\sigma$) such that
	$\abssm{\E_{\vz \sim
	\DatZ_{\sigma}(\vtheta)}[\vz^{\otimes 3}]} \le K_3(\vtheta) / \sigma$
	for all $\vtheta \in \R^d$ and all noise scale parameters $\sigma$.
\end{definition}
\begin{definition}[Bounded Moments Condition]\label{def:bound-moment}
	The NGOS $\gG_{\sigma}$ is said to satisfy the {\em bounded moments condition} if  
	for all integers $m \ge 1$ and all noise scale parameters $\sigma$,
	there exists a constant $C_{2m}$ (independent of $\sigma$) such that
	$\E_{\vz \sim
	\DatZ_{\sigma}(\vtheta)}[\normtwosm{\vz}^{2m}]^{\frac{1}{2m}} \le
	C_{2m}(1 + \normtwosm{\vtheta})$ for all $\vtheta \in \R^d$.
\end{definition}
For well-behaved loss $f(\vtheta)$ and covariance $\mSigma(\vtheta)$,
the above two conditions are satisfied when
$\DatZ_{\sigma}$ is the Gaussian distribution with mean zero and covariance $\mSigma(\vtheta)$. That is,
the Gaussian NGOS $\vg \sim \Normal(\nabla f(\vtheta), \sigma^2 \mSigma(\vtheta))$ satisfies
the low skewness and bounded moments conditions.
The low skewness condition holds because the odd moments of a Gaussian are all zeros,
and the bounded moments condition can be verified since 
$\E_{\vz \sim \Normal(\vzero, \mSigma(\vtheta))}[\normtwosm{\vz}^{2m}]^{\frac{1}{2m}} \le \E_{\vw \sim \Normal(\vzero, \mI)}[\normtwosm{\vw}^{2m}]^{\frac{1}{2m}} \cdot \normtwosm{\mSigma^{1/2}(\vtheta)}$
and $\mSigma^{1/2}(\vtheta)$ is Lipschitz.

The Gaussian NGOS with $\sigma = \frac{1}{\sqrt{B}}$ can be seen as an approximation of the gradient noise in a mini-batch
training with batch size $B$, if $\mSigma(\vtheta)$ is set to match with \eqref{eq:sigma-training}.
But this does not directly imply that the gradient noise in mini-batch training satisfies the low skewness and bounded moments conditions,
as the noise is not exactly Gaussian.
Assuming that the gradient of the loss function $f_i(\vtheta)$
at every data point is Lipschitz, these two conditions can indeed be
verified for all batch sizes $B \ge 1$. 

\subsection{Discussion on Heavy-Tailed Gradient Noise}
We note that \Cref{def:low-skew,def:bound-moment} allow some non-Gaussianity in the noise, but 
$K_3(\vtheta)$ and $C_{2m}$ 
could be large in practice.
In this case, higher order moments of the gradient noise have non-negligible effects on training that the \ito~SDE cannot capture.
\citet{zhang2020adaptive} and \citet{simsekli2019tailindex} presented experimental
evidence that the noise is heavy-tailed.
This motivated
\citet{zhou2020theoretically} to consider a \levy~SDE approximation (instead of
\ito~SDE) to study the (worse) generalization behavior of Adam. However, the
quality of the \levy~SDE approximation was not formally guaranteed (e.g., in the
sense of \Cref{def:weak_approx}), so finding a guaranteed approximation for
adaptive optimization algorithms remains an open problem. Moreover,
\citet{li2021validity,xie2021diffusion} highlighted issues with the evidence,
noting that  the measurements used in \citet{simsekli2019tailindex} are intended
only for scalar values. When applied to vector valued distributions the
measurement can (incorrectly) identify a multidimensional Gaussian distribution
as heavy-tailed too \citep{li2021validity}. 
It is in general difficult to estimate the moments of the noise distribution from samples, so the heavy-tailedness of
real-world gradient noise remains an open question. 

Our empirical results suggest that our assumptions in
\Cref{def:low-skew,def:bound-moment} are not too strong.
In \Cref{sec:svag},
we efficiently simulate
the \ito~SDE using an algorithm analogous to SVAG~\citep{li2021validity}.
The simulation closely approximates the test accuracy
achieved by the discrete trajectory, suggesting that even if heavy-tailed noise
is present during training, it is not crucial for good generalization
(\Cref{sec:app_exp_svag}).
We remain interested in exploring the heavy-tailed analogs of our \ito~SDEs. 
However, efficient simulation of such SDEs remains intractable and formal approximation guarantees are difficult to prove, so we are limited in assessing the utility of such approximations. 
We leave it for future work to
investigate if and how heavy-tailed noise plays a role in adaptive optimization.

\section{Related Work}\label{subsec:related}
We defer a full discussion of empirical and theoretical works on adaptive gradient methods to \Cref{sec:app_addl_works} and only discuss works relevant to SDEs and scaling rules here.

\paragraph{Applications of SDE Approximations.}
There are applications of the SDE approximation beyond the derivation of a scaling rule.
\cite{li2020reconciling} and \cite{kunin2021neural} assumed that the loss has some symmetry (i.e., scale invariance) and studied the resulting dynamics. Furthermore, \cite{li2020reconciling} used this property to explain the phenomenon of sudden rising error after LR decay in training.
\cite{xie2021diffusion} analyzed why SGD favors flat minima using an SDE-motivated diffusion model.

\paragraph{Past Work on Square Root Scaling Rule.}
As mentioned before, square root scaling was incorrectly believed for a few years to be theoretically justified for SGD.
\citet{granziol2021learning} decomposed the stochastic Hessian during batch training into a deterministic Hessian and stochastic sampling perturbation and assumes one of the components to be low rank to propose a square root scaling rule.
\citep{you2020lamb} empirically discovered a square root scaling rule for language models trained by LAMB, a layer-wise variant of Adam.
\citet{xie2022adai} heuristically derived, but did not show an approximation bound for, a second-order SDE for approximating Adam,
and they applied the SDE to study the time needed for Adam to escape sharp minima.
\citet{xie2022adai} did not discuss a scaling rule, though their proposed SDE may admit one.
Similarly, \citet{zhou2020theoretically} derived a \levy~SDE for Adam, but no approximation bounds are given in the paper.


\section{SDEs for Adaptive Algorithms}
\label{sec:SDE}



An SDE approximation operates in continuous time and thus implicitly considers the limit $\eta
\rightarrow 0$. 
In adaptive algorithms, the moment averaging parameters $\beta,\beta_1,\beta_2$ and $\eta$ must be taken to limits such that the adaptivity and stochasticity can still be studied.
For example, if $\beta, \beta_1, \beta_2$ remain fixed when $\eta\to 0$, then
the algorithm computes the moving averages in a very small neighborhood, which
averages out the effects of gradient noise and gradient history, causing the
flow to turn into deterministic SignGD \citep{ma2020qualitative}. 
We will need to assume $\beta,\beta_1,\beta_2\to 1$ as $\eta\to 0$, which implies the
impact of the history grows as the learning rate decreases, and thus the
adaptive features of these algorithms can still be studied in the continuous
approximation~\citep{ma2020qualitative}. To keep the stochastic nature
of the flow, we require the noise from mini-batching dominate the gradient updates.





\subsection{Warm-Up: Linear loss}\label{sec:linear_warmup} 

To build intuition for the SDE and the scaling rule, we first study a simplified setting.
In particular, consider a linear
loss $f(\vtheta)= \dotpsm{\vtheta}{\bar\vg}$ and isotropic noise in the NGOS, namely
$\vg_k \sim {\mathcal N}(\bar{\vg}, \sigma^2 \mI)$.
The RMSprop update $\vv_{k+1} = \beta \vv_k + (1-\beta) \vg_k^2$ can be expanded as
	$\vv_k = \beta^k \vv_0 + (1-\beta)\sum_{j=0}^{k-1} \beta^j \vg_j^2.$ 
By linearity of expectation, 
$$
	\E[\vv_k] = \beta^k \vv_0 + (1-\beta) \sum_{j=0}^{k-1} \beta^j (\bar\vg^2 + \sigma^2 \vone) = \beta^k \vv_0 + (1-\beta^k)(\bar\vg^2 + \sigma^2 \vone).
$$
This suggests that $\E[\vv_k]$ is approximately $\bar{\vg}^2 + \sigma^2 \vone$ after a sufficient number of steps. 
Setting $\vv_0 = \bar{\vg}^2 + \sigma^2 \vone$, we see that the approximation $\E[\vv_k] = \bar{\vg}^2 + \sigma^2 \vone$ becomes exact for all $k \ge 0$.

Using the linearity of variance (for independent variables), the standard deviation of each coordinate of $\vv_k$ is of scale
$\gO((1-\beta)\sigma^2)$. Moreover, the expectation $\E[\vv_k]$ is of scale $\gO(\sigma^2)$, so
we know that $\vv_k$ is nearly
deterministic and concentrates around its expectation $\bar{\vg}^2 + \sigma^2\vone$ when $\beta$ is close to $1$.
Therefore, we take the approximation $\vv_k \approx \bar{\vg}^2 + \sigma^2 \vone$ for all $k \ge 0$.
Ignoring $\epsilon$, the RMSprop update rule becomes:
\begin{align}
	\vtheta_{k+1} &\approx \vtheta_k - \eta \vg_k \odot (\bar{\vg}^2 + \sigma^2 \vone)^{-1/2}. \label{eq:warmup-a}
\end{align}
These dynamics depend on the relative magnitudes of $\bar\vg$ and $\sigma$.
We show that when $\sigma\ll\normsm{\bar\vg}$, no SDE approximation can exist in \Cref{sec:app_grad_dominates_noise}.
Here, we explore the case where $\sigma \gg \normsm{\bar\vg}$ which implies $\vtheta_{k+1} \approx \vtheta_k - \frac{\eta}{\sigma} \vg_k$.
Noting that $\vg_k \sim \Normal(\bar{\vg}, \sigma^2 \mI)$ gives
$\vtheta_{k+1} - \vtheta_k \sim \Normal( \frac{\eta}{\sigma} \bar{\vg}, \eta^2 \mI)$ approximately.
Since Gaussian variables are additive, we can take a telescoping sum to obtain the marginal distribution of $\vtheta_k$:
	$\vtheta_k \sim \Normal\left((k\eta/\sigma) \bar\vg, k\eta^2\mI \right)$ approximately.

If an SDE approximation of RMSprop exists, then $\vtheta_k$ can be closely approximated by a {\em fixed} random variable from the corresponding  stochastic process at a fixed (continuous) time $t$.  
Thus, to keep the SDE unchanged upon changing the noise scale $\sigma$, the hyperparameters must be adjusted to keep $\frac{k \eta}{\sigma}$ and $k \eta^2$ unchanged, which implies
 $\eta \sim \frac{1}{\sigma}$ and $k \sim \frac{1}{\eta^2}$. 
 These observations intuitively yield the square root scaling rule: noting that $\sigma$ changes with mini-batch size $B$ as $\sigma \sim 1/\sqrt{B}$ suggests $\eta \sim \sqrt{B}$, and $k \sim 1/B$.

\subsection{SDE Approximations for Adaptive Algorithms}
Having established intuition in a highly simplified setting for adaptive algorithms,
we now return to a more general and realistic setting. We derive the SDEs that are order-1 approximations of the discrete RMSprop and
Adam algorithms under \Cref{as:f},
where the SDE states consist of both the parameters $\vtheta$ and moment estimates.
From the example of \Cref{sec:linear_warmup}, we see that
SDE approximations may exist if $\sigma \sim 1/\eta^2$ and $\sigma \gg \normsm{\bar\vg}$ (see \Cref{sec:app_noise_dominates} for empirical validation of this assumption).
In this case, we can prove that $k \sim \eta^2$ is true not only for the simple setting above but also in general.
This is a key difference to the SDE for SGD ---
each step in RMSprop or Adam corresponds to a time interval of $\eta^2$
in SDEs, but each SGD step corresponds to a time interval of $\eta$.
In \Cref{sec:linear_warmup}, $\vv\sim\sigma^2 \sim 1/\eta^2$ grows to infinity as $\eta \to 0$.
This also happens in the general setting, so we track $\vu_k \triangleq \vv_k/\sigma^2$ (instead of $\vv_k$ directly) in the SDEs. 
\begin{definition}[SDE for RMSprop]\label{def:rmsprop_sde}
	Let $\sigma_0 \triangleq \sigma\eta$, $\epsilon_0\triangleq \epsilon\eta$,
	and $c_2 \triangleq (1-\beta)/\eta^2$. Define the state of the SDE as $\vX_t =
	(\vtheta_t, \vu_t)$ and the dynamics as
	\begin{align*}
		\dd\vtheta_t = - \mP_t^{-1} \left(\nabla f(\vtheta_t) \dd t + \sigma_0 \mSigma^{1/2}(\vtheta_t) \dd \vW_t\right), \qquad
		\dd\vu_t = c_2(\text{diag}(\mSigma(\vtheta_t))-\vu_t)\dd t
    \end{align*}
    where $\mP_t := \sigma_0\text{diag}(\vu_t)^{1/2} + \epsilon_0 \mI$ is a preconditioner matrix, and $\vW_t$ is the Wiener process.
\end{definition}

\begin{theorem}[Informal version of \Cref{thm:app_rmsprop_sde}]
	Let $\vu_k\triangleq\vv_k/\sigma^2$ and define the state of the discrete RMSprop trajectory with hyperparameters $\eta,\beta,\epsilon$ (\Cref{def:rmsprop}) as $\vx_k = (\vtheta_k, \vu_k)$.	
	Then, for a well-behaved NGOS (\Cref{def:NGOS}) satisfying the skewness and bounded moments conditions (\Cref{def:low-skew,def:bound-moment}), the SDE in \Cref{def:rmsprop_sde} satisfies
    \[
    	\max_{k=0,...,\lfloor T/\eta^2\rfloor} | \E g(\vx_k) - \E g(\vX_{k\eta^2}) | \leq C\eta^2
    \]
    where $g$ and $T$ are defined as in \Cref{def:weak_approx} and the initial condition of the SDE is $\vX_0 = \vx_0$.
    \label{thm:rmsprop_sde}
\end{theorem}


\begin{remark}
\Cref{sec:linear_warmup} suggested that the SDE approximation can only exist when $\sigma\gg\|\bar\vg\|$. This condition is reflected by keeping  $\sigma_0=\sigma\eta$ a constant and $C$ depends on $\sigma_0$. When $\eta\to 0$, $\sigma$ scales as $1/\eta$, so $\sigma\gg\|\vg\|_2$. 
\end{remark}

We need to find continuous approximations of the normalization constants in Adam (\Cref{def:adam}).
As in the RMSprop case, we take $(1-\beta_2)/\eta^2 = c_2$.
Then, we can estimate the normalization constant $1-\beta_2^k$ in continuous time $t = k\eta^2$ as 
$1-\beta_2^k = 1-(1-c_2\eta^2)^{t/\eta^2}\approx 1-\exp(-c_2t)$.
We can do the analogous approximation for the other normalization constant $1 - \beta_1^k$ in Adam. Taking $(1-\beta_1) / \eta^2 = c_1$, we can approximate it as $1 - \beta_1^k \approx 1 - \exp(-c_1 t)$.
This is a heuristic approach to deal with the normalization constants, but we can indeed justify it in theory. 
\begin{definition}[Adam SDE]\label{def:adam_sde}
	Let $c_1 \triangleq (1-\beta_1)/\eta^2, c_2 \triangleq (1-\beta_2)/\eta^2$ and define $\sigma_0, \epsilon_0$ as in \Cref{def:rmsprop_sde}.
	Let $\gamma_1(t)\triangleq 1-\exp(-c_1 t)$ and $\gamma_2(t)\triangleq 1-\exp(-c_2 t)$.
	Define the state of the SDE as $\vX_t = (\vtheta_t, \vm_t, \vu_t)$ and the dynamics  as
	\begin{align*}
		\dd\vtheta_t &= -\tfrac{\sqrt{\gamma_2(t)}}{\gamma_1(t)} \mP_t^{-1} \vm_t \dd t, &\qquad
		\dd\vm_t &= c_1(\nabla f(\vtheta_t) - \vm_t)\dd t + \sigma_0c_1 \mSigma^{1/2}(\vtheta_t) \dd W_t, \\
		& & \dd\vu_t &= c_2(\text{diag}(\mSigma(\vtheta_t)) - \vu_t) \dd t,
    \end{align*}
    where $\mP_t := \sigma_0 \diag(\vu_t)^{1/2} + \epsilon_0 \sqrt{\gamma_2(t)} \mI$ is a preconditioner matrix, 
	$\vW_t$ is the Wiener process.
\end{definition}
Our main theorem for Adam is given below.
The initial steps of the discrete Adam trajectory can be discontinuous and noisy because of the normalization constants changing drastically.
Hence, we introduce a constant $t_0$ and 
show that for any $t_0$,
we can construct an SDE 
to be a weak approximation for Adam starting from the $\lceil t_0 / \eta^2 \rceil$-th step of Adam.
\begin{theorem}[Informal version of \Cref{thm:app_adam_sde}]\label{thm:adam_sde}
	Define $\vu_k = \vv_k/\sigma^2$ and let $\vx_k = (\vtheta_k,\vm_k,\vu_k)\in\sR^{3d}$ be the state of the discrete Adam trajectory with hyperparameters $\eta,\beta_1,\beta_2,\epsilon$.
	Then, for a well-behaved NGOS (\Cref{def:NGOS}) satisfying the skewness and  bounded moments conditions (\Cref{def:low-skew,def:bound-moment}) and any $t_0 > 0$, the SDE in \Cref{def:adam_sde} satisfies
    \[
    	\max_{k=\lceil t_0/\eta^2\rceil,...,\lfloor T/\eta^2\rfloor} | \E g(\vx_k) - \E g(\vX_{k\eta^2}) | \leq C\eta^2
    \]
    where $g$ and $T$ are defined as in \Cref{def:weak_approx} and the initial condition of the SDE is $\vX_{t_0}=\vx_{\lceil t_0/\eta^2\rceil}$. 
\end{theorem}

\begin{proof}[Proof Sketch]
We provide a proof sketch for our SDE approximations here and defer the technical details to \Cref{thm:app_rmsprop_sde,thm:app_adam_sde}.
The proof follows the same steps as \cite{li2019stochastic}: first, we compute the approximation error of the continuous trajectory after one step in discrete time.
Then, we use the single step error to bound the error over a finite interval of time. 
The proof extends standard SDE techniques in several ways.
The given SDEs do not satisfy the Lipschitzness and smoothness conditions because the denominator can be unbounded. We thus construct an auxiliary SDE with an equivalent distribution to the desired SDE (\Cref{thm:app_rmsprop_aux_sde}) but with better smoothness conditions,
and we prove this SDE to be an order-1 weak approximation of the discrete trajectory.
Moreover, the SDE coefficients are time-dependent for Adam, unlike the ones for SGD, so we need to extend existing results to cover this case (see~\Cref{sec:app_time_dependent}).
\end{proof}

\section{Square Root Scaling Rule}

The SDE approximations in \Cref{def:rmsprop_sde,def:adam_sde} motivate scaling rules for how to adjust the optimization hyperparameters when changing the batch size.
In order for $\sigma_0, c_1, c_2,$ and $\epsilon_0$ to remain constant, as required by the SDEs, one needs to change $\eta,\beta,\beta_1,\beta_2$, and $\epsilon$ accordingly.
\begin{definition}[RMSprop Scaling Rule]\label{def:rmsprop_scaling}
	When running RMSprop (\Cref{def:rmsprop}) with batch size $B' = \kappa B$, use the hyperparameters $\eta' = \eta\sqrt{\kappa}$, $\beta' = 1 - \kappa(1-\beta)$, and $\epsilon' = \frac{\epsilon}{\sqrt{\kappa}}$. 
\end{definition}
\begin{definition}[Adam Scaling Rule]\label{def:adam_scaling}
	When running Adam (\Cref{def:adam}) with batch size $B' = \kappa B$, use the hyperparameters $\eta' = \eta\sqrt{\kappa}$, $\beta_1' = 1 - \kappa(1-\beta_1)$, $\beta_2' = 1 - \kappa(1-\beta_2)$, and $\epsilon' = \frac{\epsilon}{\sqrt{\kappa}}$.
\end{definition}

\begin{theorem}[Validity of the Scaling Rules]\label{thm:scaling_rules}
	Suppose we have a well-behaved NGOS satisfying the low skewness and bounded moments conditions.
	\begin{enumerate}[leftmargin=*]
	    \item Let $\vx_k^{(B)}$ be the discrete RMSprop (\Cref{def:rmsprop})
	    trajectory with batch size $B$ and hyperparameters $\eta,\beta,$ and
	    $\epsilon$. Let $\vx_k^{(\kappa B)}$ be the trajectory with batch size
	    $\kappa B$ and hyperparameters adjusted according to
	    \Cref{def:rmsprop_scaling}.
		If $\vx_k^{(B)}$ and $\vx_k^{(\kappa B)}$ start from the same initial
	    point, then with $g$ and $T$ defined as in \Cref{def:weak_approx},
	    \[
		    \max_{k=0,...,\lfloor T/\eta^2\rfloor} \abs{ \E g(\vx_k^{(B)}) - \E g(\vx_{\lfloor k/\kappa\rfloor}^{(\kappa B)}) } \leq C(1+\kappa)\eta^2.
	    \]
	    \item Fix $t_0 > 0$. Let $\vx_k^{(B)}$ be the discrete Adam (\Cref{def:adam})
	    trajectory with batch size $B$ and hyperparameters
	    $\eta,\beta_1,\beta_2,$ and $\epsilon$. Let $\vx_k^{(\kappa B)}$ be the
	    trajectory with batch size $\kappa B$ and hyperparameters adjusted
	    according to \Cref{def:adam_scaling}. 
		If $\vx_{\lceil t_0 / \eta^2 \rceil}^{(B)}$ and $\vx_{\lceil \kappa t_0 / \eta^2 \rceil}^{(\kappa B)}$ are equal,
		then with $g$ and $T$ defined as in \Cref{def:weak_approx},
	\[
		\max_{k=\lceil t_0/\eta^2\rceil,...,\lfloor T/\eta^2\rfloor} | \E g(\vx_k^{(B)}) - \E g(\vx_{\lfloor k/\kappa\rfloor}^{(\kappa B)}) | \leq C(1+\kappa)\eta^2.
    \]
	\end{enumerate}
\end{theorem}
\begin{proof}
	By the linearity of covariance, scaling the batch size by $\kappa$ only
	modifies the NGOS by scaling $\sigma$ by $1/\sqrt{\kappa}$. Hence, both
	scaling rules ensure that $\sigma_0, c_1, c_2$, and $\epsilon_0$ (and thus, the SDEs) are
	unchanged when the batch size changes. The approximation bounds in
	\Cref{thm:rmsprop_sde,thm:adam_sde} are in terms of $\eta^2$, and since
	$\eta$ is scaled here by $\sqrt{\kappa}$, the same method gives an upper bound $C \kappa \eta^2$.
	Adding the approximation bounds for $\eta$ and $\sqrt{\kappa} \eta$ together gives $C(1 + \kappa)\eta^2$.
\end{proof}
\begin{remark}
The $t_0$ condition on the Adam rule, a holdover from the condition in \Cref{thm:adam_sde},
implies that our theory only directly applies when there is a warm-start phase of 
$\lceil t_0 / \eta^2 \rceil$,
where the marginal distribution of the trainable parameters at the end of the phase is the same across different learning rates $\eta$.
Regardless, the scaling rules are shown to work in practice even without this phase.
\end{remark}

The scaling rules depend on maintaining the same SDE approximation, so the bounded moments and low skewness
conditions are sufficient (but not necessary) for this scaling rule to work. \cite{li2021validity} provided an analogous discussion for SGD, and they show the scaling rule can hold even
if there is heavy-tailed noise. We leave a study of heavy-tailed gradient noise in adaptive algorithms as future work.
\begin{figure}[t]
     \centering
     \begin{minipage}[t]{0.24\textwidth}
         \centering\includegraphics[width=1.15\linewidth]{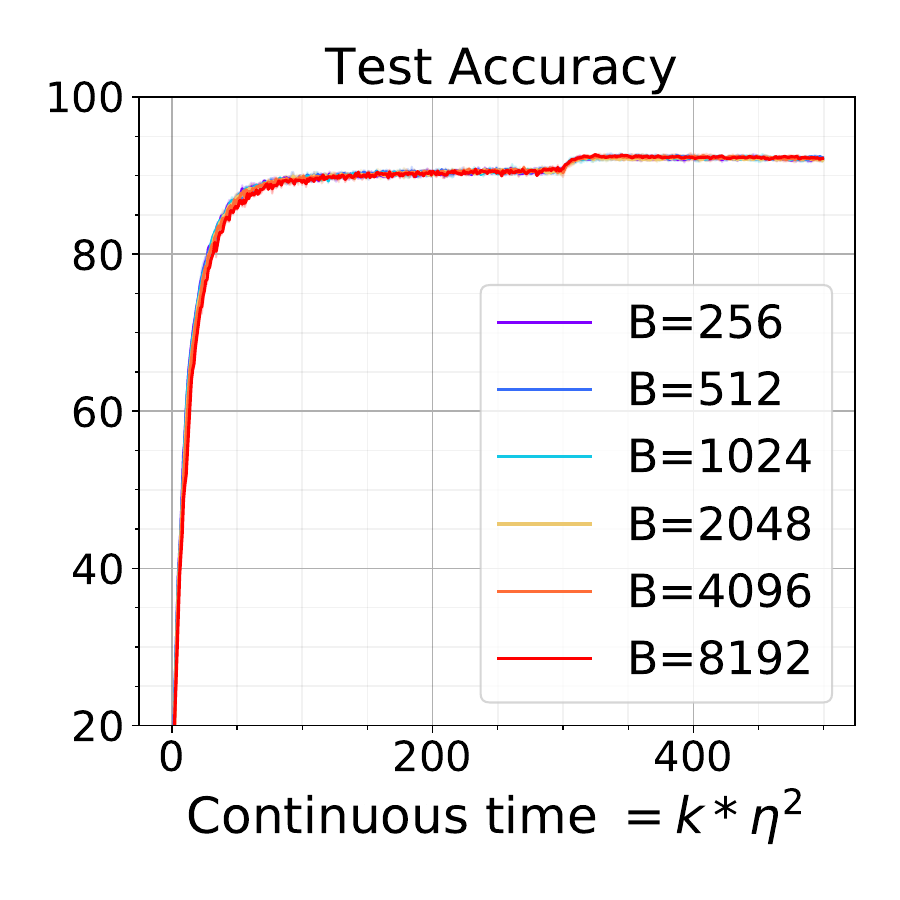} \\
        (a) VGG-16 with Adam
     \end{minipage}
     \hfill
     \begin{minipage}[t]{0.24\textwidth}
         \centering\includegraphics[width=1.15\linewidth]{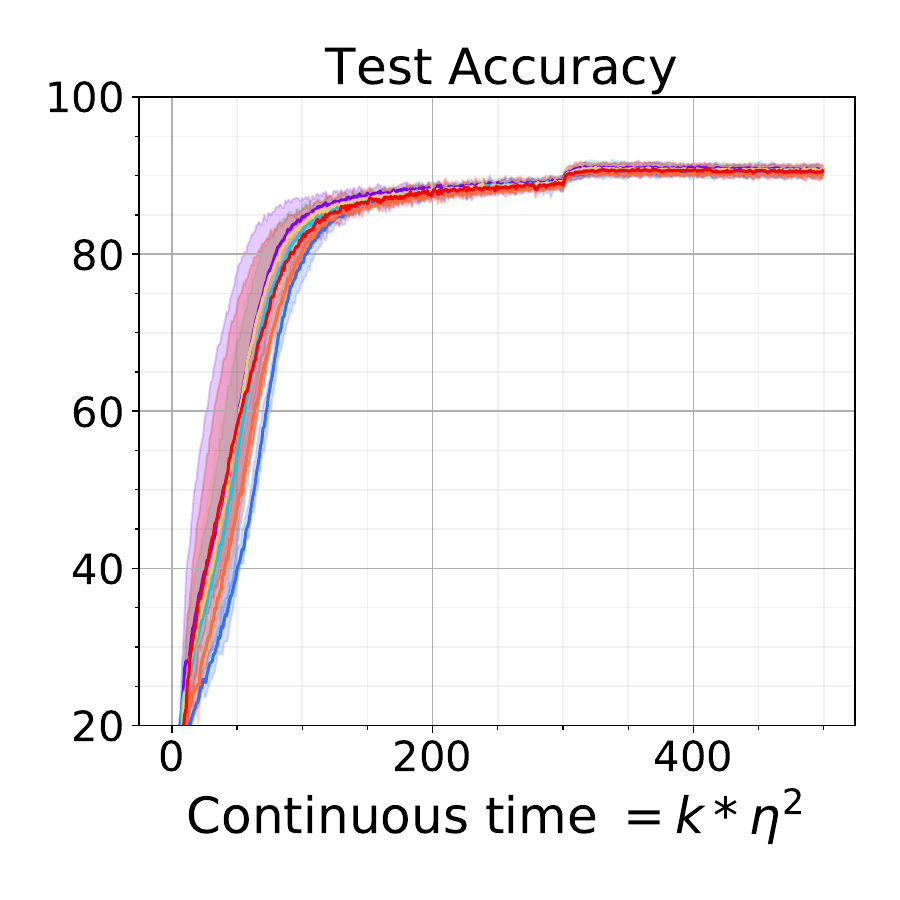} \\
        (b) ResNet-50 with Adam
         \end{minipage}
         \hfill
     \begin{minipage}[t]{0.24\textwidth}
         \centering\includegraphics[width=1.15\linewidth]{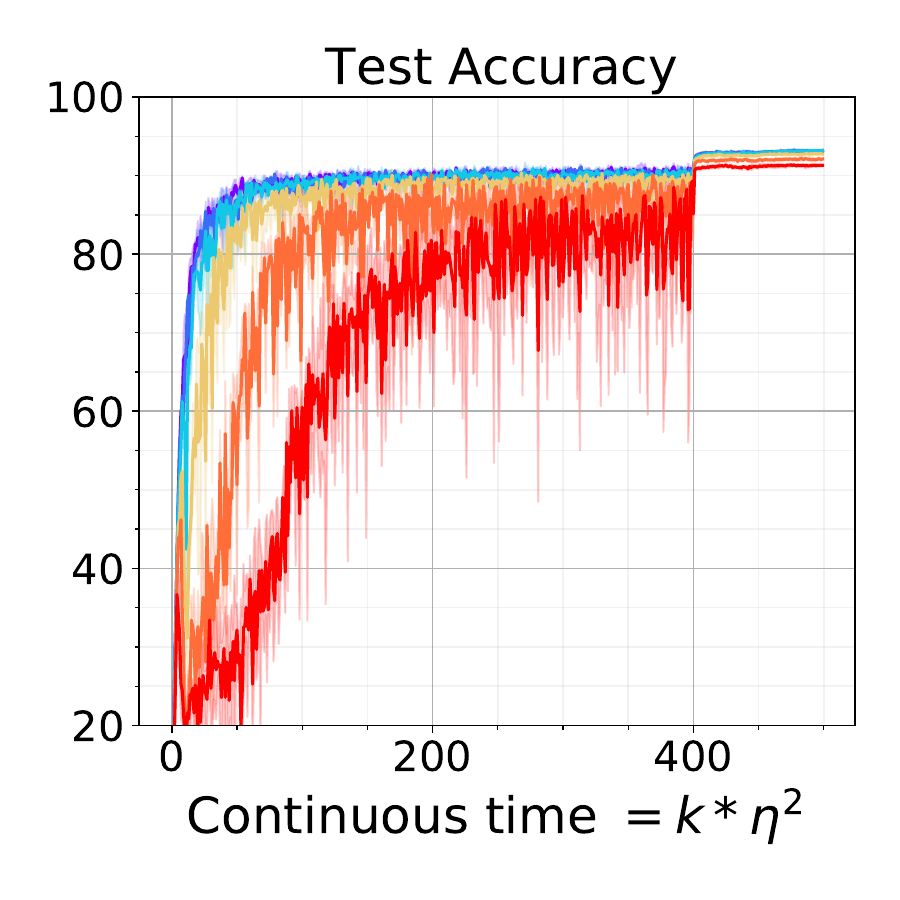}\\
        (c) VGG-16 with RMSprop
     \end{minipage}
     \hfill
     \begin{minipage}[t]{0.24\textwidth}
         \centering\includegraphics[width=1.15\linewidth]{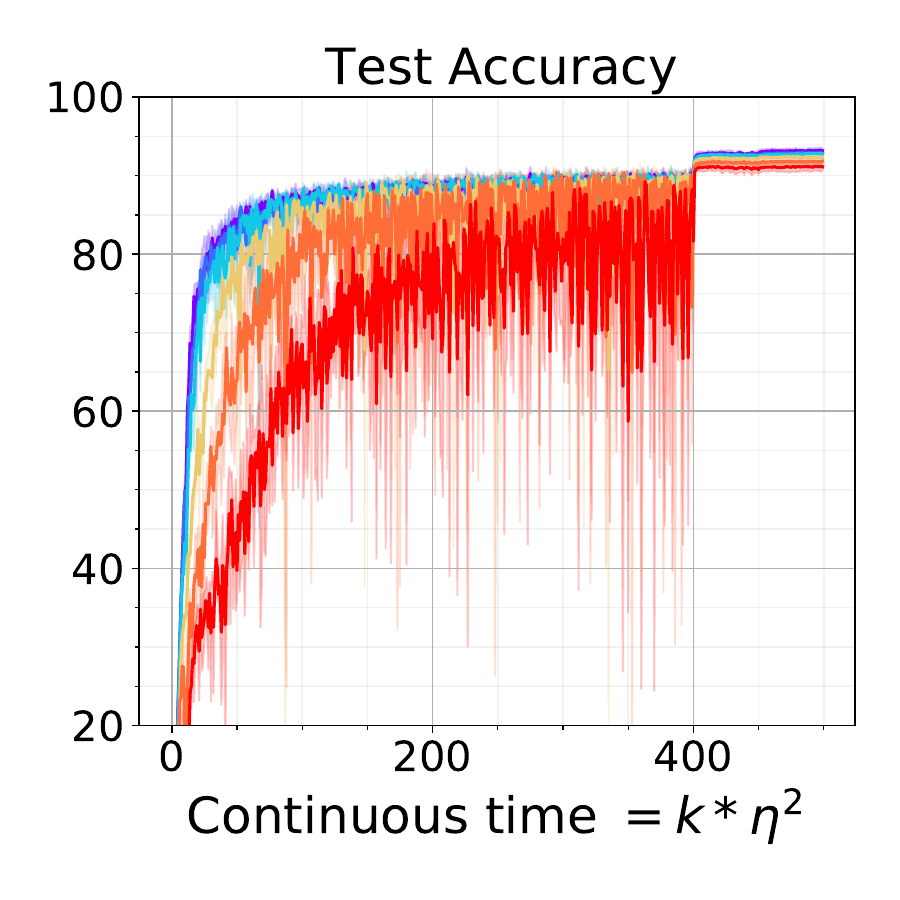} \\
        (d) ResNet-50 with RMSprop
     \end{minipage}
  \caption{Square root scaling rule experiments on CIFAR-10 with VGG-16 and ResNet-50 (details in \cref{sec:app_exp_config}). We plot the mean and variance of 3 random seeds. Same color legend has been used across all the plots. The performance gap between $B=256$ and $B=8192$ is at most $3\%$ in all cases.}
  \label{fig:testaccplots_cifarmain}
\end{figure}

\begin{figure}[t]
     \centering
     \begin{minipage}[t]{0.32\textwidth}
         \centering\includegraphics[width=\linewidth]{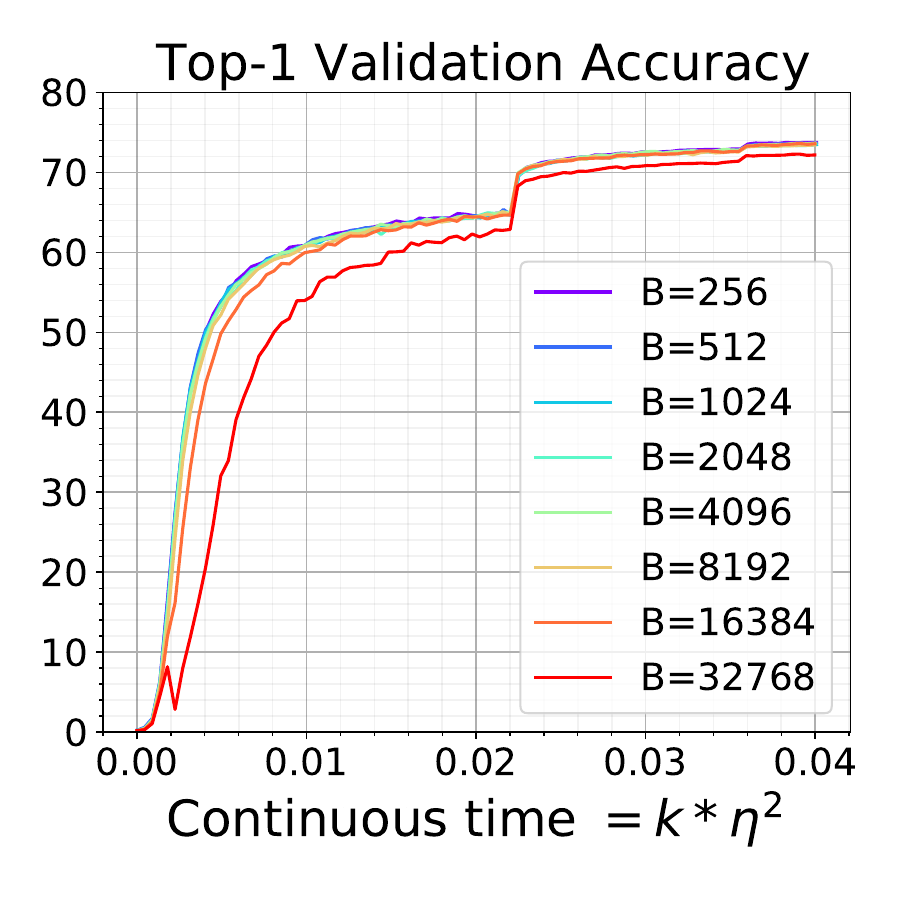}\\
        (a) ResNet-50 on ImageNet
     \end{minipage}\hfill
     \begin{minipage}[t]{0.32\textwidth}
         \centering\includegraphics[width=\linewidth]{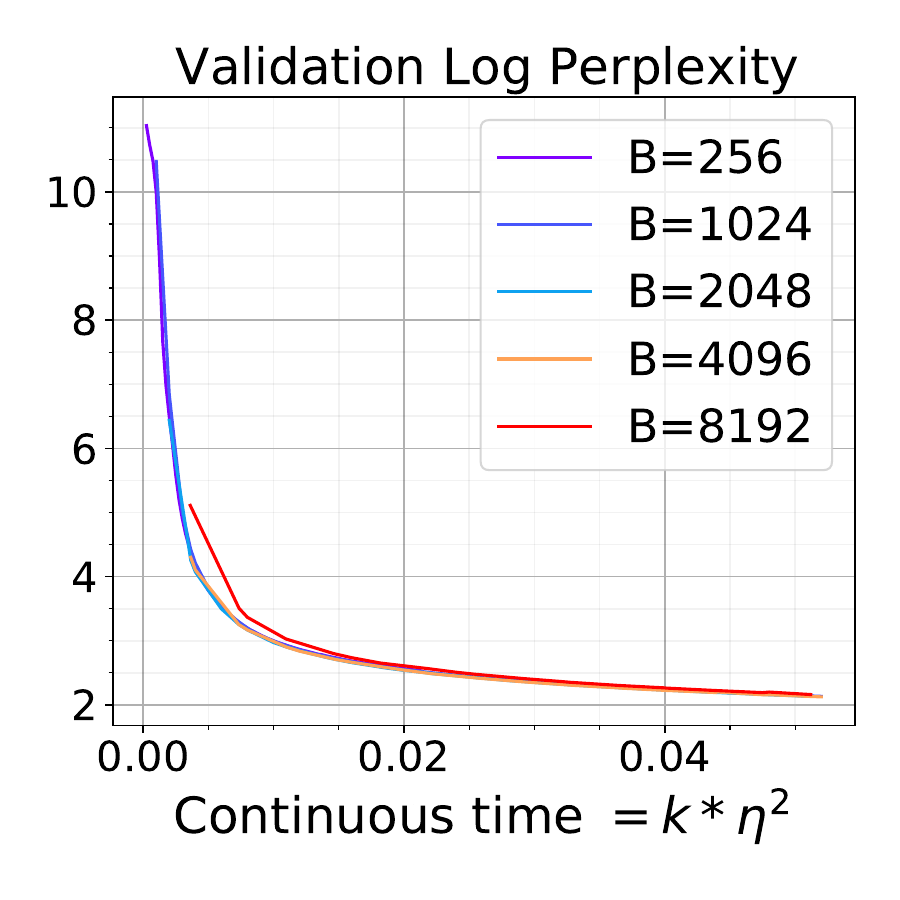} \\
         (b) RoBERTa on Wikipedia + Bookcorpus
     \end{minipage}\hfill
     \begin{minipage}[t]{0.32\textwidth}
         \centering\includegraphics[width=\linewidth]{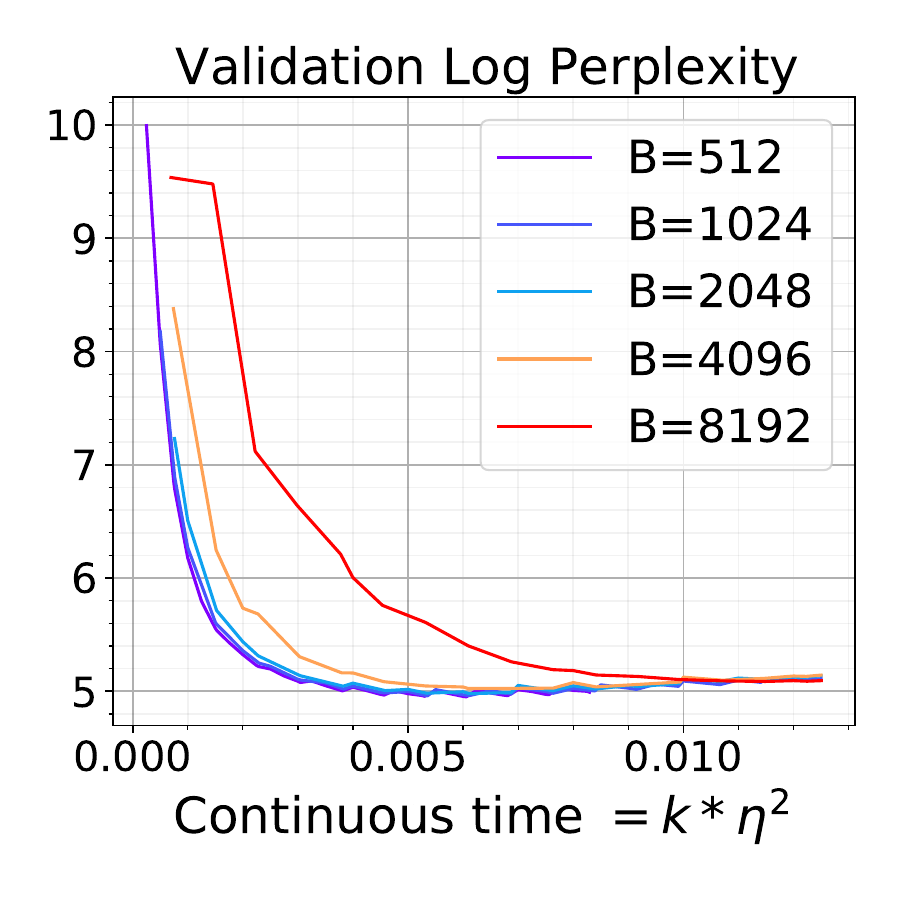} \\
        (c) GPT on WikiText-103
     \end{minipage}
  \caption{Large scale square root scaling rule experiments (details in \cref{sec:app_exp_config}). Small and large batch models differ by at most $1.5\%$ test accuracy in vision and 0.5 perplexity in language.}
  \label{fig:largescaleexpts_mainpaper}
\end{figure}

\paragraph{Experiments.} \Cref{fig:testaccplots_cifarmain,fig:largescaleexpts_mainpaper} show the square root scaling rule applied to ResNet-50~\citep{he2016deep} and VGG-16~\citep{simonyan2014very} trained on CIFAR-10~\citep{cifar10}, RoBERTa-large~\citep{liu2019roberta} trained on the Wiki+Books corpus~\citep{zhu2015aligning}, 12-layer GPT~\citep{brown2020language} on WikiText-103~\citep{merity2017pointer}
and ResNet-50 trained on ImageNet~\citep{deng2009imagenet}. We use the efficient language model pre-training recipe outlined in
\cite{izsak2021how}. \cref{sec:app_sqsr_scaling_exps} has many additional settings, including ablations against other scaling rules (\Cref{sec:app_sqsr_ablation}).

\section{SVAG for Adaptive Algorithms}\label{sec:svag}
Validating the approximation strength captured in \Cref{def:weak_approx} involves comparing the discrete algorithms and their SDEs on a set of test functions.
However, obtaining the SDE solution
through traditional simulations, e.g., Euler-Maruyama, is computationally
intractable.\footnote{One can also simulate the SDE by constructing 1st-order weak approximations while taking $\eta\to 0$ along the scaling rules, but the batch size cannot be smaller than 1 and hence $\eta$ cannot go arbitrarily close to the limit.}




\cite{li2021validity} proposed an efficient simulation, SVAG, of the SDE for SGD
in the finite LR regime: scale the constant LR by $1/\ell$ and take the
limit $\ell\to\infty$. In practice the simulation converges for a small value
of $\ell$.
We adapt SVAG technique to simulate our proposed SDEs, which requires additionally adjusting the moment averaging hyperparameters (i.e., $\beta$, $\beta_1$, $\beta_2$) and $\epsilon$.
\begin{definition}[SVAG Operator]\label{def:svag_op}
	Given an NGOS $\gG_{\sigma} = (f, \mSigma, \DatZ_\sigma)$ with scale $\sigma$ (\Cref{def:NGOS}) and hyperparameter $\ell \ge 1$,
	the SVAG operator transforms $\gG_{\sigma}$ into an NGOS $\widehat{\gG}_{\ell \sigma} = (f, \mSigma,
	\widehat{\DatZ}_{\ell \sigma})$ with scale $\ell \sigma$.
	The NGOS $\widehat{\gG}_{\ell \sigma}$ takes an input $\vtheta$ and returns $\hat{\vg} = r_1(\ell) \vg_1 + r_2(\ell)
	\vg_2$, where $\vg_1, \vg_2$ are two stochastic gradients from $\gG_{\sigma}(\vtheta)$ and 
		$r_i(\ell) =\frac{1}{2}(1+(-1)^i\sqrt{2\ell^2-1})$ for $i \in \{1, 2\}$.
	 The probability distribution
	$\widehat{\DatZ}_{\ell \sigma}$ is defined such that $\hat{\vg}$ has the same distribution as $\nabla f(\vtheta) + \ell \sigma \vz$ when $\vz \sim \widehat{\DatZ}_{\ell \sigma}(\vtheta)$.
\end{definition}
\Cref{lem:app_svag_first_two} verifies that $\widehat{G}_{\ell \sigma}$ does indeed compute stochastic
gradients for $f$ with covariance $\mSigma$. 
Applying the SVAG operator to mini-batch training amplifies the
noise scale by $\ell$. We then apply the square root scaling rule to adjust
the learning rates and other hyperparameters accordingly, which yields the SVAG
algorithm.
\begin{definition}[SVAG Algorithm] \label{def:svag_alg}
	For a loss $f$, SVAG operator hyperparameter $\ell>0$, and optimization hyperparameters
	$\eta,\beta,\beta_1,\beta_2,$ and $\epsilon$,
	compute the stochastic gradient as
	$\hat{\vg} = r_1(\ell) \vg_{\gamma_1} + r_2(\ell) \vg_{\gamma_2}$, where $r_1$ and $r_2$ are defined as in
	\Cref{def:svag_op}, and scale the optimization hyperparameters:
	\begin{enumerate}[leftmargin=*]
		\item For RMSprop, set $\eta \gets \eta/\ell$, $\beta \gets 1 - (1-\beta)/\ell^2$, and $\epsilon \gets \epsilon \ell$ and apply updates as in \Cref{def:rmsprop}.
		\item For Adam, set $\eta \gets \eta/\ell$, $\beta_1 \gets 1 - (1-\beta_1)/\ell^2$, $\beta_2 \gets 1 - (1-\beta_2)/\ell^2$ and $\epsilon \gets \epsilon \ell$ and apply updates as in \Cref{def:adam}.
	\end{enumerate}
\end{definition}
The SVAG algorithm describes a discrete trajectory that is a 1st-order approximation of the corresponding SDE (\Cref{def:rmsprop_sde,def:adam_sde}), thereby providing an efficient simulation of the SDEs.
\begin{theorem}[SVAG algorithm approximates SDE]\label{thm:svag}
	Assume the NGOS is well-behaved and satisfies the bounded moments condition (\Cref{as:f,def:bound-moment}).
	\begin{enumerate}
		\item Let $\vX_t$ be the state of the RMSprop SDE (\Cref{def:rmsprop_sde}) with hyperparameters $\eta$, $\beta$, and $\epsilon$. Let $\vx_k$ be the state of the analogous discrete SVAG algorithm with hyperparameter $\ell$.
		\item Let $\vX_t$ be the state of the Adam SDE (\Cref{def:adam_sde}) with hyperparameters $\eta$, $\beta_1$, $\beta_2$, and $\epsilon$. Let $\vx_k$ be the state of the analogous discrete SVAG algorithm with hyperparameter $\ell$.
	\end{enumerate}
	In both 1 and 2, following holds for $g$ and $T$ as in \Cref{def:weak_approx}.
	\[
		\max_{k=0,...,\lfloor \ell^2 T/\eta^2\rfloor} | \E g(\vx_k) - \E g(\vX_{k\eta^2/\ell^2}) | \leq C\eta^2 / \ell^2
	\]
\end{theorem}
\begin{proof}
    The main idea of the proof is to show that the SVAG operator transforms the
    noise distribution of a well-behaved NGOS satisfying the bounded moments
    condition into one that is well-behaved and satisfies the
    bounded moments and the low skewness conditions
    (\Cref{lem:app_svag_low_skew}).
    With these three conditions satisfied, we can directly apply \Cref{thm:rmsprop_sde,thm:adam_sde} to complete the proof.
\end{proof}

\begin{figure*}[t]
 \begin{minipage}{1\textwidth}
      \centering
      \begin{minipage}{\linewidth}
          \begin{figure}[H]
              \centerline{\includegraphics[width=1.25\linewidth]{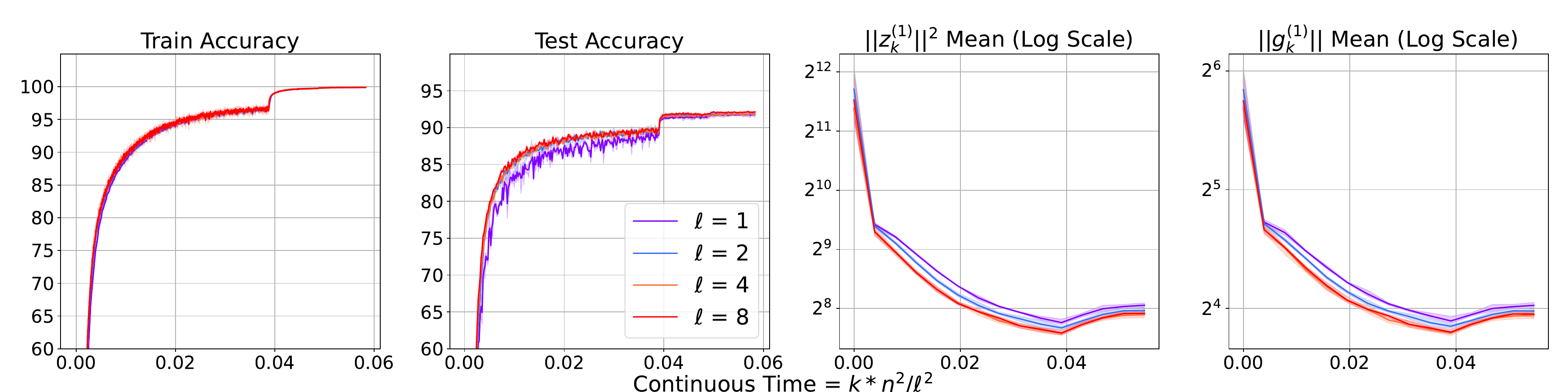}}
    \caption{SVAG on the Adam trajectory when training ResNet-50 on CIFAR-10 matches the discrete trajectory ($\ell=1$) on various test functions (see \cref{sec:app_exp_config} for details). 
    The closeness of the trajectories with respect to various test functions for different values of $\ell$ implies the SDE approximation (\Cref{def:adam_sde}) is a 1st-order weak approximation of Adam (\Cref{thm:adam_sde}).}
                \label{fig:adamsvag}
          \end{figure}
      \end{minipage}
  \end{minipage}
\end{figure*}  

Because the SDE scales time as $k = t/\eta^2$, we must run SVAG for $\ell^2$ steps
to match a single step of the discrete trajectories. 
Nevertheless, we note that in our setting and in \cite{li2021validity}, the approximation guarantee is strong enough for small $\ell$, so this simulation is still more efficient than Euler-Maruyama.

\paragraph{Experiments.} \Cref{fig:adamsvag} compares the Adam SVAG trajectories (\Cref{def:svag_alg}) up to $\ell=8$ to the discrete one ($\ell=1$) on CIFAR-10~\citep{cifar10} with ResNet-50~\citep{he2015delving}. 
We use $\Tr(\mSigma(\vtheta_k))$ and $\|\vg_k\|$ as mathematically well-behaved test functions to test the approximation strength (see \Cref{def:weak_approx}).
We also measure the train and test accuracies, even though they are not differentiable (and hence, not covered by the theory).
The converged SVAG trajectories are close to the discrete ones under these test functions, suggesting the SDE approximations are applicable to realistic deep learning settings.
Additional details and settings, including large language models, are in \cref{sec:app_exp_svag}.




\section{Conclusion}\label{sec:conclusion}
We derive SDEs that are provable 1st-order approximations of the RMSprop and Adam trajectories, immediately yielding formal derivations of square root scaling rules: increase the learning rate by $\sqrt{\kappa}$ and adjust the adaptive hyperparameters when increasing batch size by $\kappa$.
Experiments in the vision and language domains verify that applying these rules ensures that the values of several test functions, including test accuracy, are preserved.
We furthermore design an efficient simulation for the SDEs, allowing us to directly validate the applicability of these SDEs to common vision and language settings.
These SDEs can lead to a deeper understanding of how adaptivity and stochasticity impact optimization and generalization, and we hope to extend our results to formal identification of necessary and sufficient conditions for the approximation and its consequences to hold.

\section*{Acknowledgements}
We thank Zhiyuan Li and Chao Ma for helpful discussion. We also thank Tianyu Gao and Alexander Wettig for helping us run language modeling experiments.
This work is funded by NSF, ONR, Simons Foundation, DARPA and SRC. 

\bibliography{bibliography}

\newpage
\section*{Checklist}

\begin{enumerate}

\item For all authors...
\begin{enumerate}
  \item Do the main claims made in the abstract and introduction accurately reflect the paper's contributions and scope?
    \answerYes
  \item Did you describe the limitations of your work?
    \answerYes{Conditions for our results are outlined in \Cref{as:f,def:low-skew,def:bound-moment}. In particular, we discuss the limited applicability of our work to cases in which the gradient noise is heavy-tailed, though our empirical success suggests that the gradient noise satisfies our assumptions in many realistic vision and language settings.}
  \item Did you discuss any potential negative societal impacts of your work?
    \answerYes{Our work is primarily theoretical in nature, but we discuss the broader impacts of our work in \Cref{sec:app_broader_impact}.}
  \item Have you read the ethics review guidelines and ensured that your paper conforms to them?
    \answerYes{}
\end{enumerate}

\item If you are including theoretical results...
\begin{enumerate}
  \item Did you state the full set of assumptions of all theoretical results?
    \answerYes{\Cref{sec:prelim} has a discussion of the assumptions on the gradient noise and the general requirements needed to be able to construct a continuous approximation of the discrete algorithms. Moreover, \Cref{sec:linear_warmup} motivates the assumption that the gradient noise dominates the gradient. We validate this assumption experimentally in \Cref{sec:app_noise_dominates}. Our assumptions do not stray from prior works on SDEs \citep{li2019stochastic,li2021validity}.}
  \item Did you include complete proofs of all theoretical results?
	\answerYes{Our proofs are written clearly in the appendix.}
\end{enumerate}

\item If you ran experiments...
\begin{enumerate}
  \item Did you include the code, data, and instructions needed to reproduce the main experimental results (either in the supplemental material or as a URL)?
    \answerYes{We include the code for the vision experiments in the supplementary material. For the NLP experiments, we use the code of \cite{wettig2022should}. }
  \item Did you specify all the training details (e.g., data splits, hyperparameters, how they were chosen)?
    \answerYes{\cref{sec:app_exp_config} contains the training details of all the experiments.}
        \item Did you report error bars (e.g., with respect to the random seed after running experiments multiple times)?
    \answerNo{Due to computational constraints, we didn't include the error bars for all the experiments. However, we have made sure that the random seed for different runs in each experimental setting is the same.}
        \item Did you include the total amount of compute and the type of resources used (e.g., type of GPUs, internal cluster, or cloud provider)?
    \answerYes{We ran our experiments on a cluster of 34 GPUs, where 24 are RTX 2080 GPUs and 10 are A5000 GPUs. Each experiment on CIFAR-10 required a single RTX 2080 GPU, each experiment on ImageNet required a single A5000 GPU, each pretraining experiment on GPT required a set of 4 RTX 2080 GPUs, each pretraining experiment on RoBERTa required a set of 8 RTX 2080 GPUs, and each finetuning experiment on RoBERTa required a single RTX 2080 GPU.}
\end{enumerate}

\item If you are using existing assets (e.g., code, data, models) or curating/releasing new assets...
\begin{enumerate}
  \item If your work uses existing assets, did you cite the creators?
    \answerYes{We cite \cite{wettig2022should} for our language model experiments. We also cite the papers of all the deep learning architectures and datasets that we use for our experiments.}
  \item Did you mention the license of the assets?
    \answerYes{The github repository\footnote{\url{https://github.com/princeton-nlp/DinkyTrain}} of \cite{wettig2022should} has MIT license.}
  \item Did you include any new assets either in the supplemental material or as a URL?
    \answerYes{We include our code for the vision experiments in the supplementary material.}
  \item Did you discuss whether and how consent was obtained from people whose data you're using/curating?
    \answerNA{The data we used for our experiments are publicly available.}
  \item Did you discuss whether the data you are using/curating contains personally identifiable information or offensive content?
    \answerNA{The data we used for our experiments are publicly available.}
\end{enumerate}

\item If you used crowdsourcing or conducted research with human subjects...
\begin{enumerate}
  \item Did you include the full text of instructions given to participants and screenshots, if applicable?
    \answerNA{}
  \item Did you describe any potential participant risks, with links to Institutional Review Board (IRB) approvals, if applicable?
    \answerNA{}
  \item Did you include the estimated hourly wage paid to participants and the total amount spent on participant compensation?
    \answerNA{}
\end{enumerate}

\end{enumerate}

\newpage
\appendix

\section{Contextualizing our Work}
\subsection{Additional Recent Works}\label{sec:app_addl_works} Variations of
Adam have been proposed to improve its speed of convergence, generalization, and
stability during training. \citet{amsgrad} observed that Adam does not collect
long-term memory of past gradients and therefore the effective learning rate
could be increasing in some cases. Hence, they propose AMSGrad that maintains a
maximum over the exponential running average of the squared gradients.
\citet{yogi} proposed a more controlled increase in the effective learning rate
by switching to additive updates, using a more refined version of
AdaGrad~\citep{duchi2011adagrad}. \citet{padam} unified the generalization ability
of vanilla SGD and the convergence speed of Adam by introducing a new adaptive
parameter $p \in (0, \frac{1}{2}]$ that can be hypertuned for each setting.
Other variations include (a) Nadam~\citep{dozat2016incorporating} that uses
Nesterov momentum, (b) AdamW~\citep{loshchilov2018decoupled} that decouples the
weight decay from the optimization step, (c) AdaBound~\citep{luo2018adaptive}
that maintains  a dynamic upper and lower bound on the step size, (d)
AdaBelief~\citep{zhuang2020adabelief} uses a decaying average of estimated
variance in the gradient in place of the running average of the squared
gradients, (e) QHAdam~\citep{ma2018quasihyperbolic} that replaces both momentum
estimators in Adam with quasi-hyperbolic terms, etc.
LAMB~\citep{you2020lamb} used a layerwise adaptive version of Adam to pretrain large language models efficiently. 

\subsection{Broader Impact}\label{sec:app_broader_impact}
Our work is primarily theoretical in nature, but we discuss its broader impacts here. 
\cite{strubell2019energy} highlighted the environmental impact of training large language models.
Formal scaling rules remove the need to grid search over hyperparameters: in the case of adaptive algorithms, the grid search is over an even larger space because of the additional adaptivity hyperparameters.
Hence, our work reduces the number of times researchers need to train very large models to find ones that work just as well as their smaller, slower counterparts.
At the same time, we recognize that the presence of a formal scaling rule may encourage people who were otherwise discouraged by the daunting grid search to train very large models.

\section{SDE Approximation Theorem}


In this section, we introduce a theorem in aid of our analysis
for approximating Stochastic Gradient Algorithm
(SGA) with Stochastic Differential Equation (SDE).

We consider Stochastic Gradient Algorithm (SGA) of the following form:
\begin{equation}
    \vx_{k+1} = \vx_k + \eeta \vh_k(\vx_k, \vxi_k, \eeta),
\end{equation}
where $\vx_k \in \R^D$ is the parameter vector, $\eeta$ is the learning rate,
$\vh_k$ is a vector-valued function that can depend on the current step $k$, the
current parameter $\vx_k$, a random vector $\vxi_k$,
and the learning rate $\eeta$.
The random vector $\vxi_k$ is sampled from a certain distribution $\DatXi(\vx_k, \eeta)$ in every step.

We consider Stochastic Differential Equation (SDE) of the following form:
\begin{equation} \label{eq:sde-general}
    \dd \vX_t = \vb(X_t, t) \dd t + \msigma(X_t, t) \dd \vW_t,
\end{equation}
where $\vb: \R^D \times \R \to \R^D$ is the drift vector function, $\msigma: \R^D \times \R \to \R^{D \times D}$ is the diffusion matrix function.

Let $\prdX(\vx, s, t)$ be the probability distribution of $\vX_t$ in \eqref{eq:sde-general} when the initial condition is $\vX_s = \vx$.
Now we define the following random variables to characterize the one-step changes of SGA and SDE.
\begin{align}
    \vDelta(\vx, k) &:= \eeta \vh_k(\vx, \vxi, \eeta), & & \text{where} \quad \vxi \sim \DatXi(\vx, \eeta). \label{eq:def-delta} \\
    \tilde{\vDelta}(\vx, k) &:= \vX_{(k+1)\eeta} - \vx, & & \text{where} \quad \vX_{(k+1)\eeta} \sim \prdX(\vx, k\eeta, (k+1)\eeta). \label{eq:def-tilde-delta}
\end{align}

The following regularity condition is needed for our main theorem.
\begin{definition}
   A function $g \colon \R^d \to \R$ is said to have {\em polynomial growth} if there exist positive integers $\kappa_1,\kappa_2>0$ such that 
	\[
		|g(\vx)| \leq \kappa_1(1+\normtwosm{\vx}^{2\kappa_2}),
	\]
	for all $\vx \in \R^d$. We let $G$ denote the set of all such functions. For each integer $\alpha\geq 1$, $G^\alpha$ denotes the set of $\alpha$-times continuously differentiable functions $\sR^d\to\sR$ which, together with its partial derivatives up to and including order $\alpha$, belong to $G$. If $g$ depends on additional parameters, we say $g\in G^\alpha$ uniformly if the constants $\kappa_1,\kappa_2$ are independent of those parameters. 
	\label{def:app_polygrowth_fn}
\end{definition}

Now we are ready to introduce the main theorem for the order of approximation by
modeling SGA as SDE. We follow the same proof strategy as
\cite{li2019stochastic}. First, we show that the moments of the one-step
difference for the SGA and for the SDE of our interest are close to each other.
Then, under some regularity conditions, we translate the one-step error to an
error over a finite interval of time to get the final result. The following is
our main theorem to achieve such translation, and it is adapted from Theorem 3
in \citet{li2019stochastic} with slightly different conditions to match our need
in the analysis of adaptive gradient methods.
\begin{theorem}[Adaption of Theorem 3 in \citet{li2019stochastic}]
    \label{thm:app_many_step} Let $T > 0$, $\eeta \in (0, 1 \land T)$ and $N =
    \lfloor T / \eeta \rfloor$. Consider an SDE with drift vector $\vb(\vx,
    \eeta)$ and diffusion matrix $\msigma(\vx, \eeta)$, and a stochastic
    gradient algorithm with initial point $\vx_0 \in \R^D$ and update rule
    $\vx_{k+1} = \vx_k + \eeta \vh(\vx_k, \vxi_k, \eeta)$. Let $\gX_k$ be the
    support of the random variable $\vx_k$ given $\vx_0$. Assume the following
    conditions hold:
    \begin{enumerate}[(1).]
        \item The drift function $\vb$ and diffusion function $\msigma$ are
        Lipschitz and belong to $G^4$;
        \item The moments of $\vDelta\in\sR^D$ and $\tilde{\vDelta}\in\sR^D$
        satisfy the following. There is a function $K_1 \in G$ (independent of
        $\eeta$) such that for all $\vx \in \gX_k$ and $1 \le i, j, l \le D$,
        \begin{align*}
            \abs{\E[\Delta_i(\vx, k) - \tilde{\Delta}_i(\vx, k)]} &\le K_1(\vx) \eeta^2, \\
            \abs{\E[\Delta_i(\vx, k)\Delta_j(\vx, k) - \tilde{\Delta}_i(\vx, k) \tilde{\Delta}_j(\vx, k)]} & \le K_1(\vx) \eeta^2, \\
            \abs{\E[\Delta_i(\vx, k)\Delta_j(\vx, k)\Delta_l(\vx, k) - \tilde{\Delta}_i(\vx, k) \tilde{\Delta}_j(\vx, k) \tilde{\Delta}_l(\vx, k)]} &\le K_1(\vx) \eeta^2.
        \end{align*}
        \item There exists a subset $P$ of the index set $\{1, 2, \dots, D\}$ such that the following holds.
        Below we use the notations $\normPsm{\vx} := \sqrt{\sum_{i \in P} x_i^2}$ and $\normRsm{\vx} := \sqrt{\sum_{i \notin P} x_i^2}$.
        \begin{itemize}
            \item There are constants $C_1 > 0, \omega_1 > 0$ (independent of $\eeta$) so that for all $k \le N$ and $\vx \in \gX_k$,
            \begin{align*}
                \normPsm{\E\vDelta(\vx, k)} &\le C_{1} \eeta (1 + \normPsm{\vx}), \\
                \normRsm{\E\vDelta(\vx, k)} &\le C_{1} \eeta
                (1 + \normPsm{\vx}^{\omega_1})
                (1 + \normRsm{\vx}).
            \end{align*}
            \item For all $m \ge 1$, there are constants $C_{2m}, \omega_{2m} > 0$ (independent of $\eeta$) so that
            for all $k \le N$ and $\vx \in \gX_k$,
            \begin{align*}
                \E\normPsm{\vDelta(\vx, k)}^{2m} &\le C_{2m} \eeta^{m} (1 + \normPsm{\vx}^{2m}), \\
                \E\normRsm{\vDelta(\vx, k)}^{2m} &\le C_{2m} \eeta^{m}
                (1 + \normPsm{\vx}^{\omega_{2m}})
                (1 + \normRsm{\vx}^{2m}).
            \end{align*}
        \end{itemize}
    \end{enumerate}
    Then for each function $g \in G^4$, there exists a constant $C > 0$ (independent of $\eeta$) such that
    \begin{equation}
        \max_{0 \le k \le N} \abs{\E[g(\vx_k)] - \E[g(\vX_{k\eeta})]} \le C \eeta,
    \end{equation}
    when the SDE starts from the same initial point $\vx_0$ as SGA.
    That is, the SGA is a first-order weak approximation of the SDE.
\end{theorem}

To apply \Cref{thm:app_many_step}, the major condition to verify is that $\vDelta$ and $\tilde{\vDelta}$ match in moments. The following lemma computes the moments for $\tilde{\vDelta}$:
\begin{lemma} \label{lm:tilde-delta}
    For drift function $b$ and diffusion function $\msigma$ that belong to $G^4$,
    there is a function $K_2 \in G$ (independent of $\eeta$) such that for all $\vx \in \R^D$ and $1 \le i, j, l \le D$,
    \begin{align*}
        \abs{\E[\tilde{\Delta}_i(\vx, k)] - \eeta b_i(\vx, k\eeta) } &\le K_2(\vx) \eeta^2, \\
        \abs{\E[\tilde{\Delta}_i(\vx, k)\tilde{\Delta}_j(\vx, k)] - \eeta^2 \sum_{k=1}^{D} \sigma_{i,k}(\vx, k\eeta) \sigma_{j,k}(\vx, k\eeta)} &\le K_2(\vx) \eeta^2, \\
        \abs{\E[\tilde{\Delta}_i(\vx, k)\tilde{\Delta}_j(\vx, k)\tilde{\Delta}_l(\vx, k)]} &\le K_2(\vx) \eeta^2.
    \end{align*}
\end{lemma}
\subsection{Proof for Theorem \ref{thm:app_many_step}} \label{sec:app_time_dependent}

To prove \Cref{thm:app_many_step}, we need the following lemma from \cite{li2019stochastic}.
However, the original version of the lemma does not apply to time-dependent SDEs, i.e., $\vb$ and $\msigma$ cannot change with time $t$.
By carefully scrutinizing the proof, we find that the proof is indeed applicable to time-dependent SDEs.
\begin{lemma}[Adaption of Proposition 25, \cite{li2019stochastic}]
    \label{lm:prop-25} Suppose that the drift function $\vb$ and diffusion
    function $\msigma$ are Lipschitz and belong to $G^{\alpha}$ for some $\alpha
    \ge 1$. Let $s \in [0, T]$ and $g
    \in G^{\alpha}$. For $t \in [s, T]$, define
    \[
        u(\vx, s, t) := \E_{\vX_t \sim \prdX(\vx, s, t)}[g(\vX_t)],
    \]
    Then $u(\cdot, s, t) \in G^{\alpha}$ uniformly in $s, t$.
\end{lemma}
We need the following lemma to bound the growth of $\vx_k$.
\begin{lemma} \label{lm:x-growth}
    Under Condition (3) of \Cref{thm:app_many_step}, given the initial point $\vx_0$ and any $m \ge 1$, there
    exists a constant $C'_{2m} > 0$ (depending on $\vx_0$ but independent of $\eeta$) such that the
    parameters $\{\vx_k\}$ of SGA starting from $\vx_0$ can be uniformly bounded by $\E[\normtwosm{\vx_k}^{2m}] \le C'_{2m}$
    for all $k \le N := \lfloor T / \eeta \rfloor$.
\end{lemma}
\begin{proof}
It suffices to show that both $\E[(1 + \normPsm{\vx_k}^{2})^{m}]$
and
$\E[(1 + \normRsm{\vx_k}^{2})^{m}]$
are
uniformly bounded.

    First, we show that for $\E[(1 + \normPsm{\vx_k}^{2})^{m}]$.
    For all $2 \le j \le 2m$, 
    by Jensen's inequality,
    \begin{align*}
        \E\left[\normPsm{\vDelta(\vx_k, k)}^{j} \mid \vx_k\right] \le \E\left[\normPsm{\vDelta(\vx_k, k)}^{2m} \mid \vx_k\right]^{\frac{j}{2m}} \le C_{2m}^{\frac{j}{2m}} \eeta^{j/2} (1+\normPsm{\vx_k}^{2m})^{\frac{j}{2m}}.
    \end{align*}
    Then there exists a constant $\hat{C}_1 > 0$ so that
    \begin{equation} \label{eq:x-growth-a}
        \E\left[\normPsm{\vDelta(\vx_k, k)}^{j} \mid \vx_k\right] \le \hat{C}_1 \eeta (1 + \normPsm{\vx_k}^{j}).
    \end{equation}
    For any two vectors $\vx, \vy \in \R^D$, we use the notation $\dotpsm{\vx}{\vy}_{\mathrm{P}} := \sum_{i \in P} x_i y_i$ to denote the inner product of $\vx$ and $\vy$ restricting on coordinates in $P$.
    Now we expand $\left(1 + \normPsm{\vx_{k+1}}^{2}\right)^m$ using the update rule.
    Let $\delta_k := 2\dotpsm{\vx_k}{\vDelta(\vx_k, k)}_{\mathrm{P}} + \normPsm{\vDelta(\vx_k, k)}^2$. Then by the binomial theorem,
    \begin{align*}
        \E\left[\left(1 + \normPsm{\vx_{k+1}}^{2}\right)^m \mid \vx_k\right] 
        &= \E\left[\left( 1 + \normPsm{\vx_{k}}^2 + \delta_k\right)^{m} \mid \vx_k \right] \\
        &= 
        (1+\normPsm{\vx_k}^2)^m
        + m \E[\delta_k \mid \vx_k] (1+\normPsm{\vx_k}^2)^{m-1} \\
        &\quad + \sum_{j=2}^{m} \binom{m}{j} \E[\delta_k^{j} \mid \vx_k] (1+\normPsm{\vx_k}^2)^{m-j}.
    \end{align*}
    By \eqref{eq:x-growth-a},
    it can be shown that there is a constant $\hat{C}_2$
    such that $\E[\delta_k^{j} \mid \vx_k] \le \hat{C}_2 \eeta (1 + \normPsm{\vx_k}^{2})^j$
    for $j \ge 2$.
    Then there exists constant $\hat{C}_3, \hat{C}_4$
    such that
    \begin{align*}
        \E\left[\left(1 + \normPsm{\vx_{k+1}}^{2}\right)^m \mid \vx_k\right] 
        &\le (1 + \normPsm{\vx_k}^{2})^{m} + m \E[\delta_k \mid \vx_k] (1+\normPsm{\vx_k}^2)^{m-1} \\
        &\quad + \hat{C}_3 \eeta (1 + \normPsm{\vx_k}^2)^m \\
        &\le (1 + \normPsm{\vx_k}^{2})^{m} + 2 m \dotp{\vx_k}{\E[\vDelta(\vx_k, k) \mid \vx_k]}_{\mathrm{P}} (1+\normPsm{\vx_k}^2)^{m-1} \\
        &\quad + \hat{C}_4 \eeta (1 + \normPsm{\vx_k}^2)^m.
    \end{align*}
    Recall that 
    $\normPsm{\E[\vDelta(\vx_k, k) \mid \vx_k]} \le C_1 \eeta
    (1+\normPsm{\vx_k})$ by Condition (3).
    Thus, there exists a constant $\hat{C}_5$
    (independent of $\eeta$, $\normPsm{\vx_k}$)
    such that
    \begin{align*}
        \E\left[\left(1 + \normPsm{\vx_{k+1}}^{2}\right)^m \mid \vx_k\right]
        &\le (1 + \hat{C}_5 \eeta) (1 + \normPsm{\vx_k}^2)^m.
    \end{align*}
    Taking the expectation over $\vx_k$ gives
    \begin{align*}
        \E\left[\left(1 + \normPsm{\vx_{k+1}}^{2}\right)^m\right] \le (1 + \hat{C}_5 \eeta) \E[(1 + \normPsm{\vx_k}^2)^m].
    \end{align*}
    Then taking a telescoping product proves that
    for all $k \le N$,
    \begin{align*}
    \E[(1 + \normPsm{\vx_k}^{2})^{m}]
    \le (1+\hat{C}_5\eeta)^k (1 + \normPsm{\vx_0}^{2})^m &\le \exp(\hat{C}_5 k \eeta) (1 + \normPsm{\vx_0}^{2})^m \\
    &\le \exp(\hat{C}_5 T) (1 + \normPsm{\vx_0}^{2})^m.
    \end{align*}
    So $\E[(1 + \normPsm{\vx_k}^{2})^{m}]$ is uniformly bounded (independent of $\eeta, k$).

    Now we show that $\E[(1 + \normRsm{\vx_k}^{2})^{m}]$ is uniformly bounded.
    We can repeat the argument above to bound $\E[(1 + \normRsm{\vx_k}^{2})^{m}]$,
    while utilizing the bounds $\normRsm{\E\vDelta(\vx, k)} \le C_{1} \eeta(1 + \normPsm{\vx}^{\omega_1}) (1 + \normRsm{\vx})$, $\E\normRsm{\vDelta(\vx, k)}^{2m} \le C_{2m} \eeta^{m}(1 + \normPsm{\vx}^{\omega_{2m}}) (1 + \normRsm{\vx}^{2m})$.
    In the end we can obtain the following for some real constant $\hat{C}_6 > 0$ and some integer constant $\hat{\omega} > 0$:
    \begin{align*}
        \E\left[\left(1 + \normRsm{\vx_{k+1}}^{2}\right)^m\right]
        \le (1 + \hat{C}_6 \eeta \E[1 + \normPsm{\vx_k}^{2\hat{\omega}}]) \E\left[(1 + \normRsm{\vx_k}^2)^m\right].
    \end{align*}
    As we have shown, $\E[1 + \normPsm{\vx_k}^{2\hat{\omega}}]$ is uniformly bounded by a constant.
    Taking a telescoping product proves that
    $\E[(1 + \normRsm{\vx_k}^{2})^{m}]$
    is uniformly bounded (independent of $\eeta, k$).
\end{proof}
We also need the following lemma adapted from Lemma C.2, \cite{li2021validity}.
\begin{lemma}[Adaption of Lemma C.2, \cite{li2021validity}] \label{lm:lm-27} Let
    $u_1, \dots, u_N$ be a set of functions that belong to $G^4$ uniformly.
    Under Conditions (2), (3) in \Cref{thm:app_many_step}, if 
    $\vb$ and $\msigma$ are Lipschitz, then there exists
    a function $K_1' \in G$ (independent of $\eeta$) such that
    \begin{align*}
        \abs{\E[u_j(\vx + \vDelta(\vx, k))] - \E[u_j(\vx + \tilde{\vDelta}(\vx, k))]} \le K_1'(\vx) \eeta^2,
    \end{align*}
    for all $1 \le j \le N$, $1 \le k \le N$ and $\vx \in \gX_k$.
\end{lemma}
\begin{proof}
    Since $u_1, \dots, u_N \in G^4$ uniformly, we can find $K_0 \in G$ such that $u_j(\vx)$ is bounded by $K_0(\vx)$ and so are all the partial derivatives of $u_j$ up to order $4$.

    By Taylor's Theorem with Lagrange Remainder, for all $1 \le j \le N$, $1 \le k \le N$ we have
    \begin{align*}
        &u_j(\vx + \vDelta(\vx, k)) - u_j(\vx + \tilde{\vDelta}(\vx, k)) \\
        &\qquad = \underbrace{\sum_{s=1}^3 \frac{1}{s!} \sum_{1 \le i_1, \dots, i_s \le D} \tfrac{\partial^s u_j}{\partial x_{i_1} \cdots \partial x_{i_s}}(\vx) \left(\prod_{r = 1}^{s} \Delta_{i_r}(\vx, k) - \prod_{r = 1}^{s} \tilde{\Delta}_{i_r}(\vx, k) \right)}_{=: M_j} + R_j - \tilde{R}_j,
    \end{align*}
    where the remainders $R_j$, $\tilde{R}_j$ are
    \begin{align*}
        R_j &:=
        \frac{1}{4!} \sum_{1 \le i_1, \dots, i_4 \le D} \tfrac{\partial^4 u_j}{\partial x_{i_1} \cdots \partial x_{i_4}}(\vx + a \vDelta(\vx, k)) \prod_{r = 1}^{4} \Delta_{i_r}(\vx, k). \\
        \tilde{R}_j &:=
        \frac{1}{4!} \sum_{1 \le i_1, \dots, i_4 \le D} \tfrac{\partial^4 u_j}{\partial x_{i_1} \cdots \partial x_{i_4}}(\vx + \tilde{a} \tilde{\vDelta}(\vx, k)) \prod_{r = 1}^{4} \tilde{\Delta}_{i_r}(\vx, k).
    \end{align*}
    for some $a, \tilde{a} \in [0, 1]$.

    By Condition (2), the expectation of $M_j$ can be bounded by
    \begin{align*}
        \E[M_j] &\le \sum_{s=1}^3 \frac{1}{s!} \sum_{1 \le i_1, \dots, i_s \le D} \abs{\tfrac{\partial^s u_j}{\partial x_{i_1} \cdots \partial x_{i_s}}(\vx)} \cdot K_1(\vx) \eeta^2 \le \sum_{s=1}^3 \frac{D^s}{s!} K_0(\vx) K_1(\vx) \eeta^2,
    \end{align*}
    so $\frac{1}{\eeta^2}\E[M_j]$ is uniformly bounded by a function in $G$.

    Now let $\kappa_0, m$ be the constants so that $K_0(\vx)^2 \le \kappa_0^2 (1 + \normtwosm{\vx}^{2m})$.
    For $R_j$, by Cauchy-Schwarz inequality we have
    \begin{align*}
        \E[R_j]
        &\le \frac{1}{4!} \left(
            \sum_{i_1, \dots, i_4}
            \E\left[\abs{\tfrac{\partial^4 u_j}{\partial x_{i_1} \cdots \partial x_{i_4}}(\vx + a \vDelta(\vx, k))}^2\right]
        \right)^{1/2}
        \cdot
        \left(
            \sum_{i_1, \dots, i_4}
            \E\left[ \abs{\prod_{r = 1}^{4} \Delta_{i_r}(\vx, k)}^2 \right]
        \right)^{1/2} \\
        &\le \frac{1}{4!} \left(
            \sum_{i_1, \dots, i_4}
            \E\left[\abs{\tfrac{\partial^4 u_j}{\partial x_{i_1} \cdots \partial x_{i_4}}(\vx + a \vDelta(\vx, k))}^2\right]
        \right)^{1/2}
        \cdot \E\left[ \normtwosm{\vDelta(\vx, k)}^8 \right]^{1/2} \\
        &\le \frac{1}{4!} \left(D^2 \cdot K_0 (\vx + a \vDelta(\vx, k))\right)
        \cdot \left( \eeta^4 K_8(\vx) \right)^{1/2},
    \end{align*}
    where the last line uses Condition (3) and $K_8$ is a function of polynomial growth.
    For $K_0(\vx + a \vDelta(\vx, k))$, we can bound its expectation by
    \begin{align*}
        \E[K_0(\vx + a \vDelta(\vx, k))]
        &\le \kappa_0 \E\left[1 + \normtwosm{\vx + a \vDelta(\vx, k)}^{2m}\right]^{1/2} \\
        &\le \kappa_0 \left( 1 + 2^{2m-1} \E[\normtwosm{\vx}^{2m} + \E\normtwosm{\vDelta(\vx, k)}^{2m}] \right)^{1/2} \\
        &\le \kappa_0 \left( 1 + 2^{2m-1} (\normtwosm{\vx}^{2m} + C_{2m} \eeta^{2m} (1 + \normtwosm{\vx}^{2m}))  \right)^{1/2}.
    \end{align*}
    Combining this with our bound for $\E[R_j]$ proves that $\frac{1}{\eeta^2}\E[R_j]$ is uniformly bounded by a function in $G$:
    \begin{align*}
        \E[R_j] \le
        \eeta^2 \cdot 
        \frac{1}{4!} \cdot D^4 \cdot 
        \kappa_0 \left( 1 + 2^{2m-1} (\normtwosm{\vx}^{2m} + C_{2m} \eeta^{2m} (1 + \normtwosm{\vx}^{2m}))  \right)^{1/2}
        \cdot K_8^{1/2}(\vx).
    \end{align*}

    Then we can repeat the above argument for $R_j$ while replacing $\Delta$ as $\tilde{\Delta}$,
    and conclude that $\frac{1}{\eeta^2}\E[\tilde{R}_j]$ is also uniformly bounded by a function in $G$.
    To do so, we 
    note that $\vb, \msigma$ are Lipschitz,
    and apply a similar argument as in Lemma 26 of \citet{li2019stochastic}
    to show that for all $s \ge 1$ there exists a function $\tilde{K} \in G$ such that
    for all $1 \le i_1, \dots, i_s \le D$,
    \begin{align*}
        \E\left[ \abs{\prod_{r = 1}^{s} \tilde{\Delta}_{i_r}(\vx, k) } \right]
        \le \tilde{K}(\vx) \eeta^{s}.
    \end{align*}
    Finally, we can find a function $K'_1 \in G$ such that
    $\E[M_j] \le \frac{1}{3}K'_1(\vx) \eeta^2$,
    $\E[R_j] \le \frac{1}{3}K'_1(\vx) \eeta^2$,
    $\E[\tilde{R}_j] \le \frac{1}{3}K'_1(\vx) \eeta^2$.
    Then $\E[u_j(\vx + \vDelta(\vx, k))] - \E[u_j(\vx + \tilde{\vDelta}(\vx, k))] \le K_1'(\vx) \eeta^2$.
\end{proof}
Now we are ready to present our proof for \Cref{thm:app_rmsprop_sde}.
\begin{proof}[Proof for Theorem \ref{thm:app_many_step}]
    For $0 \le j \le k$, let $\hat{\vx}_{j,k}$ be a random variable
    that is distributed as the probability distribution $\prdX(\vx_j, j\eeta, k\eeta)$ conditioned on $\vx_j$.
    By definition, $\Pr[\hvx_{k, k} = \vx_k] = 1$, $\hvx_{0, k} \sim \vX_{k \eeta}$.
    Let $u(\vx, s, t) := \E_{\vX_t \sim \prdX(\vx, s, t)}[g(\vX_t)]$.
    Then we can do the following decomposition for $\E[g(\vx_k)] - \E[g(\vX_{k \eeta})]$:
    \begin{align*}
        \abs{\E[g(\vx_k)] - \E[g(\vX_{k \eeta})]}
        &= 
        \sum_{j=0}^{k-1} \left(\E[g(\hvx_{j+1,k})] - \E[g(\hvx_{j,k})]\right) \\
        &= 
        \sum_{j=0}^{k-1} \left(
            \E[u(\hvx_{j+1,j+1}, (j+1)\eeta, k\eeta)] - \E[u(\hvx_{j,j+1}, (j+1)\eeta, k\eeta)]
        \right).
    \end{align*}
    Taking absolute values gives
    \begin{align*}
        \abs{\E[g(\vx_k)] - \E[g(\vX_{k \eeta})]}
        &\le \sum_{j=0}^{k-1} \abs{\E[u(\hvx_{j+1,j+1}, (j+1)\eeta, k\eeta)] - \E[u(\hvx_{j,j+1}, (j+1)\eeta, k\eeta)]}.
    \end{align*}
    Let $u_{j+1}(\vx) := u(\vx, (j+1) \eeta, k\eeta)$. Note that $\hvx_{j+1, j+1} \sim \vx_j + \vDelta(\vx_j, j)$, $\hvx_{j,j+1} \sim \vx_j + \tilde{\vDelta}(\vx_j, j)$. We can rewrite the above formula as
    \begin{align*}
        \abs{\E[g(\vx_k)] - \E[g(\vX_{k \eeta})]} \le \sum_{j=0}^{k-1} \abs{\E[u_{j+1}(\vx_{j} + \vDelta(\vx_j, j))] - \E[u_{j+1}(\vx_j + \tilde{\vDelta}(\vx_j, j))]}.
    \end{align*}
    By \Cref{lm:prop-25}, $u \in G^4$ uniformly in $s, t$, so $u_1, \dots, u_N
    \in G^4$ uniformly. Then by \Cref{lm:lm-27}, we know that there exists a
    function $K'_1(\vx) = \kappa_1 (1 + \normtwosm{\vx}^{2m}) \in G$ such that
    \begin{align*}
        \abs{\E[u_{j+1}(\vx_{j} + \vDelta(\vx_j, j))] - \E[u_{j+1}(\vx_j + \tilde{\vDelta}(\vx_j, j))]} \le \E[K_1'(\vx_j) \eeta^2],
    \end{align*}
    for all $0 \le j < N$. Combining this with \Cref{lm:x-growth}, we can bound $\abs{\E[g(\vx_k)] - \E[g(\vX_{k \eeta})]}$ by
    \begin{align*}
    \abs{\E[g(\vx_k)] - \E[g(\vX_{k \eeta})]} \le \sum_{j=0}^{k-1} \E[K_1'(\vx_j) \eeta^2]
    &\le \eeta^2 \sum_{j=0}^{k-1} \E[\kappa_1 (1+\normtwosm{\vx_j}^{2m})] \\
    &\le \eeta^2 \sum_{j=0}^{k-1} \kappa_1(1 + C'_{2m}) \le \kappa_1(1 + C'_{2m}) T \eeta.
    \end{align*}
    We can complete the proof by noting that $\kappa_1$, $C'_{2m}$, $T$ are independent of $\eeta$.
\end{proof}

\subsection{Proof for Lemma \ref{lm:tilde-delta}}

To prove \Cref{lm:tilde-delta}, we only need to verify the following lemma using \ito-Taylor expansion.
\begin{lemma} \label{lm:psi}
    Let $\psi: \R^D \to \R$ be a function in $G^4$. Define
    \begin{align*}
        \gA_{t} \psi(\vx) &:= \sum_{i \in [D]} b_i(\vx, t) \partial_i \psi(\vx) + \frac{1}{2} \sum_{i, j \in [D]} \left( \sum_{l \in [D]} \sigma_{i,l}(\vx, t) \sigma_{l, j}(\vx, t) \right) \partial^2_{i, j} \psi(\vx).
    \end{align*}
    Then there exists a function $\hat{K} \in G$ such that
    \begin{equation}
        \abs{\E\left[\psi(\vx + \tilde{\vDelta}(\vx, k))\right] - \psi(\vx) - \eeta \gA_{k\eeta} \psi(\vx)} \le \hat{K}(\vx) \eeta^2.
    \end{equation}
\end{lemma}

\begin{proof}[Proof for \Cref{lm:tilde-delta}]
    We can prove the lemma by applying \Cref{lm:psi} with $\psi(\vx)$ being $\psi(\tilde{\vx}) = \prod_{r=1}^{s} (\tilde{x}_{i_r} - x_{i_r})$
    for any tuple $(i_1, \dots, i_s) \in [D]^s$ with $s \le 3$ elements.
\end{proof}

\begin{proof}[Proof for \Cref{lm:psi}]
    WLOG we prove the case of $k = 0$, then all the other cases can be proved by shifting the time.
    Let $\Lambda_t \psi(\vx) := \msigma(\vx, t)^{\top} \nabla \psi(\vx)$.
    \begin{align*}
        \psi(\vX_{\eta}) = \psi(\vx) + \int_{0}^{\eeta} \gA_s \psi(\vX_s) \dd s + \int_{0}^{\eeta} \dotpsm{\Lambda_s \psi(\vX_s)}{\dd \vW_s}.
    \end{align*}
    Now we further apply the above formula to $\gA_s \psi(\vX_s)$. Then we have
    \begin{align*}
        \psi(\vX_{\eta}) &= \psi(\vx) + \int_{0}^{\eeta} \left( \gA_s \psi(\vx) + \int_{0}^{s} \gA_r \gA_s \psi(\vX_r) \dd r  + \int_{0}^{s} \dotpsm{\Lambda_r \gA_s \psi(\vX_r)}{\dd \vW_r} \right) \dd s \\
        &\qquad + \int_{0}^{\eeta} \dotpsm{\Lambda_s \psi(\vX_s)}{\dd \vW_s} \\
        &= \psi(\vx) + \int_{0}^{\eeta} \gA_s \psi(\vx) \dd s + \int_0^{\eeta} \int_{0}^{s} \gA_r \gA_s \psi(\vX_r) \dd r \dd s \\
        &\qquad 
        + \int_0^{\eeta} \int_{0}^{s} \dotpsm{\Lambda_r \gA_s \psi(\vX_r)}{\dd \vW_r} \dd s 
        + \int_{0}^{\eeta} \dotpsm{\Lambda_s \psi(\vX_s)}{\dd \vW_s}.
    \end{align*}
    Taking expectation, the last two integrals vanish. So we have
    \begin{align*}
        \E\psi(\vX_{\eta})
        &= \psi(\vx) + \int_{0}^{\eeta} \gA_s \psi(\vx) \dd s + \int_0^{\eeta} \int_{0}^{s} \E[\gA_r \gA_s \psi(\vX_r)] \dd r \dd s.
    \end{align*}
    By Lipschitzness of $\vb$ and $\msigma$, $\frac{1}{\eeta}\left(\gA_s \psi(\vx) - \gA_{0}
    \psi(\vx)\right)$ is bounded by a function of $\vx$ with polynomial growth.
    Also, $\gA_r \gA_s \psi(\cdot)$ is in $G$ uniformly, then Theorem 19 in \citep{li2019stochastic} implies that
    $\E[\gA_r \gA_s \psi(\vX_r)]$ is also in $G$ uniformly. Then we know that there exists two functions $\hat{K}_1, \hat{K}_2 \in G$ such that
    \begin{align*}
        \abs{\E\left[\psi(\vX_{\eeta})\right] - \psi(\vx) - \eeta \gA_0 \psi(\vx)} &\le \int_0^{\eeta} \eeta\hat{K}_1(\vx) \dd s + \int_0^{\eeta} \int_{0}^{s} \hat{K}_2(\vx) \dd r \dd s \\
        &\le \eeta^2 (\hat{K}_1(\vx) + \hat{K}_2(\vx)).
    \end{align*}
    We can conclude the proof by setting $\hat{K}(\vx) := \hat{K}_1(\vx) + \hat{K}_2(\vx)$.
\end{proof}

\section{RMSProp SDE Proof}

In this section, we prove the theorem for the SDE approximation of RMSprop.
\begin{definition}[SDE for RMSprop, matrix form] \label{def:rmsprop-sde-matrix}
For constants $\sigma_0$, $\epsilon_0$, and $c_2$, define
the SDE as $\dd \vX_t = \vb(\vX_t) \dd t + \msigma(\vX_t) \dd \vW_t$,
where
$\vX_t \in \R^d \times \R^d$, and
$\vb$ and $\msigma$ are defined by
\begin{align*}
    b_i(\vtheta, \vu) &:= -\tfrac{1}{\sigma_0 \sqrt{u_i} + \epsilon_0} \cdot \partial_i f(\vtheta), & b_{d+i}(\vtheta, \vu) &:= c_2(\Sigma(\vtheta)_{i,i} - u_i). \\
    \sigma_{i,j}(\vtheta, \vu) &:= \tfrac{1}{\sqrt{u_i} + \epsilon_0 / \sigma_0} \cdot \left( \mSigma^{\sfrac{1}{2}}(\vtheta) \right)_{i,j}, & \sigma_{i,d+j}(\vtheta, \vu) &:= 0, \\
    \sigma_{d+i,j}(\vtheta, \vu) &:= 0, & \sigma_{d+i,d+j}(\vtheta, \vu) &:= 0,
\end{align*}
for all $1 \le i, j \le d$.
\end{definition}

\begin{theorem} \label{thm:app_rmsprop_sde}	
Fix constants $\sigma_0, c_2 > 0$, $\epsilon_0 \ge 0$.
Let $T > 0$, $\eta^2\in(0,1\land T \land \frac{1}{2c_2})$ and set $N = \lfloor T/\eta^2\rfloor$. 
Let $\vu_k\triangleq\vv_k/\sigma^2$ and $\vx_k\triangleq (\vtheta_k, \vu_k)\in\sR^{2d}$. 
Let $\{\vx_k:k\geq 0\}$ be the discrete RMSprop iterations defined in \Cref{def:rmsprop},
where 
$\sigma, \epsilon, \beta$ are set
so that $\sigma_0 = \sigma\eta$, $\epsilon_0 = \epsilon \eta$ and $c_2 = (1-\beta)/\eta^2$. 
For well-behaved NGOS that satisfies the bounded moments and low skewness conditions,
the SDE as defined in \Cref{def:rmsprop-sde-matrix}
is an order-1 weak approximation (\Cref{def:weak_approx}) of discrete RMSprop,
if they start with $\vX_0 = \vx_0$.
\end{theorem}
The basic idea is to apply the general theorem (\Cref{thm:app_many_step}) with $\eeta := \eta^2$.
However, the SDE above does not satisfy Condition (1) in
\Cref{thm:app_many_step} because the denominators such as $\sigma_0 \sqrt{u_i} +
\epsilon_0$ can be unbounded. To solve this issue, the first step is to reduce \Cref{thm:app_rmsprop_sde} to proving the order-1 weak approximation for the following auxiliary SDE:
\begin{definition}
	Define $\tau: \R \to \R$ to be the following smooth transition function:
	\begin{equation}
		\tau(z) = \begin{cases}
			1 & \quad \text{if } z \ge 1, \\
			\frac{e^{-1/z}}{e^{-1/z} + e^{-1/(1-z)}} & \quad \text{if } z \in (0, 1), \\
			0 & \quad \text{if } z \le 0.
		\end{cases}
	\end{equation}
\end{definition}

\begin{definition}[Auxiliary SDE for RMSprop, matrix form] \label{def:rmsprop-aux-sde-matrix}
For constants $\sigma_0$, $\epsilon_0$, $c_2$ and $u_{\min}$, define
$\mu(u)$ as the following function
\begin{equation}
	\mu(u) := \tfrac{1}{2}u_{\min} + \tau(\tfrac{2 u}{u_{\min}} - 1) \cdot (u - \tfrac{1}{2} u_{\min}),
\end{equation}
and define the SDE as $\dd \vX_t = \vb(\vX_t) \dd t + \msigma(\vX_t) \dd \vW_t$,
where
$\vX_t \in \R^d \times \R^d$, and
$\vb$ and $\msigma$ are defined by
\begin{align*}
    b_i(\vtheta, \vu) &:= -\tfrac{1}{\sigma_0 \sqrt{\mu(u_i)} + \epsilon_0} \cdot \partial_i f(\vtheta), & b_{d+i}(\vtheta, \vu) &:= c_2(\Sigma(\vtheta)_{i,i} - u_i). \\
    \sigma_{i,j}(\vtheta, \vu) &:= \tfrac{1}{\sqrt{\mu(u_i)} + \epsilon_0 / \sigma_0} \cdot \left( \mSigma(\vtheta)^{1/2} \right)_{i,j}, & \sigma_{i,d+j}(\vtheta, \vu) &:= 0, \\
    \sigma_{d+i,j}(\vtheta, \vu) &:= 0, & \sigma_{d+i,d+j}(\vtheta, \vu) &:= 0.
\end{align*}
for all $1 \le i, j \le d$.
\end{definition}

\begin{theorem} \label{thm:app_rmsprop_aux_sde}	
In the setting of \Cref{thm:app_rmsprop_sde}, let $u_{\min} = 2^{-c_2 T} \min_{i \in [d]} u_{0,i}$
then the SDE defined by \Cref{def:rmsprop-aux-sde-matrix} is an order-1 weak approximation (\Cref{def:weak_approx}) of discrete RMSprop (\Cref{def:rmsprop}),
if they start with $\vX_0 = \vx_0$.
\end{theorem}

\begin{proof}[Proof for \Cref{thm:app_rmsprop_sde}]
	Given \Cref{thm:app_rmsprop_aux_sde},
	we only need to show that $\vX_t$ has the same distribution in the original and auxiliary SDEs for all $t \in [0, T]$, when $\vX_0 = (\vtheta_0, \vu_0)$.
	To see this, we only need to note that in the original SDE
	\begin{align*}
		\frac{\dd u_{t,i}}{\dd t} = c_2 (\Sigma(\vtheta_t)_{i, i} - u_{t,i}) \ge - c_2 u_{t,i}.
	\end{align*}
	Thus, $u_{t, i} \ge \exp(-c_2 t) u_{0,i} \ge u_{\min}$, which means $\Pr[u_{t,i} = \mu(u_{t,i})] = 1$ for all $t \in [0, T]$.
\end{proof}

It remains to prove \Cref{thm:app_rmsprop_aux_sde} by applying
\Cref{thm:app_many_step}. In the rest of this section, we verify the three
conditions in \Cref{thm:app_many_step} respectively.

Below we use the notations $\vx_k, \vX_t, \vb, \msigma$ defined as in \Cref{thm:app_rmsprop_aux_sde}.
Let $D = 2d$. Every $\vx_k \in \R^D$ is a concatenation of two $\R^d$-vectors $\vtheta_k$ and $\vu_k$.
According to the update rule of RMSprop, $\vx_k$ can be seen as SGA $\vx_{k+1} = \vx_k - \eeta \vh_k(\vx_k, \vz_k, \eeta)$,
where $\eeta = \eta^2$, $\vz_k \sim \DatZ_{\sigma}(\vtheta_k)$, and $\vh_k$ is defined below:
\begin{align*}
	\vh_k(\vtheta, \vu, \vz, \eeta) := \begin{bmatrix}
		-(\nabla f(\vtheta) + \sigma \vz) \odot (\sigma_0\sqrt{\vu} + \epsilon_0)^{-1} \\
		c_2 \left((\nabla f(\vtheta) / \sigma + \vz)^2 - \vu \right)
	\end{bmatrix}.
\end{align*}
We define $\vDelta$ and $\tilde{\vDelta}$ as in \eqref{eq:def-delta} and \eqref{eq:def-tilde-delta}.
Fix $\vx_0 = (\vtheta_0, \vu_0)$ with $u_{0, j} > 0$ for all $j \in [d]$. Define $u_{\min}$ as in \Cref{thm:app_rmsprop_aux_sde}.
Let $\gX_k$ be the support of the random variable $\vx_k$ given $\vx_0$, then it is easy to show that $\gX_k$ is a subset of $\{(\vtheta, \vu) : u_j \ge u_{\min} \text{ for all } j \in [d]\}$.

\subsection{Verifying Condition (1)}
\begin{lemma} \label{lm:rmsprop-cond-1}
    The drift function $\vb$ and diffusion function $\msigma$ are Lipschitz and belong to $G^4$.
\end{lemma}
\begin{proof}
	$\mSigma^{1/2}(\vtheta)$ is
	bounded and Lipschitz, so $\mSigma(\vtheta)$ is Lipschitz.
	Note that the denominators in the fractions in the formulas of $\vb$ and $\msigma$ are
	always lower bounded by a constant. Then the
	Lipschitz property of $\vb$ and $\msigma$ can be implied by the Lipschitz
	property of $\nabla f(\vtheta)$, $\mSigma(\vtheta)$,
	$\mSigma^{1/2}(\vtheta)$, and $\vb, \msigma \in G^4$ can be implied by $\nabla f, \mSigma^{1/2}, \mu \in G^4$.
\end{proof}

\subsection{Verifying Condition (2)}

To verify Condition (2), we only need to compute the moments of $\vDelta$ and
$\tilde{\vDelta}$ for the discrete RMSprop and the auxiliary SDE, and show that they are close to each other.
We compute them by the following two lemmas.

\begin{lemma} \label{lem:rmsprop_discrete_moments}
	For $\vx = (\vtheta, \vu) \in \gX_k$,
	if the NGOS is well-behaved and $\DatZ_{\sigma}$ satisfies the bounded moments and low skewness condition,
	then the moments of $\vDelta := \vDelta(\vx, k) - \vx$ can be written as below.
	\begin{enumerate}
		\item For $1 \le i \le d$, the following holds for the first moments:
		\begin{align*}
			\E[\Delta_i] &= - \frac{\eta^2}{\sigma_0 \sqrt{u_i} + \epsilon_0} \partial_i f(\vtheta), &
			\E[\Delta_{d+i}] &= \eta^2c_2\left(\Sigma_{ii}(\vtheta) - u_i\right) + \frac{\eta^4 c_2}{\sigma_0^2} (\partial_i f(\vtheta))^2 \\
			& & &= \eta^2c_2\left(\Sigma_{ii}(\vtheta) - u_i\right) + \gO(\eta^4).
		\end{align*}
		\item For $1 \le i,j \le d$, the following holds for the second moments:
		\begin{align*}
			\E[\Delta_i\Delta_j] &= \frac{\eta^2\Sigma_{ij}(\vtheta)}{(\sqrt{u_i} + \epsilon_0 / \sigma_0)(\sqrt{u_j} + \epsilon_0 / \sigma_0)} + \gO(\eta^4) &
			\E[\Delta_i\Delta_{d+j}] &= \gO(\eta^4) \\
			\E[\Delta_{d+i}\Delta_j] &= \gO(\eta^4) &
			\E[\Delta_{d+i}\Delta_{d+j}] &= \gO(\eta^4).
		\end{align*}
		for all $i, j \in [d]$.
		\item The third moments are bounded by $\E[\vDelta^{\otimes 3}] = \gO(\eta^4)$.
	\end{enumerate}
    Here the big-O notation $\gO(\,\cdot\,)$ is used in a way that $\gO(1)$ hides constants (independent of $\eta$ and $\vx$)
	and values that are bounded by a function of $\vx$ with polynomial growth.
\end{lemma}
\begin{proof}

We note that
\begin{align*}
	\Delta_i = -\frac{\eta^2}{\sigma_0 \sqrt{u_i} + \epsilon_0} \left(\partial_i f(\vtheta) + \sigma z_i \right), \quad\quad \Delta_{d+i} = (1-\beta)\left((\partial_i f(\vtheta) / \sigma + z_i)^2 - u_i\right).
\end{align*}
Let $\nu_i := \frac{1}{\sigma_0 \sqrt{u_i} + \epsilon_0}$. Since $\vx \in \gX_k$, $\nu_i \le \frac{1}{\sigma_0 \sqrt{u_{\min}} + \epsilon_0} = \gO(1)$.
Writing $1-\beta$ as $c_2 \eta^2$, we have
\begin{equation} \label{eq:rmsprop-delta-tilde-delta-formula}
	\Delta_i = - \nu_i \eta^2\left(\partial_i f(\vtheta) + \sigma z_i \right), \quad\quad \Delta_{d+i} = c_2 \eta^2 \left((\partial_i f(\vtheta) / \sigma + z_i)^2 - u_i\right).
\end{equation}
We can now compute the first moments:
\begin{align*}
	\E[\Delta_i] &= -\nu_i\eta^2\partial_i f(\vtheta) \tag{$\E[z_i]=0$} \\ 
	\E[\Delta_{d+i}] &= c_2\eta^2\left(\E[(\partial_i f(\vtheta) / \sigma)^2 + \E[z_i^2] - \E[u_i]\right) \\ 
	&= c_2\eta^2 \left((\partial_i f(\vtheta)/\sigma)^2 + \Sigma_{ii} - u_i\right). 
\end{align*}
Now we observe that $1/\sigma = \gO(\eta)$, so we can write
\begin{equation*}	
	\E[\Delta_{d+i}] = \eta^2c_2(\Sigma_{ii}-u_i) + \gO(\eta^4).
\end{equation*}
Let $\vdelta := \vDelta - \E[\vDelta]$. That is,
\begin{align*}
	\delta_i &= - \nu_i \eta^2 \sigma z_i &
	\delta_{d+i} &= c_2 \eta^2 (z_i^2 - \Sigma_{ii} + 2 z_i \partial_i f(\vtheta) / \sigma) \\
	&= -\nu_i \sigma_0 \eta z_i. & 
	&= c_2 \eta^2 (z_i^2 - \Sigma_{ii}) + 2c_2 (\partial_i f(\vtheta) / \sigma_0) \eta^3 z_i.
\end{align*}
For convenience we also define $w_i = (z_i^2 - \Sigma_{ii}) + 2(\partial_i f(\vtheta) / \sigma_0) \eta z_i$ and write $\delta_{d+i} = c_2 \eta^2 w_i$.

For the second moments we have
\begin{align*}
	\E[\Delta_p \Delta_q] &= \E[\delta_p \delta_q] + \E[\Delta_p]\E[\Delta_q] = \E[\delta_p \delta_q] + \gO(\eta^4) & \text{for all } 1 \le p, q \le 2d.
\end{align*}
Then it suffices to compute the second moments for $\delta$.
For $\E[\delta_i \delta_j]$ we have
\begin{align*}
	\E[\delta_i\delta_j] &= \nu_i\nu_j\eta^4 \sigma^2 \E[z_i z_j] = \nu_i\nu_j\sigma_0^2  \Sigma_{ij}.
\end{align*}
For $\E[\delta_i\delta_{d+j}]$ we have
\begin{align*}
	\E[\delta_i\delta_{d+j}]
	&= -c_2 \nu_i \sigma_0 \eta^3 \E\left[z_iw_j\right] \\
	&= -c_2 \nu_i \sigma_0 \eta^3 \E[z_iz_j^2] + \gO(\eta^4) \\
	&= \E[z_iz_j^2] \cdot \gO(\eta^3) + \gO(\eta^4).
\end{align*}
Similarly, we have $\E[\Delta_{d+i} \Delta_j] = \E[z_i^2 z_j] \cdot \gO(\eta^3) + \gO(\eta^4)$.

For $\E[\delta_{d+i} \delta_{d+j}]$, we note that $\vz \sim \DatZ_{\sigma}(\vtheta)$ has bounded 4th-order moments,
so $\E[g(\vz)] = \gO(1)$ for any polynomial $g$ of degree at most $4$, if the coefficients of $g$ are bounded by $\gO(1)$. Then we have
\begin{align*}
	\E[\delta_{d+i}\delta_{d+j}] &= c_2^2 \eta^4 \E[w_i w_j] = \gO(\eta^4).
\end{align*}

Now we can check the third moments. 
\begin{align*}
	\E[\Delta_p \Delta_q \Delta_r] &= \E[\delta_p \delta_q \delta_r] + (\E[\delta_p \delta_q] \E[\Delta_r] + \E[\delta_p \delta_r] \E[\Delta_q] + \E[\delta_q \delta_r] \E[\Delta_p]) + \E[\Delta_p]\E[\Delta_q]\E[\Delta_r] \\
	&= \E[\delta_p \delta_q \delta_r] + \gO(\eta^2) \cdot \gO(\eta^2) + \gO(\eta^6) \\
	&= \E[\delta_p \delta_q \delta_r] + \gO(\eta^4).
\end{align*}
Note that $\delta_p \delta_q \delta_r$ is a polynomial of $\vz$.
For $p = i, q = j, r = k$, by the low skewness condition for $\DatZ_{\sigma}$ we have
\[
	\E[\delta_i \delta_j \delta_k] = -\nu_i^3 \sigma_0^3 \eta^3 \E[z_i z_j z_k] = K_3(\vtheta) / \sigma \cdot \gO(\eta^3) = \gO(\eta^4).
\]
Except the above case, it can be shown that $\delta_p \delta_q \delta_r$ is a polynomial with coefficients bounded by $\gO(\eta^4)$.
Combining this with the fact that $\vz \sim \DatZ_{\sigma}(\vtheta)$ has bounded moments of any order, we have $\E[\delta_p \delta_q \delta_r] = \gO(\eta^4)$.
\end{proof}

\begin{lemma}
	For $\vx = (\vtheta, \vu) \in \gX_k$,
	if the NGOS is well-behaved,
then the moments of $\tilde\vDelta := \tilde\vDelta(x, k)$ can be written as below.
\begin{enumerate}
	\item The first moments are given by
	\begin{align*}
		\E[\tilde\Delta_i] &= - \frac{\eta^2}{\sigma_0 \sqrt{\mu(u_i)} + \epsilon_0} \partial_i f(\vtheta), &
		\E[\tilde\Delta_{d+i}] &= \eta^2c_2\left(\Sigma_{ii}(\vtheta) - u_i\right) + \gO(\eta^4).
	\end{align*}
	for all $i \in [d]$.
	\item The second moments are given by
	\begin{align*}
		\E[\tilde\Delta_i\tilde\Delta_j] &= \frac{\eta^2\Sigma_{ij}(\vtheta)}{(\sqrt{\mu(u_i)} + \epsilon_0 / \sigma_0)(\sqrt{\mu(u_j)} + \epsilon_0 / \sigma_0)} + \gO(\eta^4) &
		\E[\tilde\Delta_i\tilde\Delta_{d+j}] &= \gO(\eta^4) \\
		\E[\tilde\Delta_{d+i}\tilde\Delta_j] &= \gO(\eta^4) &
		\E[\tilde\Delta_{d+i}\tilde\Delta_{d+j}] &= \gO(\eta^4).
	\end{align*}
	for all $i, j \in [d]$.
	\item The third moments are bounded by $\E[\tilde\vDelta^{\otimes 3}] = \gO(\eta^4)$.
\end{enumerate}
Here the big-O notation $\gO(\,\cdot\,)$ is used in a way that $\gO(1)$ hides constants (independent of $\eta$ and $\vx$)
and values that are bounded by a function of $\vx$ with polynomial growth.
\label{lem:rmsprop_sde_moments}
\end{lemma}
\begin{proof}
Applying \Cref{lm:tilde-delta} gives
\begin{align*}
    \E[\tilde\vDelta] &= \eta^2 \vb(\vx) + \gO(\eta^4), &
	\E[\tilde\vDelta\tilde\vDelta^\top] &= \eta^2 \msigma(\vx) \msigma(\vx)^\top + \gO(\eta^4), &
    \E[\tilde\vDelta^{\otimes 3}] &= \gO(\eta^4).
\end{align*}
Splitting up the formula by indices proves the claim.
\end{proof}

\subsection{Verifying Condition (3)}

\begin{lemma}
	Let $P := \{1, 2, \dots, d\}$. Then
	\begin{enumerate}
		\item There is a constant $C_1 > 0$ (independent of $\eeta$) so that
			for all $k \le N$ and $\vx \in \gX_k$,
			\begin{align*}
				\normPsm{\E\vDelta(\vx, k)} &\le C_{1} \eeta (1 + \normPsm{\vx}), \\
				\normRsm{\E\vDelta(\vx, k)} &\le C_{1} \eeta
				(1 + \normPsm{\vx}^2)
				(1 + \normRsm{\vx}),
			\end{align*}
		\item For all $m \ge 1$, there is a constant $C_{2m}$ (independent of $\eeta$) so that
			for all $k \le N$ and $\vx \in \gX_k$,
			\begin{align*}
				\E\normPsm{\vDelta(\vx, k)}^{2m} &\le C_{2m} \eeta^{m} (1 + \normPsm{\vx}^{2m}), \\
				\E\normRsm{\vDelta(\vx, k)}^{2m} &\le C_{2m} \eeta^{m}
				(1 + \normPsm{\vx}^{4m})
				(1 + \normRsm{\vx}^{2m}),
			\end{align*}
	\end{enumerate}
\end{lemma}
\begin{proof}
	By \Cref{lem:rmsprop_discrete_moments}, for all $i \in [d]$,
	\begin{align*}
			\E[\Delta_i] &= - \frac{\eta^2}{\sigma_0 \sqrt{u_i} + \epsilon_0} \partial_i f(\vtheta), &
			\E[\Delta_{d+i}] &= \eta^2c_2\left(\Sigma_{ii}(\vtheta) - u_i\right) + \frac{\eta^4 c_2}{\sigma_0^2} (\partial_i f(\vtheta))^2.
	\end{align*}
	Combining this with the Lipschitzness of $\nabla f(\vtheta)$ and the boundedness of $\mSigma(\vtheta)$ proves Item 1.

	By \eqref{eq:rmsprop-delta-tilde-delta-formula}, for all $i \in [d]$,
	\begin{align*}
		\abssm{\Delta_i} &= \nu_i \eta^2 \abs{\partial_i f(\vtheta) + \sigma z_i}
		\le \eta^2 \nu_i (1 + \abssm{\partial_i f(\vtheta)})(1 + \sigma \abssm{z_i})
		\le \eta \nu_i (1 + \abssm{\partial_i f(\vtheta)})(1 + \sigma_0 \abssm{z_i})
		\\
		\abssm{\Delta_{d+i}}     &= c_2 \eta^2 \abs{(\partial_i f(\vtheta) / \sigma + z_i)^2 - u_i}
		\le c_2 \eta^2 (1 + (\partial_i f(\vtheta) / \sigma + z_i)^2)(1 + u_i)
	\end{align*}
	By the Lipschitzness of $\nabla f(\vtheta)$ and the bounded moments
	condition for $\DatZ_{\sigma}$, we can prove Item 2 by taking powers and
	expectations on both sides of the above inequalities.
\end{proof}

\section{Adam SDE Proof}

In this section, we prove the theorem for the SDE approximation of Adam.
\begin{definition}[SDE for Adam, matrix form] \label{def:adam-sde-matrix}
For constants $\sigma_0$, $\epsilon_0$, $c_1$ and $c_2$, define
$\gamma_1(t) := 1 - e^{-c_1 t}, \gamma_2(t) := 1 - e^{-c_2 t}$
and define the SDE as $\dd \vX_t = \vb(\vX_t) \dd t + \msigma(\vX_t) \dd \vW_t$,
where $\vX_t \in \R^d \times \R^d \times \R^d$, $\vb$ is defined by
\begin{align*}
    b_i(\vx, t) &:= -\tfrac{\sqrt{\gamma_2(t)}}{\gamma_1(t)} \cdot \tfrac{m_i}{\sigma_0 \sqrt{u_i} + \epsilon_0 \sqrt{\gamma_2(t)}}, \\
	b_{d+i}(\vx, t) &:= c_1(\partial_i f(\vtheta) - m_i), \\
	b_{2d+i}(\vx, t) &:= c_2(\Sigma(\vtheta)_{i,i} - u_i),
\end{align*}
for all $1 \le i \le d$, and $\sigma_{d+i, d+j}$ for all $1 \le i, j \le d$ is given by
\begin{align*}
	\sigma_{d+i,d+j}(\vx, t) &:= \sigma_0 c_1 \left( \mSigma^{1/2}(\vtheta) \right)_{i,j},
\end{align*}
and all the other entries of $\msigma$ are zero. 
\end{definition}

\begin{theorem} \label{thm:app_adam_sde}
Fix $\sigma_0, c_1, c_2 > 0$, $\epsilon_0 \ge 0$.
Let $T > 0$, $\eta^2\in(0,1\land T \land \frac{1}{2c_2})$ and set $N = \lfloor T/\eta^2\rfloor$. 
Let $\vu_k\triangleq\vv_k/\sigma^2$ and $\vx_k\triangleq (\vtheta_k, \vm_k, \vu_k)\in\sR^{3d}$. 
Let $\{\vx_k:k\geq 0\}$ be the discrete Adam iterations defined in \Cref{def:adam}.
Set $\sigma, \epsilon, \beta_1, \beta_2$ so that $\sigma_0 = \sigma\eta$, $\epsilon_0 = \epsilon \eta$, $c_1 = (1-\beta_1)/\eta^2$ and $c_2 = (1-\beta_2)/\eta^2$. 
For well-behaved NGOS that satisfies the bounded moments and low skewness conditions,
for any constant $t_0 > 0$, the solution $\vX_t$ ($t \in [t_0, T]$) of the SDE defined as in \Cref{def:adam-sde-matrix}
is an order-1 weak approximation (\Cref{def:weak_approx}) of
the sequence of Adam iterates $\vx_k$ starting from $k_0 = \lceil t_0 / \eta^2 \rceil$,
if the initial condition of the SDE is set to $\vX_{t_0} = \vx_{k_0}$.
\end{theorem}

The proof strategy is essentially the same as that for RMSprop. Similar to what we have done for RMSprop, we turn to prove the approximation order
for the following auxiliary SDE. 
\begin{definition}[Auxiliary SDE for Adam, matrix form] \label{def:adam-aux-sde-matrix}
For constants $\sigma_0$, $\epsilon_0$, $c_1$, $c_2$ and $u_{\min}$, define
$\gamma_1(t) := 1 - e^{-c_1 t}, \gamma_2(t) := 1 - e^{-c_2 t}$,
$\mu(u) := \tfrac{1}{2}u_{\min} + \tau(\tfrac{2 u}{u_{\min}} - 1) \cdot (u - \tfrac{1}{2} u_{\min})$,
and define the SDE as $\dd \vX_t = \vb(\vX_t) \dd t + \msigma(\vX_t) \dd \vW_t$,
where $\vX_t \in \R^d \times \R^d \times \R^d$, $\vb$ is defined by
\begin{align*}
    b_i(\vx, t) &:= -\tfrac{\sqrt{\gamma_2(t)}}{\gamma_1(t)} \cdot \tfrac{m_i}{\sigma_0 \sqrt{\mu(u_i)} + \epsilon_0 \sqrt{\gamma_2(t)}}, \\
	b_{d+i}(\vx, t) &:= c_1(\partial_i f(\vtheta) - m_i), \\
	b_{2d+i}(\vx, t) &:= c_2(\Sigma(\vtheta)_{i,i} - u_i),
\end{align*}
for all $1 \le i \le d$, and $\sigma_{d+i, d+j}$ for all $1 \le i, j \le d$ is given by
\begin{align*}
	\sigma_{d+i,d+j}(\vx, t) &:= \sigma_0 c_1 \left( \mSigma^{1/2}(\vtheta) \right)_{i,j},
\end{align*}
and all the other entries of $\msigma$ are zero. 
\end{definition}

\begin{theorem} \label{thm:app_adam_aux_sde}	
In the setting of \Cref{thm:app_adam_sde}, let $u_{\min} = 2^{-(T-t_0)} \min_{i \in [d]} u_{k_0,i}$,
For any constant $t_0 > 0$, the solution $\vX_t$ ($t \in [t_0, T]$) of the SDE defined as in \Cref{def:adam-aux-sde-matrix}
is an order-1 weak approximation (\Cref{def:weak_approx}) of
the sequence of Adam iterates $\vx_k$ starting from $k_0 = \lceil t_0 / \eta^2 \rceil$,
if the initial condition of the SDE is set to $\vX_{t_0} = \vx_{k_0}$.
\end{theorem}
\begin{proof}[Proof for \Cref{thm:app_adam_sde}]
	Given \Cref{thm:app_adam_aux_sde},
	we only need to show that $\vX_t$ has the same distribution in the original and auxiliary SDEs for all $t \in [0, T]$, when $\vX_0 = (\vtheta_0, \vm_0, \vu_0)$.
	To see this, we only need to note that in the original SDE
	\begin{align*}
		\frac{\dd u_{t,i}}{\dd t} = c_2 (\Sigma(\vtheta_t)_{i, i} - u_{t,i}) \ge - c_2 u_{t,i}.
	\end{align*}
	Thus, $u_{t, i} \ge \exp(-c_2 t) u_{0,i} \ge u_{\min}$, which means $\Pr[u_{t,i} = \mu(u_{t,i})] = 1$ for all $t \in [0, T]$.
\end{proof}

It remains to prove \Cref{thm:app_adam_aux_sde} by applying
\Cref{thm:app_many_step}. In the rest of this section, we verify the three
conditions in \Cref{thm:app_many_step} respectively.

Below we use the notations $\vx_k, \vX_t, \vb, \msigma$ defined as in \Cref{thm:app_adam_aux_sde}.
Let $D = 3d$. Every $\vx_k \in \R^D$ is a concatenation of three $\R^d$-vectors $\vtheta_k$, $\vm_k$ and $\vu_k$.
According to the update rule of Adam, $\vx_k$ can be seen as SGA $\vx_{k+1} = \vx_k - \eeta \vh_k(\vx_k, \vz_k, \eeta)$,
where $\eeta = \eta^2$, $\vz_k \sim \DatZ_{\sigma}(\vtheta_k)$, and $\vh_k$ is defined below:
\begin{align*}
	\vh_k(\vtheta, \vu, \vz, \eeta) := \begin{bmatrix}
		-\frac{\sqrt{1-\beta_2^k}}{1-\beta_1^{k+1}} \vm \odot \left(\sigma_0\sqrt{\vu} + \epsilon_0 \sqrt{1-\beta_2^k}\right)^{-1} \\
		c_1(\nabla f(\vtheta) + \sigma \vz - \vm) \\
		c_2 \left((\nabla f(\vtheta) / \sigma + \vz)^2 - \vu \right)
	\end{bmatrix}.
\end{align*}
We define $\vDelta$ and $\tilde{\vDelta}$ as in \eqref{eq:def-delta} and \eqref{eq:def-tilde-delta}.

Fix $\vx_0 = (\vtheta_0, \vm_0, \vu_0)$ with $u_{0, j} > 0$ for all $j \in [d]$.
Define $u_{\min}$ as in \Cref{thm:app_adam_aux_sde}. Let $\gX_k$ be the support
of the random variable $\vx_k$ given $\vx_0$, then it is easy to show that
$\gX_k$ is a subset of $\{(\vtheta, \vm, \vu) : u_j \ge u_{\min} \text{ for all
} j \in [d]\}$.

\subsection{Verifying Condition (1)}

\begin{lemma}
    The drift function $\vb$ and diffusion function $\msigma$ are Lipschitz and belong to $G^4$.
\end{lemma}
\begin{proof}
	Same argument as for \Cref{lm:rmsprop-cond-1}.
\end{proof}

\subsection{Verifying Condition (2)}

Let $\vx = (\vtheta, \vm, \vu) \in \sR^{3d}$. For ease of notation, let
$\hat{\gamma}_1=1-\beta_1^{k+1}$ and $\hat{\gamma}_2= 1-\beta_2^k$. These are not constants
across time steps like the other constants, but they are deterministic and upper
bounded by constants for $k \ge t_0 / \eta^2$. 

To verify Condition 2, we only need to compute the moments of $\vDelta$ and
$\tilde{\vDelta}$ for the discrete Adam and the auxiliary SDE, and show that they are close to each other.
We compute them by the following two lemmas.

\begin{lemma} \label{lem:adam_discrete_moments}
	For $\vx = (\vtheta, \vu) \in \gX_k$,
	if the NGOS is well-behaved and $\DatZ_{\sigma}$ satisfies the bounded moments and low skewness condition,
	then the moments of $\vDelta := \vDelta(\vx, k)$ can be written as below.
	\begin{enumerate}
		\item The first moments are given by
		\begin{align*}
			\E[\Delta_i]
			&= -\frac{\sqrt{\hat{\gamma}_2}}{\hat{\gamma}_1} \cdot \frac{\eta^2}{\sigma_0 \sqrt{u_i} + \epsilon_0 \sqrt{\hat{\gamma}_2}}  (m_i + c_1 \eta^2 (\partial_i f(\vtheta) - m_i)) \\
			&= -\frac{\sqrt{\hat{\gamma}_2}}{\hat{\gamma}_1} \cdot \frac{\eta^2 m_i}{\sigma_0 \sqrt{u_i} + \epsilon_0 \sqrt{\hat{\gamma}_2}} + \gO(\eta^4) \\
			\E[\Delta_{d+i}] &= c_1 \eta^2 (\partial_i f(\vtheta) - m_i) \\
			\E[\Delta_{2d+i}] &= c_2 \eta^2 ((\partial_i f(\vtheta)/\sigma)^2 + \Sigma_{ii} - u_i) \\
			&= c_2 \eta^2 (\Sigma_{ii} - u_i) + \gO(\eta^4).
		\end{align*}
		for all $i \in [d]$.
		\item The second moments are given by
		\begin{align*}
			\E[\Delta_p \Delta_q] &= \begin{cases}
				c_1^2 \sigma_0^2 \eta^2 \Sigma_{ij} + \gO(\eta^4) & \qquad \text{if} \quad p = d+i, q = d+j \quad\text{for some}\quad i, j \in [d] \\
				\gO(\eta^4) & \qquad \text{otherwise.}
			\end{cases}
		\end{align*}
		for all $p, q \in [3d]$.
		\item The third moments are bounded by $\E[\vDelta^{\otimes 3}] = \gO(\eta^4)$.
	\end{enumerate}
    Here the big-O notation $\gO(\,\cdot\,)$ is used in a way that $\gO(1)$ hides constants (independent of $\eta$ and $\vx$)
	and values that are bounded by a function of $\vx$ with polynomial growth.
\end{lemma}
\begin{proof}
For notational convenience, we write $\Delta, \Sigma$ instead of $\Delta(\vx), \Sigma(\vtheta)$.
We note that
\begin{align*}
	\Delta_i &=
	-\frac{\sqrt{\hat{\gamma}_2}}{\hat{\gamma}_1} \cdot \frac{\eta^2 m_i + \eta^2 (1-\beta_1)(\partial_i f(\vtheta) + \sigma z_i - m_i)}{\sigma_0 \sqrt{u_i} + \epsilon_0 \sqrt{\hat{\gamma}_2}} \\
	\Delta_{d+i} &= (1-\beta_1)(\partial_i f(\vtheta) + \sigma z_i - m_i) \\
	\Delta_{2d+i} &= (1-\beta_2)((\partial_i f(\vtheta)/\sigma + z_i)^2 - u_i)
\end{align*}
Let $\nu_i := \frac{\sqrt{\hat{\gamma}_2}}{\hat{\gamma}_1} \cdot \frac{1}{\sigma_0 \sqrt{u_i} + \epsilon_0 \sqrt{\hat{\gamma}_2}}$. And we write $1-\beta = c_2 \eta^2$. Then we have
\begin{align}
	\begin{aligned}
	\Delta_i &= -\nu_i \eta^2 (m_i + c_1 \eta^2 (\partial_i f(\vtheta) + \sigma z_i - m_i)) \\
	\Delta_{d+i} &= c_1 \eta^2 (\partial_i f(\vtheta) + \sigma z_i - m_i) \\
	\Delta_{2d+i} &= c_2 \eta^2 ((\partial_i f(\vtheta)/\sigma + z_i)^2 - u_i)
	\end{aligned}
\label{eq:adam-delta-formula}
\end{align}
The first moments follow directly:
\begin{align*}
	\E[\Delta_i]
	&= -\nu_i \eta^2 (m_i + c_1 \eta^2 (\partial_i f(\vtheta) - m_i))
 	= -\nu_i \eta^2 m_i + \gO(\eta^4) \\
	\E[\Delta_{d+i}] &= c_1 \eta^2 (\partial_i f(\vtheta) - m_i) \\
	\E[\Delta_{2d+i}] &= c_2 \eta^2 ((\partial_i f(\vtheta)/\sigma)^2 + \Sigma_{ii} - u_i) = c_2 \eta^2 (\Sigma_{ii} - u_i) + \gO(\eta^4).
\end{align*}
Let $\vdelta := \vDelta - \E[\vDelta]$. That is,
\begin{align*}
	\delta_i &= -\nu_i c_1 \eta^4 \sigma z_i = -\nu_i c_1 \sigma_0 \eta^3 z_i \\
	\delta_{d+i} &= c_1 \eta^2 \sigma z_i = c_1 \sigma_0 \eta z_i \\
	\delta_{2d+i}
	&= c_2 \eta^2 (z_i^2 - \Sigma_{ii} + 2 z_i \partial_i f(\vtheta) / \sigma) \\
	&= c_2 \eta^2 (z_i^2 - \Sigma_{ii}) + 2 c_2 (\partial_i f(\vtheta) / \sigma_0) \eta^3 z_i.
\end{align*}
For convenience we also define $w_i = (z_i^2 - \Sigma_{ii}) + 2(\partial_i
f(\vtheta) / \sigma_0) \eta z_i$ and write $\delta_{d+i} = c_2 \eta^2 w_i$.

Similar as the proof for \Cref{lem:rmsprop_discrete_moments}, 
for the second moments we have
\begin{align*}
	\E[\Delta_p \Delta_q] &= \E[\delta_p \delta_q] + \E[\Delta_p]\E[\Delta_q] = \E[\delta_p \delta_q] + \gO(\eta^4) & \text{for all } 1 \le p, q \le 2d.
\end{align*}
Then it suffices to compute the second order moments for $\delta$. A direct computation gives the following:
\begin{align*}
	\E[\delta_i \delta_j] &= \nu_i \nu_j c_1^2 \sigma_0^2 \eta^6 \E[z_i z_j] = \gO(\eta^6). \\
	\E[\delta_i \delta_{d+j}] &= -\nu_i c_1^2 \sigma_0^2 \eta^4 \E[z_i z_j] = \gO(\eta^4). \\
	\E[\delta_i \delta_{2d+j}] &= -\nu_i c_1 c_2 \sigma_0 \eta^5 \E[z_i w_j] = \gO(\eta^5). \\
	\E[\delta_{d+i}\delta_{d+j}] &= c_1^2 \sigma_0^2 \eta^2 \E[z_i z_j] = c_1^2 \sigma_0^2 \eta^2 \Sigma_{ij}. \\
	\E[\delta_{d+i}\delta_{2d+j}] &= c_1 c_2 \sigma_0 \eta^3 \E[z_i w_j] = c_1 c_2 \sigma_0 \eta^3 \E[z_i z_j^2] + \gO(\eta^4). \\
	\E[\delta_{2d+i}\delta_{2d+j}] &= c_2^2 \eta^4 \E[w_i w_j] = \gO(\eta^4).
\end{align*}
Now we check the third moments.
Similar as the proof for \Cref{lem:rmsprop_discrete_moments},
$\E[\Delta_p \Delta_q \Delta_r] = \E[\delta_p \delta_q \delta_r] + \gO(\eta^4)$.
Note that $\delta_p \delta_q \delta_r$ is a polynomial of $\vz$.
For $p = d+i, q = d+j, r = d+k$, by the low skewness condition for $\DatZ_{\sigma}$ we have
\[
	\E[\delta_{d+i} \delta_{d+j} \delta_{d+k}] = c_1^3 \sigma_0^3 \eta^3 \E[z_i z_j z_k] = K_3(\vtheta)/\sigma \cdot \gO(\eta^3) = \gO(\eta^4).
\]
Except the above case, it can be shown that $\delta_p \delta_q \delta_r$ is a polynomial with coefficients bounded by $\gO(\eta^4)$.
Combining this with the fact that $\vz \sim \DatZ_{\sigma}(\vtheta)$ has bounded moments of any order, we have $\E[\delta_p \delta_q \delta_r] = \gO(\eta^4)$.
\end{proof}

\begin{lemma}
For $\vx = (\vtheta, \vu) \in \gX_k$,
if the NGOS is well-behaved,
then the moments of $\tilde\vDelta := \tilde\vDelta(\vx, k)$ can be written as below.
\begin{enumerate}
	\item The first moments are given by
	\begin{align*}
		\E[\Delta_i]
		&= -\frac{\sqrt{\hat{\gamma}_2}}{\hat{\gamma}_1} \cdot \frac{\eta^2 m_i}{\sigma_0 \sqrt{u_i} + \epsilon_0 \sqrt{\hat{\gamma}_2}} + \gO(\eta^4) \\
		\E[\Delta_{d+i}] &= c_1 \eta^2 (\partial_i f(\vtheta) - m_i) + \gO(\eta^4) \\
		\E[\Delta_{2d+i}]
		&= c_2 \eta^2 (\Sigma_{ii} - u_i) + \gO(\eta^4).
	\end{align*}
	for all $i \in [d]$.
	\item The second moments are given by
	\begin{align*}
		\E[\Delta_p \Delta_q] &= \begin{cases}
			c_1^2 \sigma_0^2 \eta^2 \Sigma_{ij} + \gO(\eta^4) & \qquad \text{if} \quad p = d+i, q = d+j \quad\text{for some}\quad i, j \in [d] \\
			\gO(\eta^4) & \qquad \text{otherwise.}
		\end{cases}
	\end{align*}
	for all $p, q \in [3d]$.
	\item The third moments are bounded by $\E[\vDelta^{\otimes 3}] = \gO(\eta^4)$.
\end{enumerate}
Here the big-O notation $\gO(\,\cdot\,)$ is used in a way that $\gO(1)$ hides constants (independent of $\eta$ and $\vx$)
and values that are bounded by a function of $\vx$ with polynomial growth.
\label{lem:adam_sde_moments}
\end{lemma}
\begin{proof}
Applying \Cref{lm:tilde-delta} gives
\begin{align*}
    \E[\tilde\vDelta] &= \eta^2 \vb(\vx) + \gO(\eta^4), &
	\E[\tilde\vDelta\tilde\vDelta^\top] &= \eta^2 \msigma(\vx) \msigma(\vx)^\top + \gO(\eta^4), &
    \E[\tilde\vDelta^{\otimes 3}] &= \gO(\eta^4).
\end{align*}
Noting that $\gamma_1(k\eeta) = \hat{\gamma}_1 + \gO(\eta^2)$
and 
$\gamma_2(k\eeta) = \hat{\gamma}_2 + \gO(\eta^2)$,
we can prove the claim.
\end{proof}

\subsection{Verifying Condition (3)}

\begin{lemma}
	Let $P := \{1, 2, \dots, 2d\}$. Then
	\begin{enumerate}
		\item There is a constant $C_1 > 0$ (independent of $\eeta$) so that
			for all $k \le N$ and $\vx \in \gX_k$,
			\begin{align*}
				\normPsm{\E\vDelta(\vx, k)} &\le C_{1} \eeta (1 + \normPsm{\vx}), \\
				\normRsm{\E\vDelta(\vx, k)} &\le C_{1} \eeta
				(1 + \normPsm{\vx}^2)
				(1 + \normRsm{\vx}).
			\end{align*}
		\item For all $m \ge 1$, there is a constant $C_{2m}$ (independent of $\eeta$) so that
			for all $k \le N$ and $\vx \in \gX_k$,
			\begin{align*}
				\E\normPsm{\vDelta(\vx, k)}^{2m} &\le C_{2m} \eeta^{m} (1 + \normPsm{\vx}^{2m}), \\
				\E\normRsm{\vDelta(\vx, k)}^{2m} &\le C_{2m} \eeta^{m}
				(1 + \normPsm{\vx}^{4m})
				(1 + \normRsm{\vx}^{2m}).
			\end{align*}
	\end{enumerate}
\end{lemma}
\begin{proof}
	By \Cref{lem:adam_discrete_moments}, for all $i \in [d]$,
	\begin{align*}
		\E[\Delta_i]
		&= -\nu_i \eta^2 (m_i + c_1 \eta^2 (\partial_i f(\vtheta) - m_i)) \\
		\E[\Delta_{d+i}] &= c_1 \eta^2 (\partial_i f(\vtheta) - m_i) \\
		\E[\Delta_{2d+i}] &= c_2 \eta^2 ((\partial_i f(\vtheta)/\sigma)^2 + \Sigma_{ii} - u_i).
	\end{align*}
	Combining this with the Lipschitzness of $\nabla f(\vtheta)$ and the boundedness of $\mSigma(\vtheta)$ proves Item 1.

	By \eqref{eq:adam-delta-formula}, for all $i \in [d]$,
	one can show that there exists a constant $\hat{C}$ such that
	\begin{align*}
		\abssm{\Delta_i} &\le \eta^2 \hat{C} (1 + \abssm{m_i} + \abssm{\partial_i f(\vtheta)})(1 + \abssm{z_i}) \\
		\abssm{\Delta_{d+i}} &\le \eta^2 \hat{C} (1 + \abssm{m_i} + \abssm{\partial_i f(\vtheta)}) (1 + \abssm{z_i}) \\
		\abssm{\Delta_{2d+i}} &\le \eta^2 \hat{C} (1 + \abssm{\partial_i f(\vtheta)}^2 + z_i^2) (1 + u_i)
	\end{align*}
	By the Lipschitzness of $\nabla f(\vtheta)$ and the bounded moments
	condition for $\DatZ_{\sigma}$, we can prove Item 2 by taking powers and
	expectations on both sides of the above inequalities.
\end{proof}
\section{Analysis of SVAG Operator}\label{sec:app_svag}

\begin{lemma}\label{lem:app_svag_first_two}
	Let $\gG_{\sigma} = (f, \mSigma, \DatZ_{\sigma})$ be a NGOS and
	$\widehat{\gG}_{\ell \sigma} = (f, \mSigma, \widehat{\DatZ}_{\ell \sigma})$
	be the NGOS after applying the SVAG operator with hyperparameter $\ell > 0$.
	Then $\widehat{\gG}_{\ell \sigma}$ is indeed an NGOS. That is,
	$\widehat{\DatZ}_{\ell \sigma}(\vtheta)$ is well-defined and has mean zero, covariance $\mSigma(\vtheta)$.
\end{lemma}
\begin{proof}
	Let $\widehat{\DatZ}_{\ell \sigma}(\vtheta)$ be the distribution of $\hat{\vz} := \frac{1}{\ell} \left(r_1(\ell) \vz_1 + r_2(\ell) \vz_2\right)$ when $\vz_1, \vz_2 \sim \DatZ_{\sigma}(\vtheta)$.
	Then it is easy to check that $\hat{\vg}$ has the same distribution as $\nabla f(\vtheta) + \ell \sigma \hat{\vz}$, since
	\begin{align*}
		\hat{\vg} = r_1(\ell) \vg_1 + r_2(\ell) \vg_2 
		&\sim r_1(\ell) (\nabla f(\vtheta) + \sigma \vz_1) + r_2(\ell) (\nabla f(\vtheta) + \sigma \vz_2) \\
		&= (r_1(\ell) + r_2(\ell))\nabla f(\vtheta) + \sigma (r_1(\ell) \vz_1 + r_2(\ell) \vz_2) \\
		&= \nabla f(\vtheta) + \ell \sigma \hat{\vz},
	\end{align*}
	where the last equality uses the fact that $r_1(\ell) + r_2(\ell) = 1$.
	Hence $\widehat{\DatZ}_{\ell \sigma}(\vtheta)$ is well-defined.

	Now we check the mean and covariance of $\hat{\vz} \sim \widehat{\DatZ}_{\ell \sigma}(\vtheta)$.
	By linearity of expectation and linearity of variance (for independent variables), we have
	\begin{align*}
		\E[\hat{\vz}] &= \frac{1}{\ell} \left(r_1(\ell) \E[\vz_1] + r_2(\ell) \E[\vz_2]\right) = \frac{1}{\ell} \left( \vzero + \vzero \right) = \vzero, \\
		\Cov(\hat{\vz}) &= \frac{1}{\ell^2} \left(r_1^2(\ell) \Cov(\vz_1) + r_2^2(\ell) \Cov(\vz_2)\right) = \frac{1}{\ell^2}(r_1^2(\ell) + r_2^2(\ell)) \mSigma(\vtheta) = \mSigma(\vtheta),
	\end{align*}
	where the last equality is due to $r_1^2(\ell) + r_2^2(\ell) = \left(\frac{1+\sqrt{2\ell^2-1}}{2}\right)^2 + \left(\frac{1-\sqrt{2\ell^2-1}}{2}\right)^2 = \frac{1 + 2 \ell^2 - 1}{2} = \ell^2$.
\end{proof}

\begin{lemma}\label{lem:app_svag_low_skew}
	Let $\gG_{\sigma} = (f, \mSigma, \DatZ_{\sigma})$ be a NGOS with scale $\sigma$.
	Applying the SVAG operator with hyperparameter $\ell \ge 1$,
	we obtain $\widehat{\gG}_{\hat{\sigma}} = (f, \mSigma, \widehat{\DatZ}_{\hat{\sigma}})$ with scale $\hat{\sigma} = \ell \sigma$.
	Fixing $\sigma$ and changing $\ell$ produces $\widehat{\gG}_{\hat{\sigma}}$ for all scales $\hat{\sigma} \ge \sigma$.
	If $\gG_{\sigma}$ is well-behaved and satisfies the bounded moments condition,
	then $\widehat{\gG}_{\hat{\sigma}}$ is also well-behaved and satisfies the bounded moments condition.
	Furthermore, $\widehat{\gG}_{\hat{\sigma}}$ satisfies the low-skewness condition.
\end{lemma}
\begin{proof}
The loss function $f$ and covariance function $\mSigma$ are not changed after applying the SVAG operator, so $\widehat{\gG}_{\hat{\sigma}}$ is well-behaved.

Now we verify the bounded moments condition.
Let $\hat{\vz} = \frac{1}{\ell} \left( r_1(\ell) \vz_1 + r_2(\ell) \vz_2 \right)$, where $\vz_1, \vz_2 \sim \DatZ_{\sigma}(\vtheta)$.
Then we have
\begin{align*}
	\E[\normtwosm{\hat{\vz}}^{2m}]^{\frac{1}{2m}} &\le \frac{1}{\ell}\left(
		\abssm{r_1(\ell)} \E[\normtwosm{\vz_1}^{2m}]^{\frac{1}{2m}}
		+ \abssm{r_2(\ell)} \E[\normtwosm{\vz_2}^{2m}]^{\frac{1}{2m}}
	\right) \\
	&\le \frac{1}{\ell} (\abssm{r_1(\ell)} + \abssm{r_2(\ell)}) \cdot \E[\normtwosm{\vz_1}^{2m}]^{\frac{1}{2m}} \\
	&\le (1+\sqrt{2}) \E[\normtwosm{\vz_1}^{2m}]^{\frac{1}{2m}}.
\end{align*}
By the bounded moments condition for $\gG_{\sigma}$, there exists a constant $C_{2m}$ such that
	$\E_{\vz \sim
	\DatZ_{\sigma}(\vtheta)}[\normtwosm{\vz}^{2m}]^{\frac{1}{2m}} \le
	C_{2m}(1 + \normtwosm{\vtheta})$ for all $\vtheta \in \R^d$.
So $\E[\normtwosm{\hat{\vz}}^{2m}]^{\frac{1}{2m}} \le (1+\sqrt{2}) C_{2m}(1+\normtwosm{\vtheta})$ for all $\vtheta \in \R^d$.

We now verify the low skewness condition by showing that third moment of $\hat{\vz}$ is $\gO(1/\ell)$.
By the bounded moments condition with $m = 2$ and Jensen's inequality,
\begin{align*}
	\abs{\E[\vz_1^{\otimes 3}]} \le \E[\normtwosm{\vz_1}^{3}] \le \E[\normtwosm{\vz_1}^4]^{3/4} \le \left( C_4 (1+\normtwosm{\vtheta}) \right)^{3/4}.
\end{align*}
So $\abs{\E[\vz_1^{\otimes 3}]}$ is bounded by $\tilde{K}_3(\vtheta) := \left( C_4 (1+\normtwosm{\vtheta}) \right)^{3/4}$ of polynomial growth.

Let $r = \sqrt{2\ell^2 - 1}$. Then $r_1(\ell) = \frac{1}{2}(1+r), r_2(\ell) = \frac{1}{2}(1-r)$.
Since the third moments of two independent random vectors are additive,
\begin{align*}
	\E[\hat{\vz}^{\otimes 3}] &= \frac{1}{8\ell^3} \E\left[
		\left( (1+r) \vz_1 + (1-r) \vz_2 \right)^{\otimes 3}
	\right] \\
	&= \frac{1}{8\ell^3}\left(\E\left[
		(1+r)^3 \vz_1^{\otimes 3}\right] + \left[(1-r)^3 \vz_2^{\otimes 3}
	\right]\right) \\
	&= \frac{1}{8\ell^3}\left((1+r)^3 + (1-r)^3\right) \E\left[ \vz_1^{\otimes 3} \right] \\
	&= \frac{1}{8\ell^3} (2 + 6 r^2) \E[ \vz_1^{\otimes 3}] \\
	&= \frac{1}{8\ell^3} (12 \ell^2 - 4) \E[ \vz_1^{\otimes 3}] \le \frac{1}{\ell} \tilde{K}_3(\vtheta),
\end{align*}
where the 4th equality is due to $(1 + r)^3 + (1 - r)^3 = (1 + 3r + 3r^2 + r^3) + (1 - 3r + 3r^2 - r^3) = 2 + 6r^2$.
Therefore, the low skewness condition is verified.
\end{proof}

\section{Miscellaneous Theoretical Arguments}

\subsection{How does the noise scale change with batch size?}\label{sec:app_sigma_scaling}
We discuss how the noise scale $\sigma$ in the NGOS (\Cref{def:NGOS}) changes when the batch size changes: in particular, $\sigma\sim 1/\sqrt{B}$.
The argument follows from the linearity of covariance and is an already well-known result - we reproduce it here for clarity.
We first fix a parameter vector $\vtheta$.
Let $\vg^{(1)}, \dots, \vg^{(B)}$ be the gradients evaluated at data points from a batch of size $B$,
where $\vg^{(b)} = \nabla f(\vtheta) + \vz^{(b)}$, and every $\vz^{(b)}$ is a gradient noise vector drawn i.i.d.~from a distribution with mean $\vzero$ and covariance $\mSigma(\vtheta)$.
For sampling with replacement on a finite dataset of size $n$,
where $f_1(\vtheta), \dots, f_n(\vtheta)$ are the loss functions for the $n$ data points (and the average of these $n$ functions is $f(\vtheta)$),
this covariance matrix can be explicitly written as:
\begin{equation*}
	\mSigma(\vtheta) = \frac{1}{n}\sum_{i=1}^{n} (\nabla
f_{i}(\vtheta) - \nabla f(\vtheta))(\nabla f_{i}(\vtheta) - \nabla
f(\vtheta))^\top.
\end{equation*}
The average gradient over the batch is $\vg := \frac{1}{B} \sum_{b=1}^{B} \vg^{(b)} = \nabla f(\vtheta) + \frac{1}{B} \sum_{b=1}^{B} \vz^{(b)}$.
As $\vz^{(1)}, \dots, \vz^{(B)}$ are sampled i.i.d.,
their average
$\frac{1}{B}\sum_{b=1}^{B} \vz^{(b)}$ has mean $\vzero$ and covariance $\mSigma(\vtheta)$
by linearity of expectation and covariance.
We can set $\DatZ_{\sigma}(\vtheta)$
to be the distribution of 
the random variable $\frac{1}{\sqrt{B}}\sum_{b=1}^{B} \vz_b$, where $\sigma = \frac{1}{\sqrt{B}}$,
then $\vg$ has the same distribution as the stochastic gradient produced by the NGOS $\gG_{\sigma} = (f, \mSigma, \DatZ_{\sigma})$.

\subsection{What happens when the noise does not dominate the gradient?}\label{sec:app_grad_dominates_noise}
We discuss the linear warm-up setting described in \Cref{sec:linear_warmup}. 
Recall that when ignoring the effect of $\epsilon$, the RMSprop update can be written as
$$\vtheta_{k+1} \approx \vtheta_k - \eta \vg_k \odot (\bar{\vg}^2 + \sigma^2 \vone)^{-1/2}.$$
From the above equation, it is clear that the dynamics of $\vtheta$ depends on the relationship between the noise scale $\sigma$ and the gradient $\normsm{\bar\vg}$.
In \Cref{sec:linear_warmup}, we discuss the case where $\sigma \gg \normsm{\bar \vg}$, which is the regime where the SDE approximation can exist.

Here, we argue that when $\sigma \ll \normsm{\bar{\vg}}$, no SDE approximation can exist for the discrete trajectory.
In this case, the RMSprop update would instead be $\vtheta_{k+1} \approx \vtheta_k - \eta \mU^{-1} \vg_k$, where $\mU = \diag(\sqrt{\bar{\vg}^2})$.
Combining this with $\vg_k \sim \Normal(\bar{\vg}, \sigma \mI)$ yields that
$\vtheta_{k+1} - \vtheta_k \sim \Normal( \eta \mU^{-1} \bar{\vg}, \eta^2 \sigma^2 \mU^{-2})$ approximately.
We can again take a telescoping sum to obtain the marginal distribution of $\vtheta_k$:
$
	\vtheta_k \sim \Normal\left(
		k \eta \mU^{-1} \bar{\vg}, k \eta^2 \sigma^2 \mU^{-2}
	\right)
$
approximately.

However, it is impossible to make the above distribution fixed even as $\sigma$ changes, so no SDE approximation exists.
In particular, we need to make both $k \eta$ and $k \eta^2 \sigma^2$ fixed, so $\eta \sigma$ must be a constant. 
If $\sigma \ll\normsm{\bar{\vg}}$, then $\frac{1}{\eta} \ll \normsm{\bar{\vg}}$ too. 
This requirement on $\eta$ implies that no SDE approximation exists when
$\sigma \ll \normsm{\bar{\vg}}$, and hence motivates us to study the case of
$\sigma \gg \normsm{\bar{\vg}}$.
\section{Experimental Verification of Assumptions}
In this section, we take measurements and perform experiments to verify that the various assumptions made in our theory do not harm the applicability of our findings to realistic settings.

\subsection{Noise Dominates the Gradient}\label{sec:app_noise_dominates}
Our analysis in \Cref{sec:linear_warmup} suggests that an SDE approximation cannot exist when the the gradient $\bar{\vg}$ dominates the noise scale $\sigma$.
Note that \Cref{sec:linear_warmup} performs a rough analysis under the assumption of a linear loss function (i.e., fixed gradient throughout training), which is far from practice.
In the more general setting, we require that for every step $k$, $\E\|\vz_k\|^2$ (i.e., the gradient variance) dominates $\|\E \vg_k\|$ (i.e., the norm of the average gradient), where the expectations are taken over sampling seeds, in order for the SDE approximation to exist.
\Cref{fig:grad_vs_noise_rmsprop} shows that our assumption holds for small batches and for large batches near the end of training.

\begin{figure}
     \centering
     \begin{minipage}[t]{0.45\textwidth}
         \centering\includegraphics[width=\linewidth]{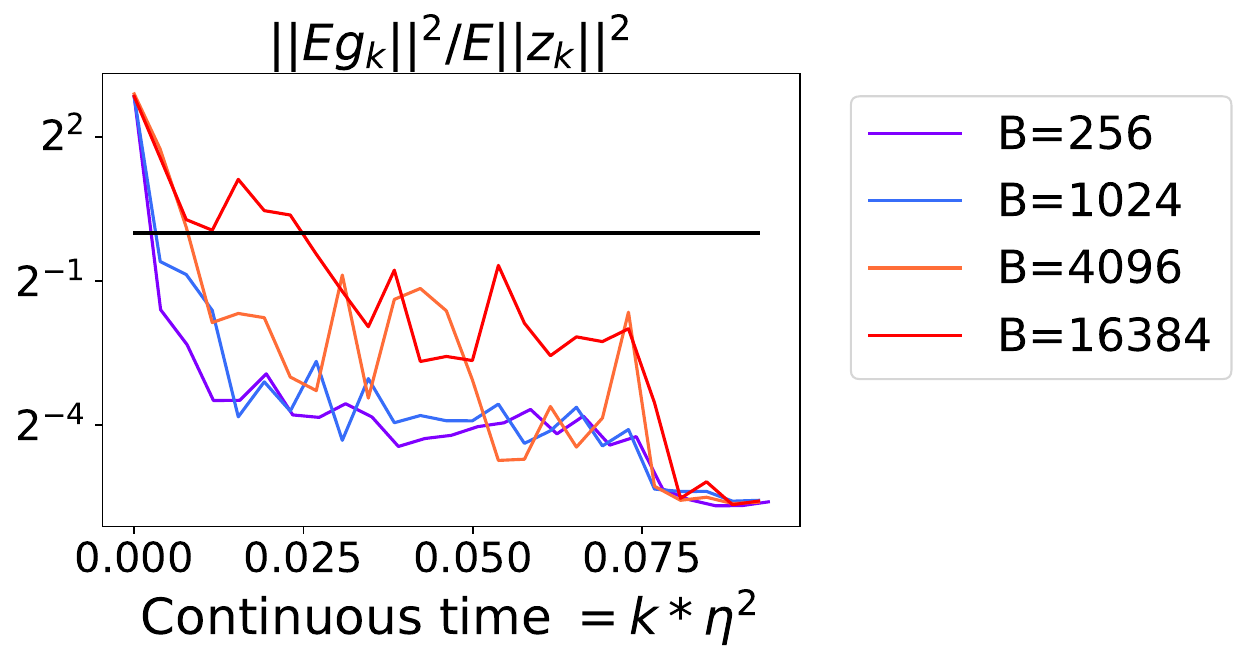}\\
        (a) ResNet-50 with $\epsilon$ order $10^{-30}$
     \end{minipage}
     \begin{minipage}[t]{0.45\textwidth}
         \centering\includegraphics[width=\linewidth]{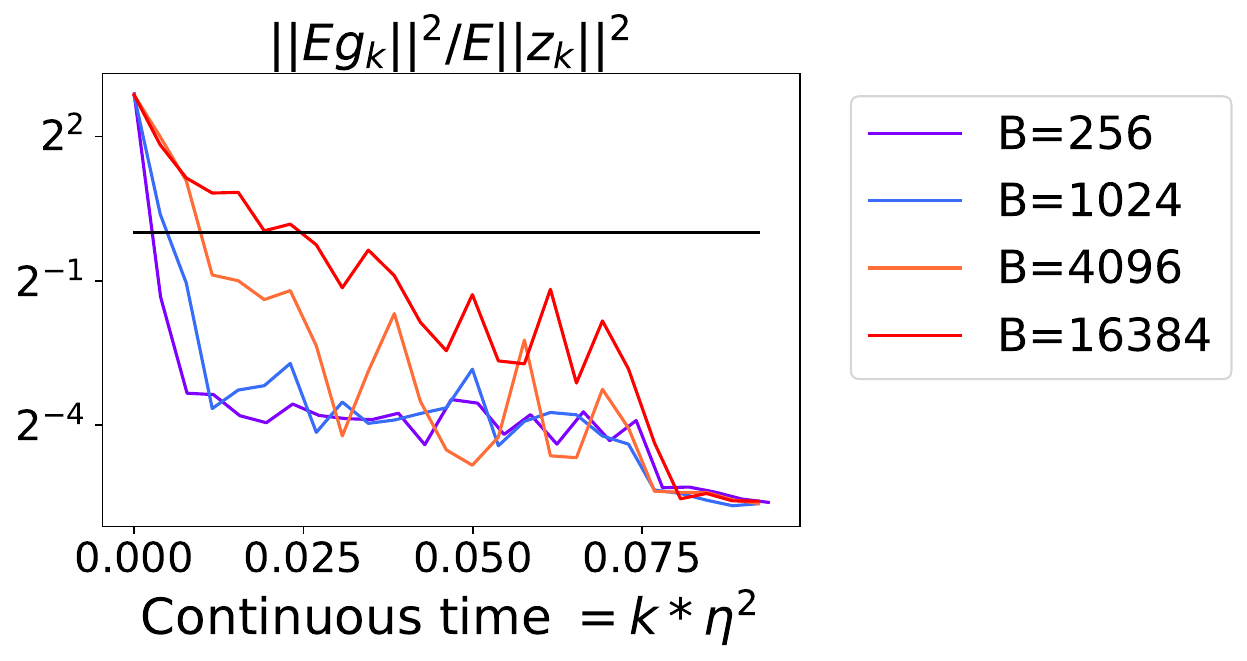} \\
         (b) ResNet-50 with $\epsilon$ order $10^{-8}$
     \end{minipage}
     \vfill
     \begin{minipage}[t]{0.45\textwidth}
         \centering\includegraphics[width=\linewidth]{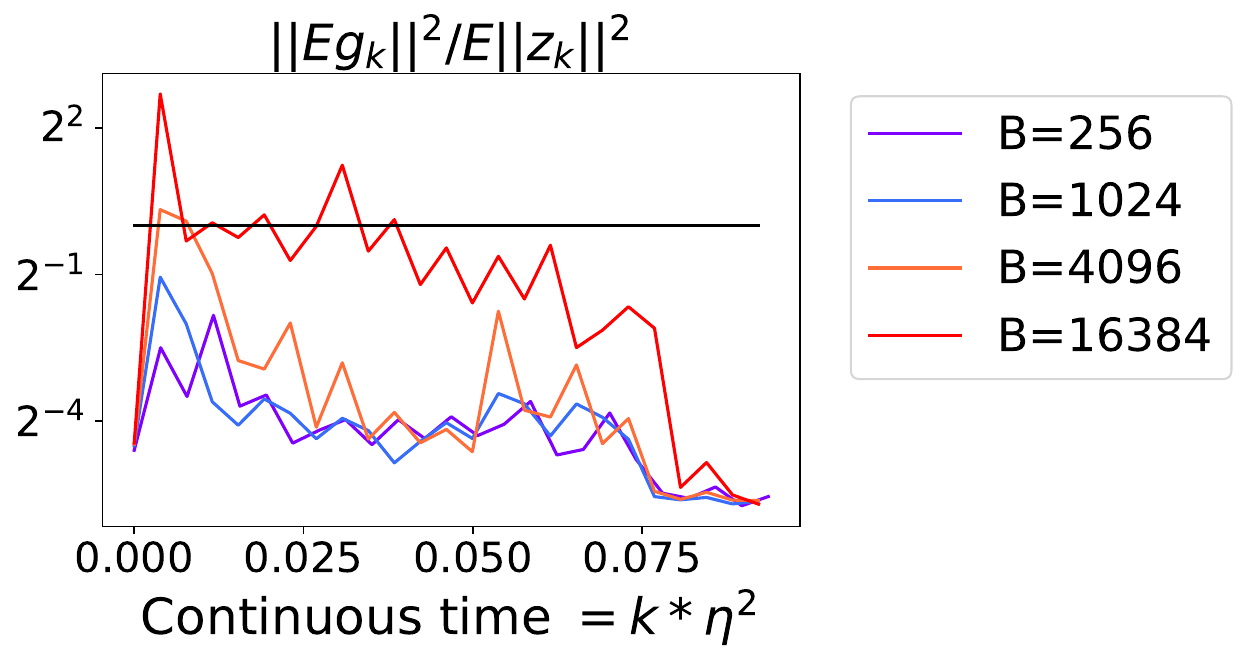} \\
        (c) VGG-16 with $\epsilon$ order $10^{-30}$
     \end{minipage}
     \begin{minipage}[t]{0.45\textwidth}
         \centering\includegraphics[width=\linewidth]{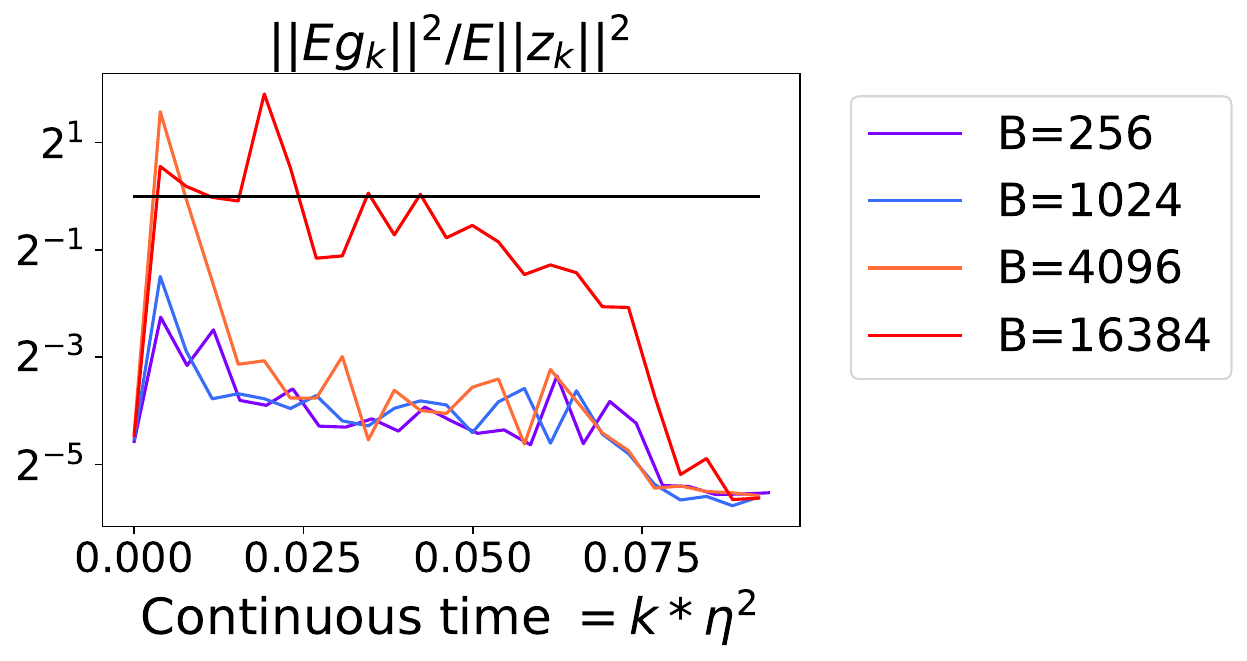} \\
         (d) VGG-16 with $\epsilon$ order $10^{-8}$
     \end{minipage}
  \caption{We compare the norm of the average gradient with the noise scale for different batch sizes during training of ResNet-50 and VGG-16 model with RMSprop on the CIFAR-10 dataset. Here, $(\eta, \beta) = (10^{-3}, 0.999)$ for batch size $256$ and scaled with our proposed square root scaling rule (\Cref{def:rmsprop_scaling}) for the other batch sizes. We show the results for $\epsilon$ at both small (of order $10^{-30}$) and large scale (of order $10^{-8}$). We observe that for small batches, the noise in the gradient dominates the signal in the gradient, supporting our hypothesis. For larger batches, the hypothesis seems to hold true towards the end of training. }
  \label{fig:grad_vs_noise_rmsprop}
\end{figure}

\begin{figure}
     \centering
     \begin{minipage}[t]{0.45\textwidth}
         \centering\includegraphics[width=\linewidth]{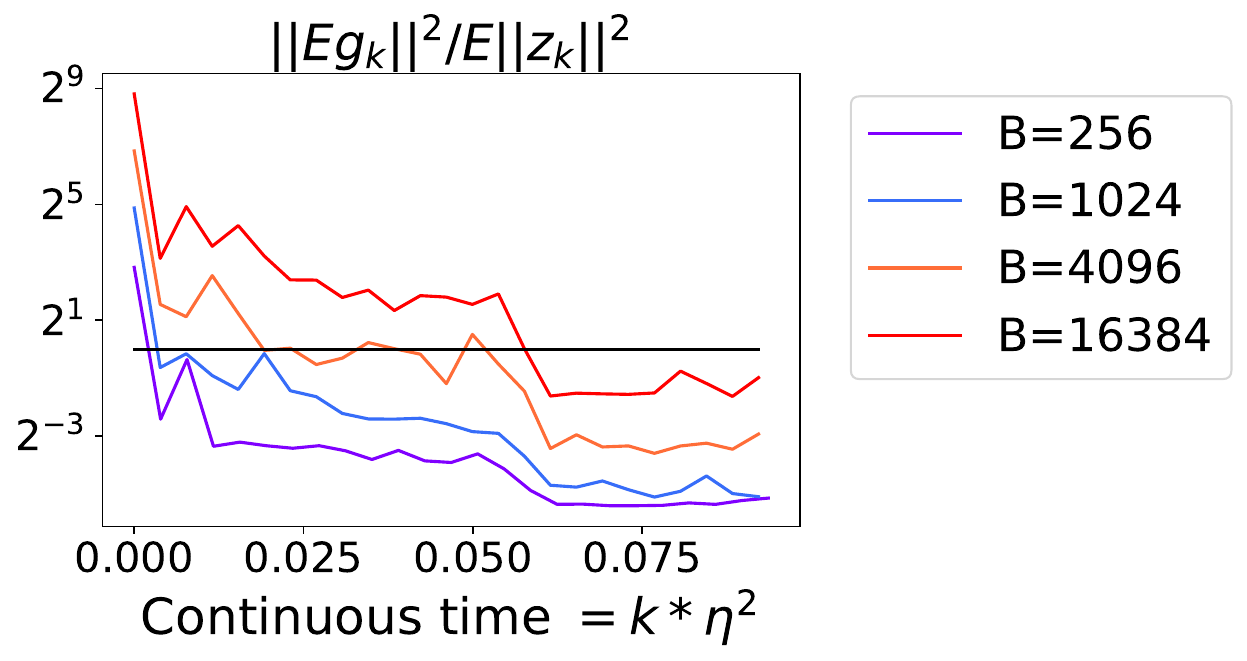}\\
        (a) ResNet-50 with $\epsilon$ order $10^{-30}$ 
     \end{minipage}
     \begin{minipage}[t]{0.45\textwidth}
         \centering\includegraphics[width=\linewidth]{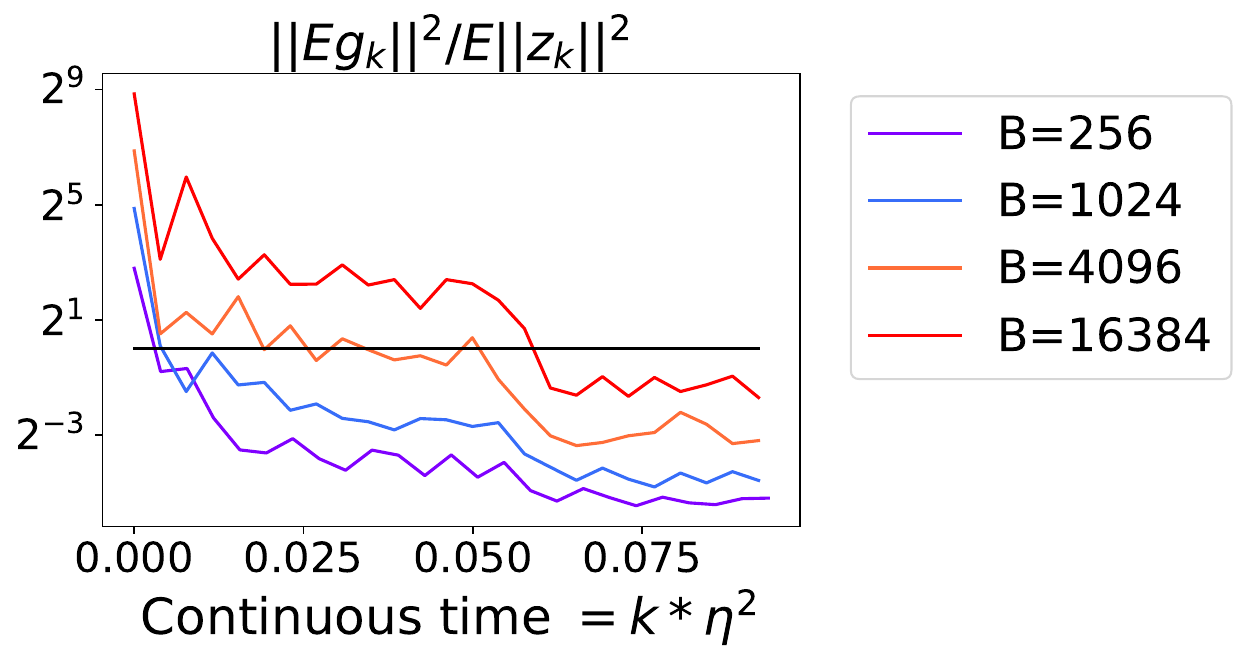}\\
        (b) ResNet-50 with $\epsilon$ order $10^{-8}$
     \end{minipage}
     \vfill\vfill
     \begin{minipage}[t]{0.45\textwidth}
         \centering\includegraphics[width=\linewidth]{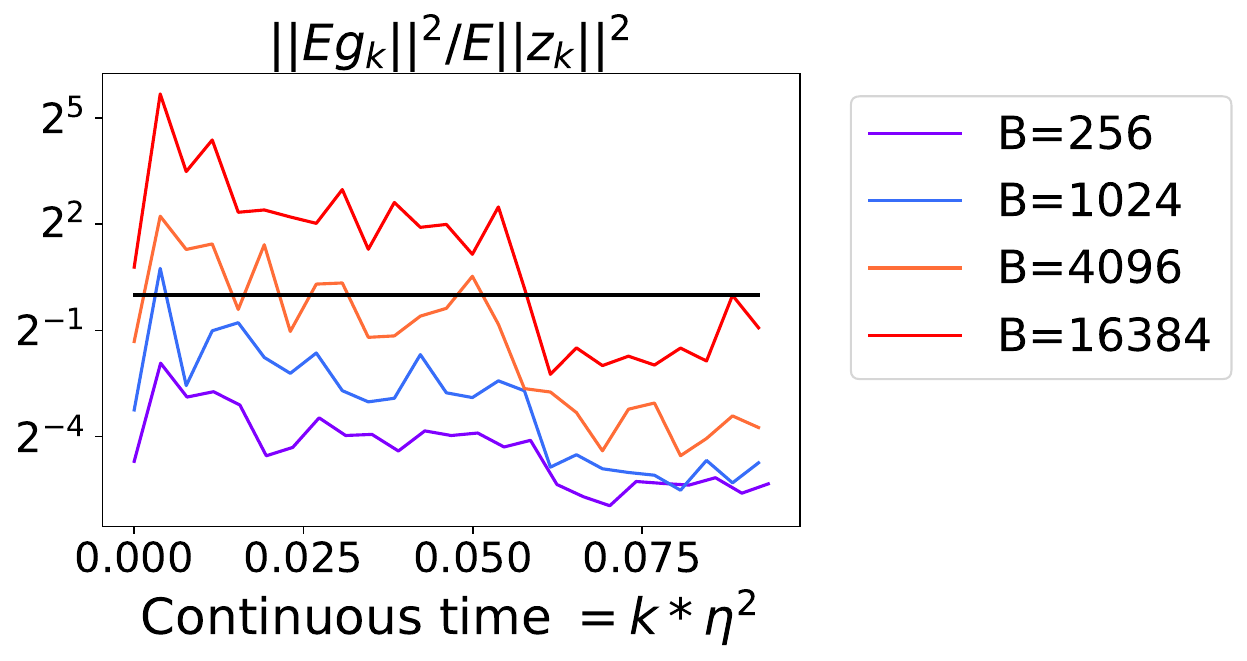}\\
        (c) VGG-16 with $\epsilon$ order $10^{-30}$
     \end{minipage}
     \begin{minipage}[t]{0.45\textwidth}
         \centering\includegraphics[width=\linewidth]{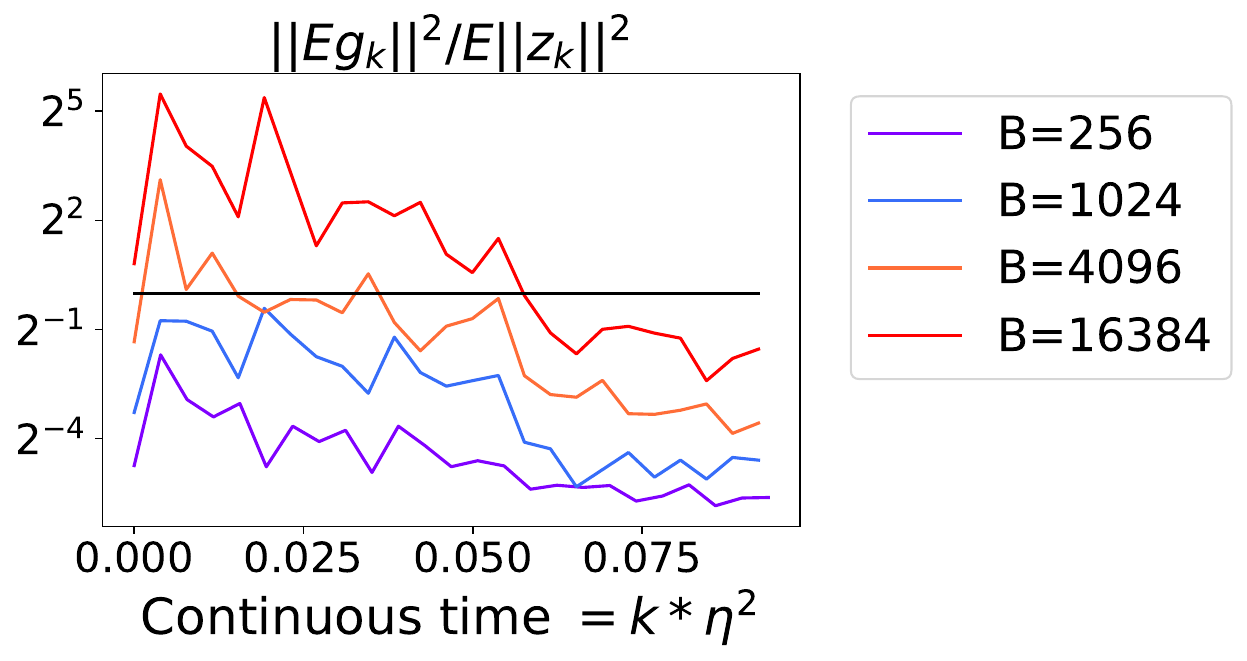}\\
        (d) VGG-16 with $\epsilon$ order $10^{-8}$
     \end{minipage}
  \caption{We compare the norm of the average gradient with the noise scale for different batch sizes during training of ResNet-50 model with Adam on the CIFAR-10 dataset. Here, $(\eta, \beta_1, \beta_2) = (10^{-3}, 0.999, 0.999)$ for batch size $256$ and scaled with our proposed square root scaling rule (\Cref{def:rmsprop_scaling}) for the other batch sizes. We show the results for $\epsilon$ at both small (of order $10^{-30}$) and large scale (of order $10^{-8}$). We observe that for small batches, the noise in the gradient dominates the signal in the gradient, supporting our hypothesis. For larger batches, the hypothesis seems to hold true towards the end of training. }
   \label{fig:grad_vs_noise_adam}
\end{figure}

\subsection{Using \texorpdfstring{$\vv_k$}{vvk} instead of \texorpdfstring{$\vv_{k+1}$}{vvk1} in the update rule}\label{sec:app_vk}
In \Cref{def:rmsprop,def:adam}, we slightly modify the standard implementation of RMSprop and Adam by using $\vv_k$ in the update rule instead of $\vv_{k+1}$.
Here, we verify for Adam that this modification of the optimization algorithms does not significantly harm performance. \cref{fig:prevstep_adam_largeeps} shows the behavior of ResNet-50 and VGG-16 trained with Adam on CIFAR-10 with the above modification of the optimization algorithm. We observe a small drop ($\approx 1\%$) in test accuracies. However, the behavior of the trajectories across different batch sizes for the proposed scaling rule stays the same, i.e. we observe a maximum of $3\%$ test accuracy gap between training batch size $256$ and $8192$. Moreover, the behavior of the test functions match across the trajectories of different batch sizes.

  \begin{figure}[!htbp]
  \centering
     \begin{minipage}[t]{0.9\textwidth}
         \centering\includegraphics[width=\linewidth]{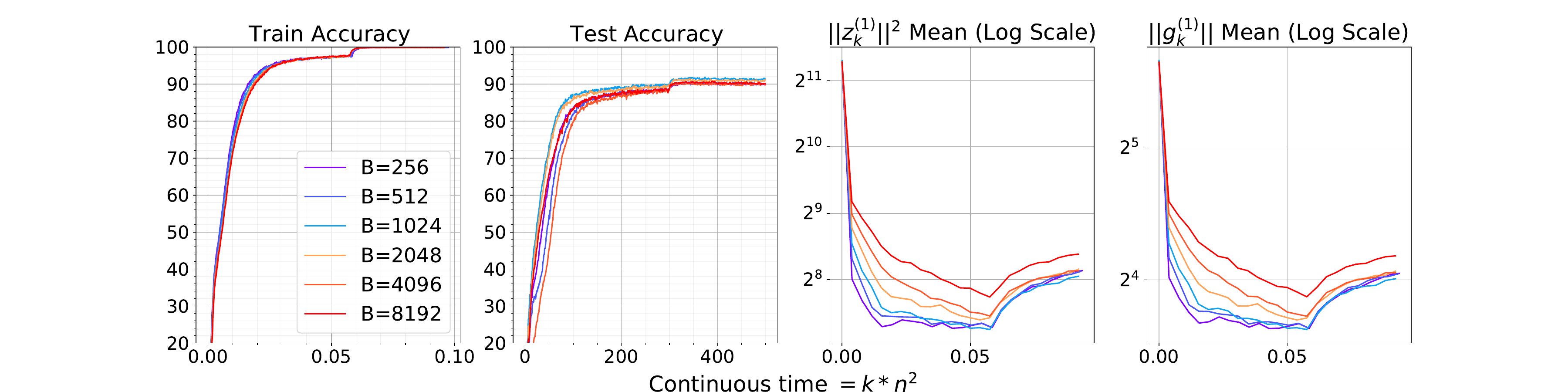}\\(a) ResNet-50
     \end{minipage}
     \begin{minipage}[t]{0.9\textwidth}
         \centering\includegraphics[width=\linewidth]{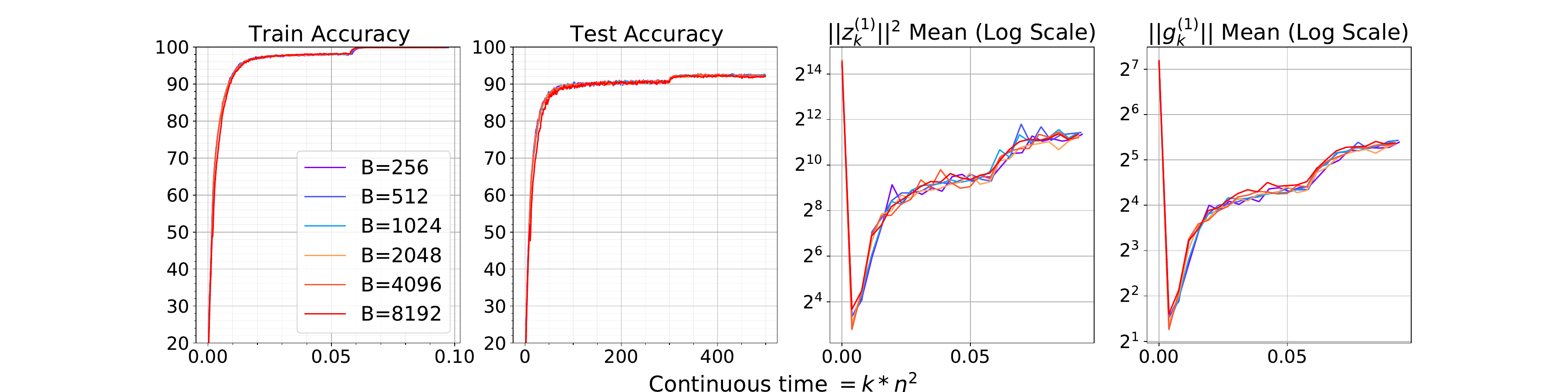}\\(b) VGG-16
     \end{minipage}
	\caption{We repeat the Square Root Scaling experiments with Adam for ResNet-50 and VGG-16 on CIFAR-10 dataset, where we slightly modify the standard implementation of Adam by using $\vv_k$ in the update rule instead of $\vv_{k+1}$. For batch size $256$, $(\eta, \epsilon, \beta_1, \beta_2) = (10^{-3}, 10^{-8}, 0.999, 0.999)$ and the hyperparameters are scaled according to the square root scaling rule for other batch sizes. We observe a small drop ($\approx 1\%$) in test accuracies. However, the behavior of the trajectories across different batch sizes stays the same.}
        \label{fig:prevstep_adam_largeeps}
  \end{figure}

\section{SVAG Experiments}\label{sec:app_exp_svag}

Recall that the SVAG algorithm (\Cref{def:svag_alg}) is a computationally efficient simulation of the SDEs corresponding to RMSprop and Adam.
The SVAG algorithm requires a hyperparameter $\ell$, and the resulting parameters after $k\ell^2$ steps should match the parameters on the corresponding discrete optimization trajectory after $k$ steps.
In particular, \Cref{thm:svag} shows that the SVAG algorithm is an order-1 weak approximation (\Cref{def:weak_approx}) of the SDE, and the approximation error scales as $1/\ell$.
One may be initially concerned that realistic deep learning settings require $1/\ell$ to be very small, which would make $\ell$ large and hence computationally intractable.
\cite{li2021validity} showed that the SVAG trajectories appear to converge to the SDE trajectory for computationally tractable small values of $\ell$.
We similarly find that our proposed SVAG-like algorithms in \Cref{def:svag_alg} appear to converge for small $\ell$ in various settings.

\paragraph{CIFAR-10.} \Cref{fig:rmspropsvag,fig:rmspropsvag1e8} show that SVAG converges and closely tracks RMSprop at smaller $\epsilon$ (=$10^{-30}$), and at larger $\epsilon$ (=$10^{-8}$) respectively. \Cref{fig:adamsvag,fig:Adamsvag1e8} show that SVAG converges and closely tracks Adam at smaller $\epsilon$ (=$10^{-30}$), and at larger $\epsilon$ (=$10^{-8}$) respectively. All experiments follow the setting in \Cref{sec:app_cifar10_config} for batch size 256.

\paragraph{Wikipedia + Books (Academic BERT).} \Cref{fig:sqrtscaling_roberta_svag} shows that SVAG converges and closely tracks Adam. We use the experimental setting for batch size $1024$ in \Cref{sec:app_books_wiki_config} except the hyperparameters $\beta_1$ and $\beta_2$ are fixed at $0.9$ and $0.98$ respectively.

\paragraph{WikiText-103 (GPT).} \Cref{fig:sqrtscaling_gpt} shows that SVAG converges and closely tracks Adam. We use the experimental setting for batch size $1024$ in \Cref{sec:app_books_wiki_config} except the hyperparameters $\beta_1$ and $\beta_2$ are fixed at $0.9$ and $0.98$ respectively. Additionally, for computational reasons, we pretrain on sequences of length $64$.

\begin{figure*}
 \begin{minipage}{1\textwidth}
      \centering
          \begin{minipage}{\linewidth}
          \begin{figure}[H]
              \centerline{\includegraphics[width=\linewidth]{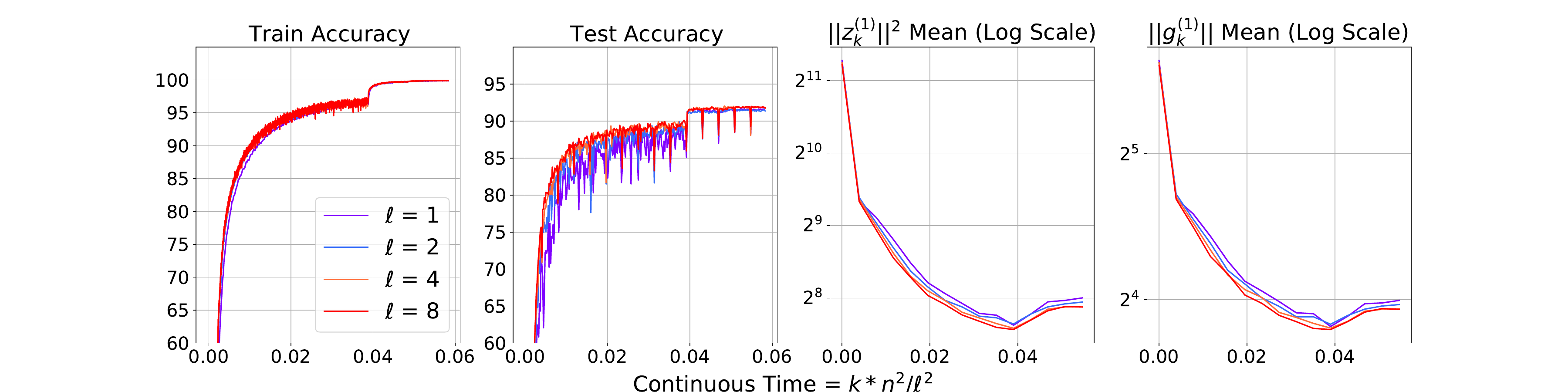}}
\caption{SVAG experiments on ResNet-50 trained on CIFAR-10 with RMSprop. We use batch size $256$ and the hyperparameters $\eta = 10^{-3}, \beta = 0.999, \epsilon = 10^{-30}$ and a weight decay of $10^{-4}$. Since SVAG takes $\ell$ smaller steps to simulate
the continuous dynamics in $\eta$ time, we plot accuracy against continuous time defined as $k \times \eta^2 / \ell^2$.}
            \label{fig:rmspropsvag}

          \end{figure}
      \end{minipage}
\end{minipage}
\end{figure*}
\begin{figure*}
 \begin{minipage}{1\textwidth}
      \centering
          \begin{minipage}{\linewidth}
          \begin{figure}[H]
              \centerline{\includegraphics[width=\linewidth]{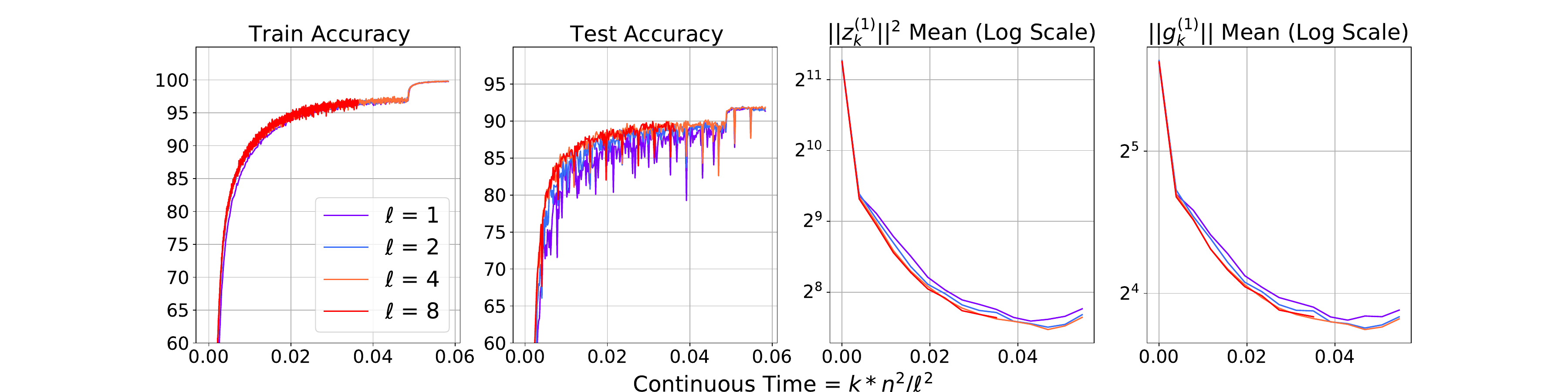}}
\caption{SVAG experiments on ResNet-50 trained on CIFAR-10 with RMSprop.  We use batch size $256$ and  the hyperparameters $\eta = 10^{-3}, \beta = 0.999, \epsilon = 10^{-8}$, and a weight decay of $10^{-4}$. Since SVAG takes $\ell$ smaller steps to simulate
the continuous dynamics in $\eta$ time, we plot accuracy against continuous time defined as $k \times \eta^2 / \ell^2$.}
            \label{fig:rmspropsvag1e8}
          \end{figure}
      \end{minipage}
\end{minipage}
\end{figure*}

\begin{figure*}
 \begin{minipage}{1\textwidth}
      \centering
          \begin{minipage}{\linewidth}
          \begin{figure}[H]
              \centerline{\includegraphics[width=\linewidth]{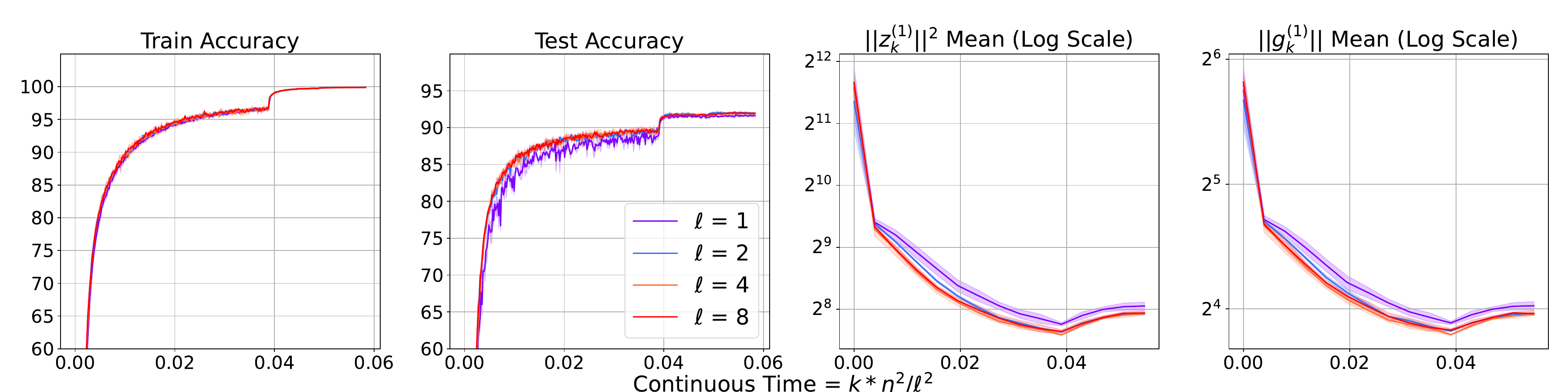}}
\caption{SVAG experiments on ResNet-50 trained on CIFAR-10 with Adam.  We use batch size $256$ and the hyperparameters $\eta = 10^{-3}, \beta_1 = 0.9, \beta_2 = 0.999, \epsilon = 10^{-30}$, and a weight decay of $10^{-4}$. Since SVAG takes $\ell$ smaller steps to simulate
the continuous dynamics in $\eta$ time, we plot accuracy against continuous time defined as $k \times \eta^2 / \ell^2$.}
            \label{fig:Adamsvag1e8}
          \end{figure}
      \end{minipage}
\end{minipage}
\end{figure*}

\begin{figure*}
\begin{center}
\centerline{\includegraphics[width=\linewidth]{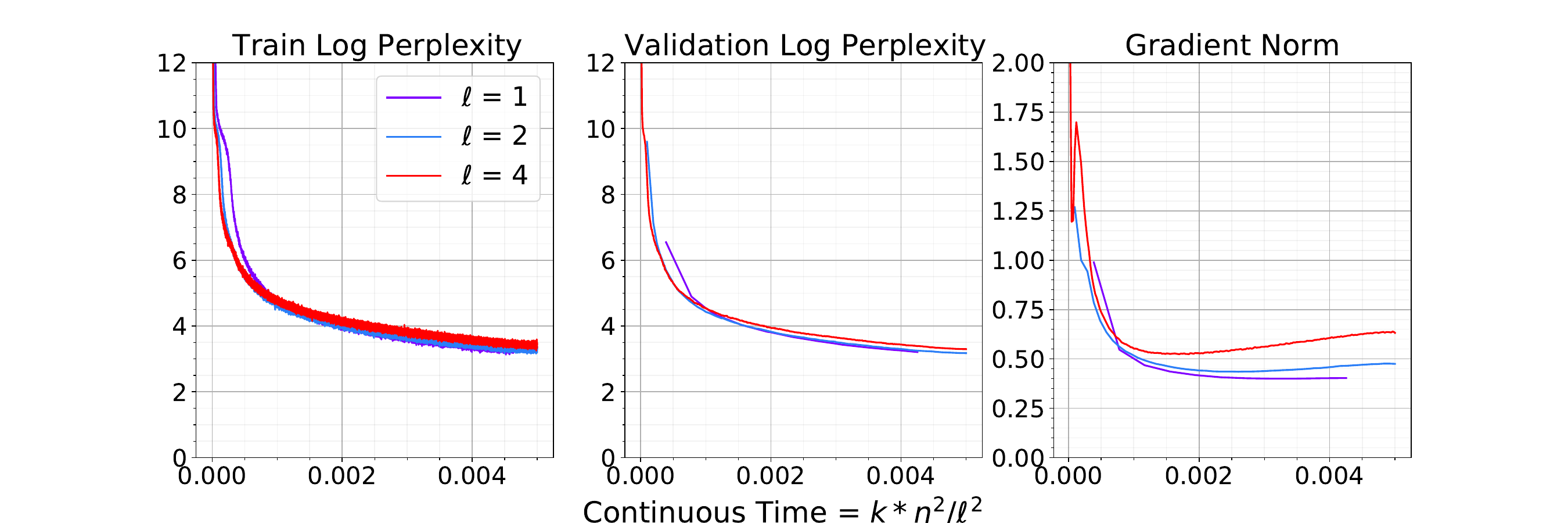}}
\caption{SVAG experiments on RoBERTa pretrained on Bookcorpus + Wikipedia dataset with Adam. We use batch size $1024$ and the hyperparameters $\eta = 10^{-3}, \beta_1 = 0.9, \beta_2 = 0.98, \epsilon = 2\times10^{-6}$. Since SVAG takes $\ell$ smaller steps to simulate
the continuous dynamics in $\eta$ time, we plot accuracy against continuous time defined as $k \times \eta^2 / \ell^2$.}
\label{fig:sqrtscaling_roberta_svag}
\end{center}
\end{figure*}

\begin{figure*}
\begin{center}
\centerline{\includegraphics[width=\linewidth]{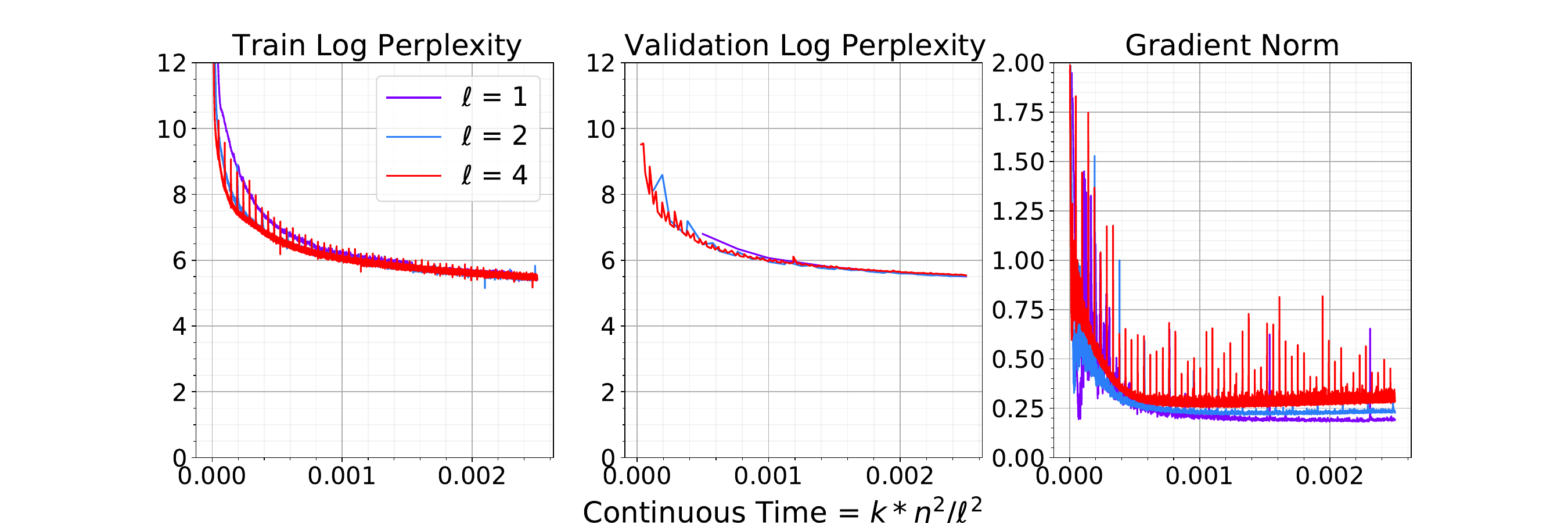}}
\caption{SVAG experiments on GPT pretrained on WikiText-103 with Adam. We use batch size $1024$ and the hyperparameters $\eta = 10^{-3}, \beta_1 = 0.9, \beta_2 = 0.98, \epsilon = 2\times10^{-8}$. Since SVAG takes $\ell$ smaller steps to simulate
the continuous dynamics in $\eta$ time, we plot accuracy against continuous time defined as $k \times \eta^2 / \ell^2$. For computational constraints, we train on sequences of length $64$.}
\label{fig:sqrtscaling_gpt_svag}
\end{center}
\end{figure*}

\section{Square Root Scaling Experiments}\label{sec:app_sqsr_scaling_exps}
We experimentally evaluate the scaling rules proposed in \Cref{def:adam_scaling,def:rmsprop_scaling} by training models with different batch sizes and modifying the optimization hyperparameters accordingly.
The number of gradient steps were modified to keep the total amount of continuous time same across all the batches. The warmup schedule and the learning rate decay schedule were kept the same. See \Cref{sec:app_exp_config} for details about the baseline runs for each dataset.

We perform two types of ablation studies. In \Cref{sec:app_sqsr_ablation}, we compare our proposed scaling rule to variants that only scale some subset of the optimization hyperparameters. We find that our proposed scaling rule is the best at preserving the validation accuracy and the other test functions across different batch sizes.
In \Cref{sec:app_lsr_ablation}, we compare the proposed square root scaling rule to a linear one and find that they perform comparably on CIFAR-10, though the square root scaling significantly outperforms linear scaling on the ImageNet dataset.
We hypothesize that for simpler datasets, like CIFAR-10, the differences between the two scaling rules is not reflected clearly in the validation accuracies because the task is too easy to learn.

\paragraph{CIFAR-10.}
    \Cref{fig:accrmsprop_vgg_smalleps,fig:accrmsprop_vgg_largeeps} show the performance of VGG-16 when trained with different batch sizes with RMSprop at smaller $\epsilon$ (=$10^{-30}$), and at larger $\epsilon$ (=$10^{-8}$) respectively. \Cref{fig:accAdam_vgg_smalleps,fig:accAdam_vgg_largeeps} show the performance of VGG-16 when trained with different batch sizes with Adam at smaller $\epsilon$ (=$10^{-30}$), and at larger $\epsilon$ (=$10^{-8}$) respectively. For the corresponding experiments on ResNet-50, please refer to  \Cref{fig:accrmsprop_resnet_smalleps,fig:accrmsprop_resnet_largeeps,fig:accAdam_resnet_smalleps,fig:accAdam_resnet_largeeps}. For the details on the experimental setting, please refer to \Cref{sec:app_cifar10_config}. 

We observe that in all settings the test accuracies vary by at most $3\%$ across batch sizes when using the proposed square root scaling rule. 
Moreover, the test functions stay close across multiple batch sizes, signifying that the trajectories stay close across the batch sizes using the Square root scaling rule.


\paragraph{ImageNet.}
\Cref{fig:sqrtscaling_imagenet_smalleps,fig:sqrtscaling_imagenet_largeeps} show the performance of ResNet-50 when trained with different batch sizes with RMSprop at smaller $\epsilon$ (=$10^{-30}$), and at larger $\epsilon$ (=$10^{-8}$) respectively. For the details on the experimental setting, please refer to \Cref{sec:app_imagenet_config}.

We observe that on the validation set, the loss and accuracy behavior for the model is very similar, when trained with different batch sizes. Moreover, the difference between the validation accuracies is atmost $3\%$ between the batch sizes $256$, $1024$, $4096$, and $16384$ for smaller $\epsilon$. The difference between the validation accuracies is atmost $1.5\%$ between the batch sizes $256$, $1024$, $4096$, $16384$, and $32768$ for larger $\epsilon$.

\paragraph{Wikipedia + Books (Academic BERT).} \Cref{fig:sqrtscaling_roberta} shows the performance of a $\text{RoBERTa}$ model when pretrained with different batch sizes with ADAM. The scaling rule is applied to modify the peak values of the optimization hyperparameters. We also evaluate the pretrained models on several downstream tasks, and show the results in \cref{tab:Downstream_performance_academicbert}. For the details on the setting for both pretraining and downstream experiments, please refer to \Cref{sec:app_books_wiki_config}. 

We observe that the log perplexity on the training and validation datasets matches across different batch sizes throughout pretraining. Moreover, we also observe that the gradient norms across different batch sizes remain close throughout the pretraining. Furthermore, the models pretrained across different batch sizes can achieve very similar performance in the downstream tasks.

\paragraph{WikiText-103 (Academic GPT).} \Cref{fig:sqrtscaling_gpt} shows the log perplexity behavior 
of $\text{GPT}$ on training and validation datasets, across different batch sizes of pretraining . We observe that the log perplexity matches across different batch sizes. Moreover, we also observe that the gradient norms across different batch sizes remain close throughout the training.

\begin{figure}[!htbp]
\begin{center}
\centerline{\includegraphics[width=\linewidth]{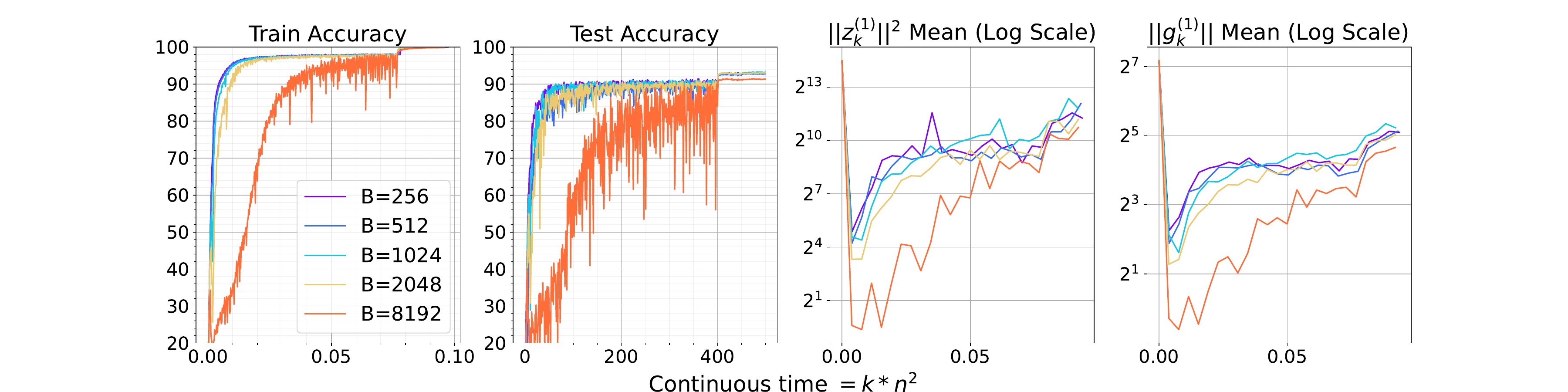}}
\caption{VGG-16 trained on CIFAR-10 using RMSprop are close for different batch sizes when the optimization hyperparameters are varied according to the proposed scaling rule for RMSprop (\Cref{def:adam_scaling}). We use a baseline setting of $\eta=10^{-3}$, $\epsilon=10^{-8}$, and $\beta =  0.999$ for batch size $256$.  We use a weight decay factor of $10^{-4}$. We observe a gap of at most $3\%$ among the different batch sizes under consideration.}
\label{fig:accrmsprop_vgg_largeeps}
\end{center}
\end{figure}

\begin{figure}[!htbp]
\begin{center}
\centerline{\includegraphics[width=\linewidth]{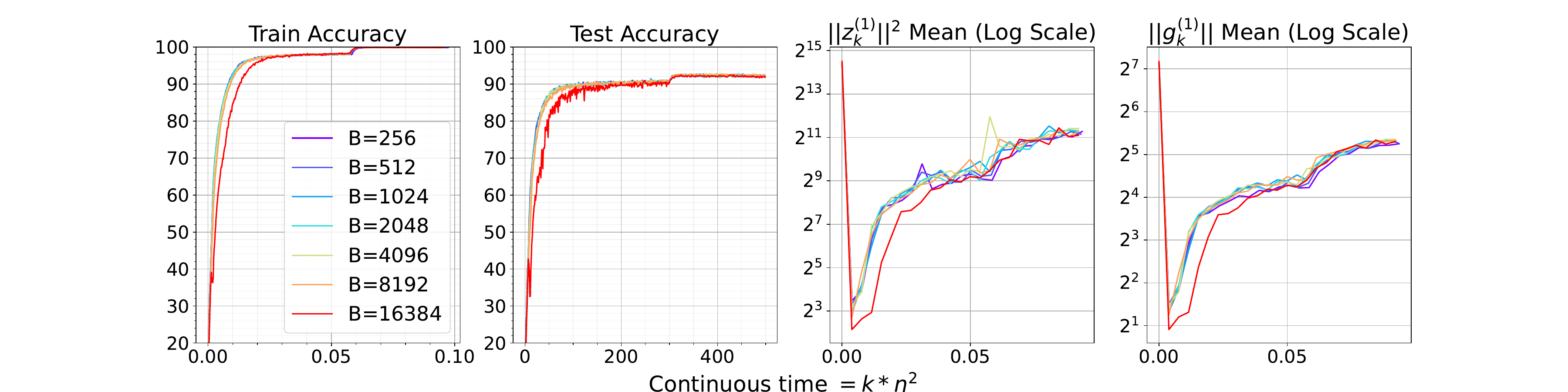}}
\caption{VGG-16 trained on CIFAR-10 using Adam are close for different batch sizes when the optimization hyperparameters are varied according to the proposed scaling rule for Adam (\Cref{def:adam_scaling}). We use a baseline setting of $\eta=10^{-3}$, $\epsilon=10^{-8}$, and $(\beta_1, \beta_2) = (0.999, 0.999)$ for batch size $256$.  We use a weight decay factor of $10^{-4}$. We observe a gap of at most $3\%$ among the different batch sizes under consideration.}
\label{fig:accAdam_vgg_largeeps}
\end{center}
\end{figure}

\begin{figure}[!htbp]
\begin{center}
\centerline{\includegraphics[width=\linewidth]{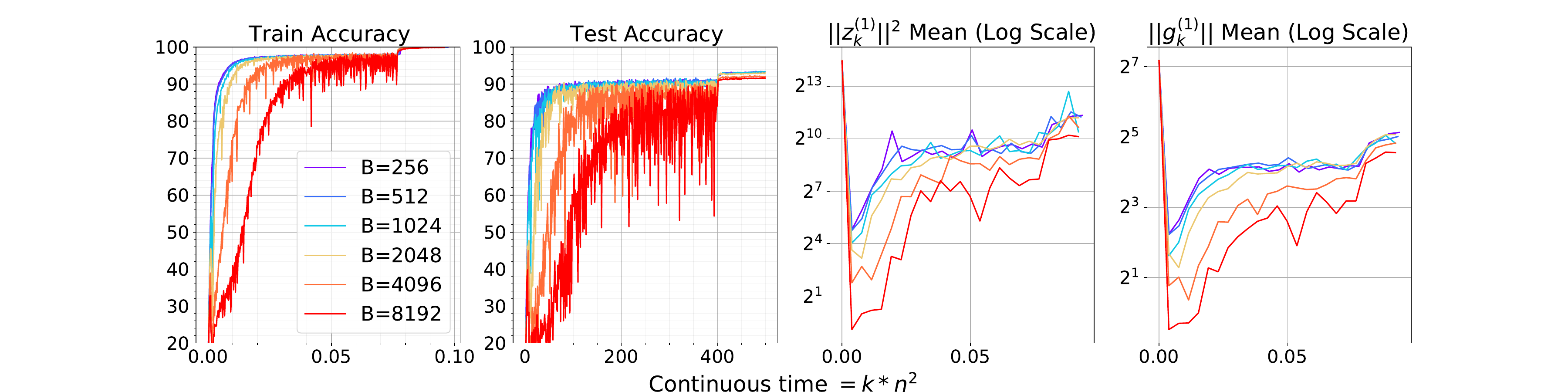}}
\caption{VGG-16 trained on CIFAR-10 using RMSprop are close for different batch sizes when the optimization hyperparameters are varied according to the proposed scaling rule for RMSprop (\Cref{def:adam_scaling}). We use a baseline setting of $\eta=10^{-3}$, $\epsilon=10^{-30}$, and $\beta = 0.999$ for batch size $256$.  We use a weight decay factor of $10^{-4}$. We observe a gap of at most $3\%$ among the different batch sizes under consideration.}
\label{fig:accrmsprop_vgg_smalleps}
\end{center}
\end{figure}

\begin{figure}[!htbp]
\begin{center}
\centerline{\includegraphics[width=\linewidth]{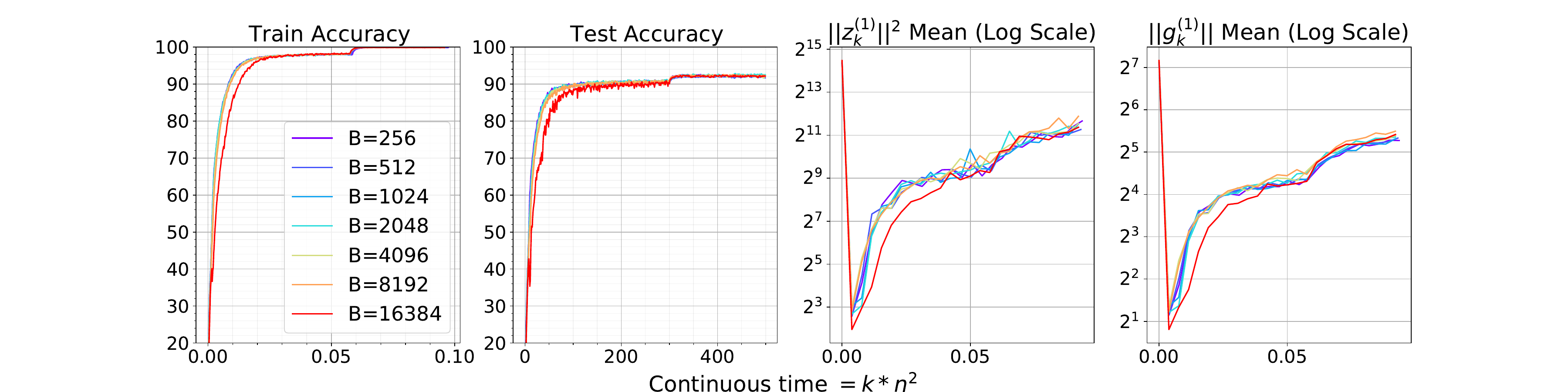}}
\caption{VGG-16 trained on CIFAR-10 using Adam are close for different batch sizes when the optimization hyperparameters are varied according to the proposed scaling rule for Adam (\Cref{def:adam_scaling}). We use a baseline setting of $\eta=10^{-3}$, $\epsilon=10^{-30}$, and $(\beta_1, \beta_2) = (0.999, 0.999)$ for batch size $256$.  We use a weight decay factor of $10^{-4}$. We observe a gap of at most $3\%$ among the different batch sizes under consideration.}
\label{fig:accAdam_vgg_smalleps}
\end{center}
\end{figure}

\begin{figure}[!htbp]
\begin{center}
\centerline{\includegraphics[width=\linewidth]{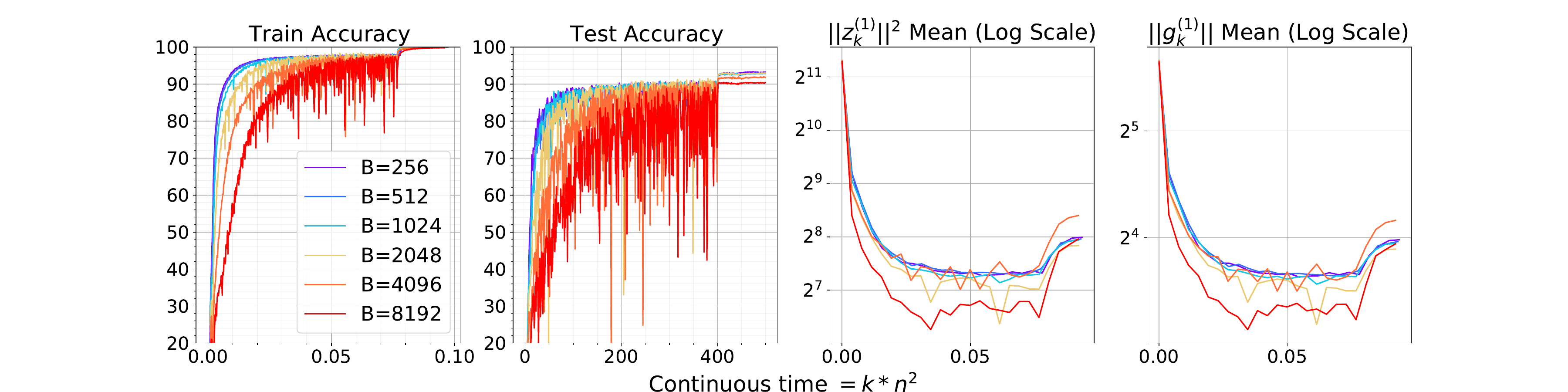}}
\caption{ResNet-50 trained on CIFAR-10 using RMSprop are close for different batch sizes when the optimization hyperparameters are varied according to the proposed scaling rule for RMSprop (\Cref{def:adam_scaling}). We use a baseline setting of $\eta=10^{-3}$, $\epsilon=10^{-8}$, and $\beta = 0.999$ for batch size $256$. $\epsilon=10^{-30}\approx 0$ for all experiments. We use a weight decay factor of $10^{-4}$.  We observe a gap of at most $3\%$ among the different batch sizes under consideration.}
\label{fig:accrmsprop_resnet_smalleps}
\end{center}
\end{figure}

\begin{figure}[!htbp]
\begin{center}
\centerline{\includegraphics[width=\linewidth]{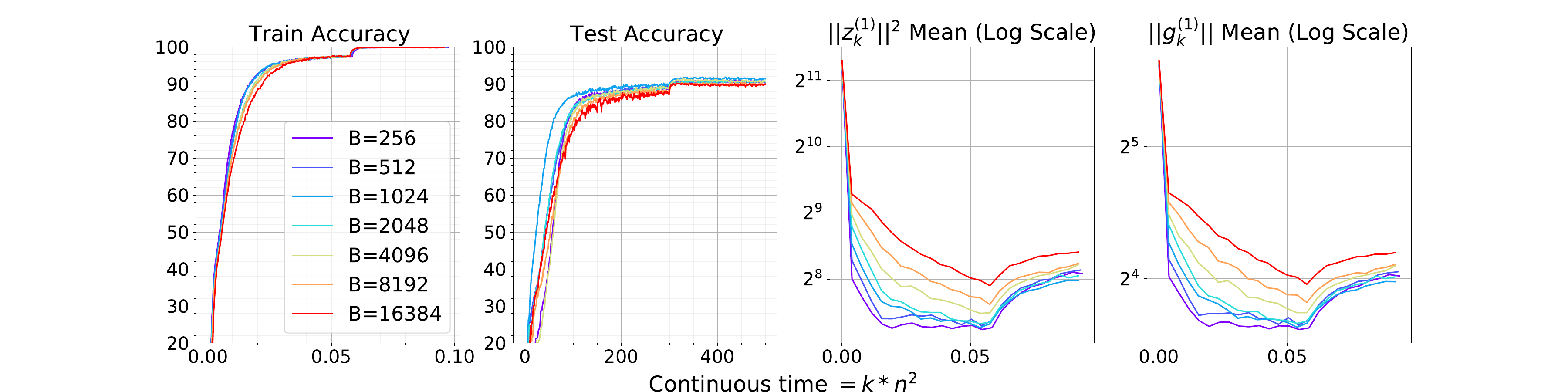}}
\caption{ResNet-50 trained on CIFAR-10 using Adam are close for different batch sizes when the optimization hyperparameters are varied according to the proposed scaling rule for Adam (\Cref{def:adam_scaling}). We use a baseline setting of $\eta=10^{-3}$ and $(\beta_1, \beta_2) = (0.999, 0.999)$ for batch size $256$. $\epsilon=10^{-30}\approx 0$ for all experiments. We use a weight decay factor of $10^{-4}$.  We observe a gap of at most $3\%$ among the different batch sizes under consideration.}
\label{fig:accAdam_resnet_smalleps}
\end{center}
\end{figure}

\begin{figure}[!htbp]
\begin{center}
\centerline{\includegraphics[width=\linewidth]{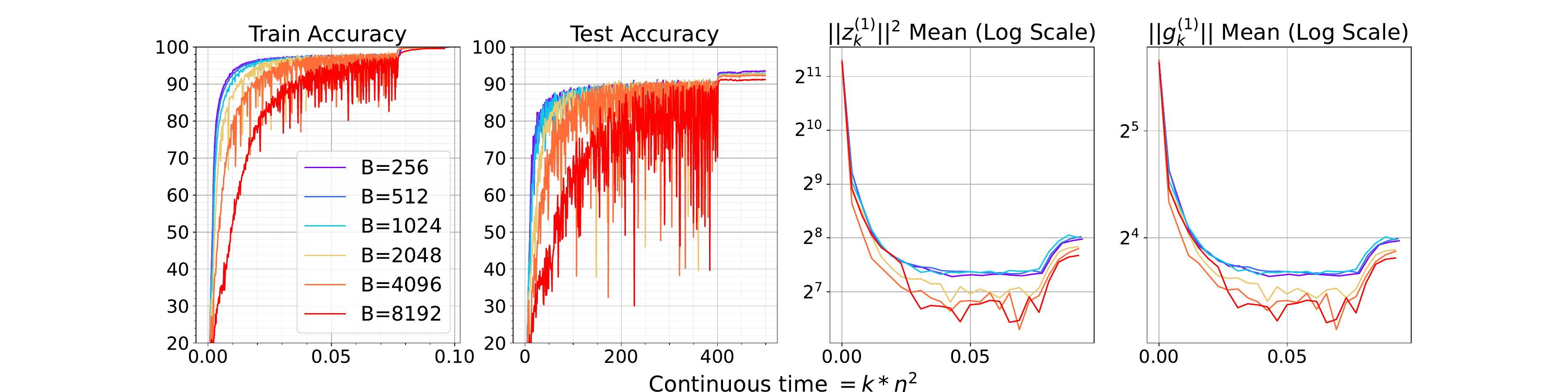}}
\caption{ResNet-50 trained on CIFAR-10 using RMSprop are close for different batch sizes when the optimization hyperparameters are varied according to the proposed scaling rule for RMSprop (\Cref{def:adam_scaling}). We use a baseline setting of $\eta=10^{-3}$, $\epsilon = 10^{-8}$, and $\beta =  0.999$ for batch size $256$.  We use a weight decay factor of $10^{-4}$. We observe a gap of at most $3\%$ among the different batch sizes under consideration.}
\label{fig:accrmsprop_resnet_largeeps}
\end{center}
\end{figure}

\begin{figure}[!htbp]
\begin{center}
\centerline{\includegraphics[width=\linewidth]{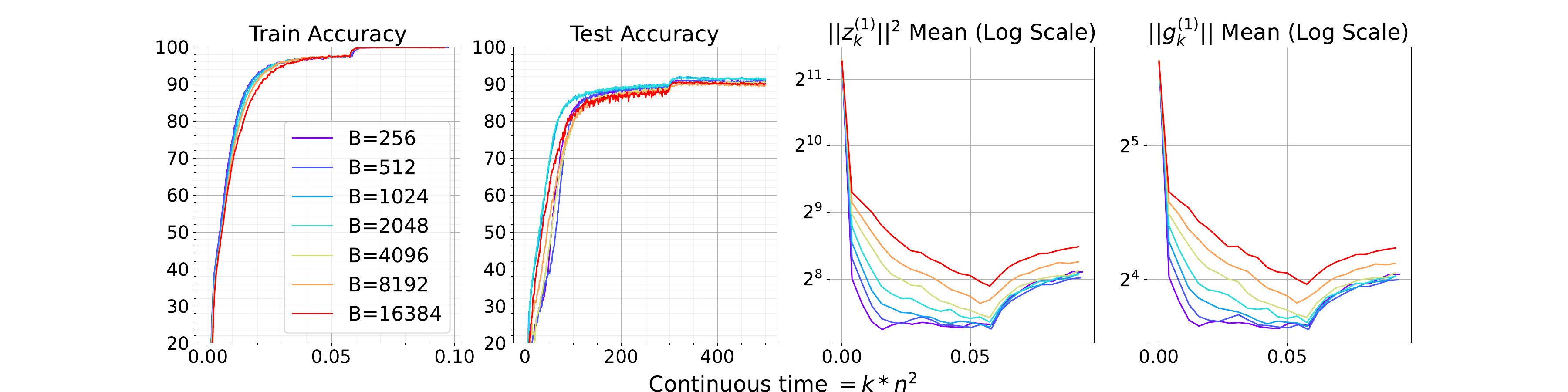}}
\caption{ResNet-50 trained on CIFAR-10 using Adam are close for different batch sizes when the optimization hyperparameters are varied according to the proposed scaling rule for Adam (\Cref{def:adam_scaling}). We use a baseline setting of $\eta=10^{-3}$, $\epsilon = 10^{-8}$, and $(\beta_1, \beta_2) = (0.999, 0.999)$ for batch size $256$.  We use a weight decay factor of $10^{-4}$. We observe a gap of at most $3\%$ among the different batch sizes under consideration.}
\label{fig:accAdam_resnet_largeeps}
\end{center}
\end{figure}

\begin{figure}[!htbp]
\begin{center}
\centerline{\includegraphics[width=\linewidth]{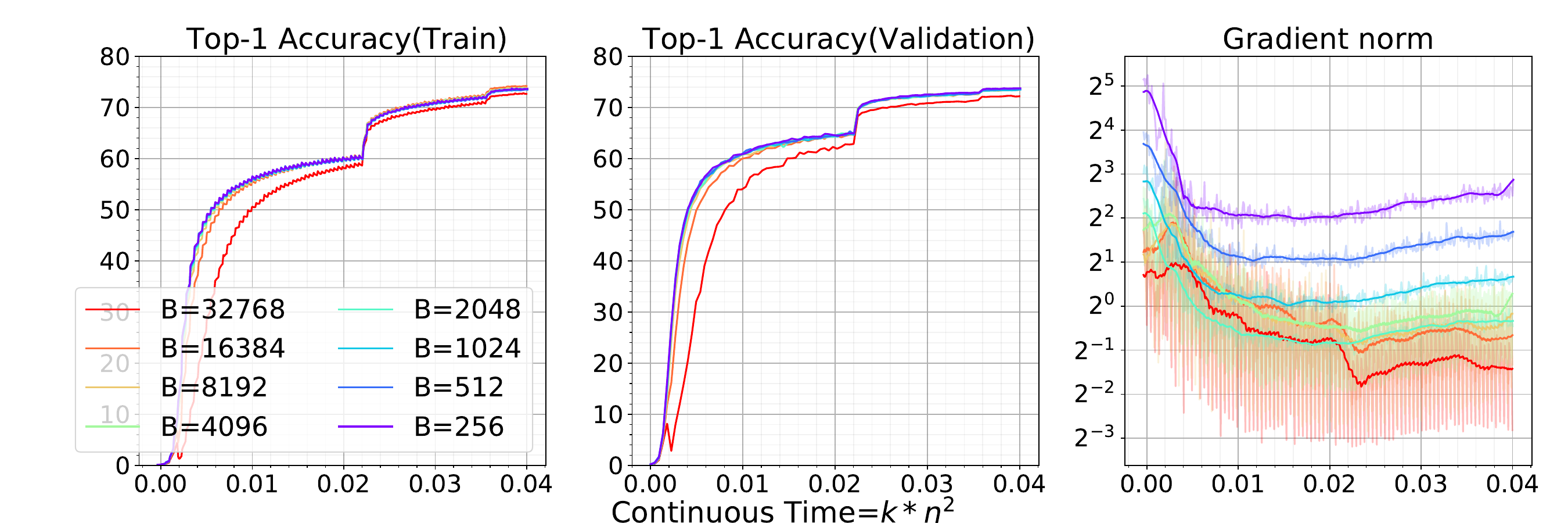}}
\caption{ResNet-50 trained on ImageNet using Adam are close for different batch sizes when the optimization hyperparameters are varied according to the proposed scaling rule for Adam (\Cref{def:adam_scaling}). We use a baseline setting of $\eta=3 \times 10^{-4}$, $\epsilon = 10^{-8}$, and $(\beta_1, \beta_2) = (0.999, 0.999)$ for batch size $256$. We use a weight decay factor of $10^{-4}$. We achieve around $74\%$ validation accuracy with batch size $256$ and the accuracy drops by at most $1.5\%$ at batch size $32768$.}
\label{fig:sqrtscaling_imagenet_largeeps}
\end{center}
\end{figure}

\begin{figure}[!htbp]
\begin{center}
\centerline{\includegraphics[width=\linewidth]{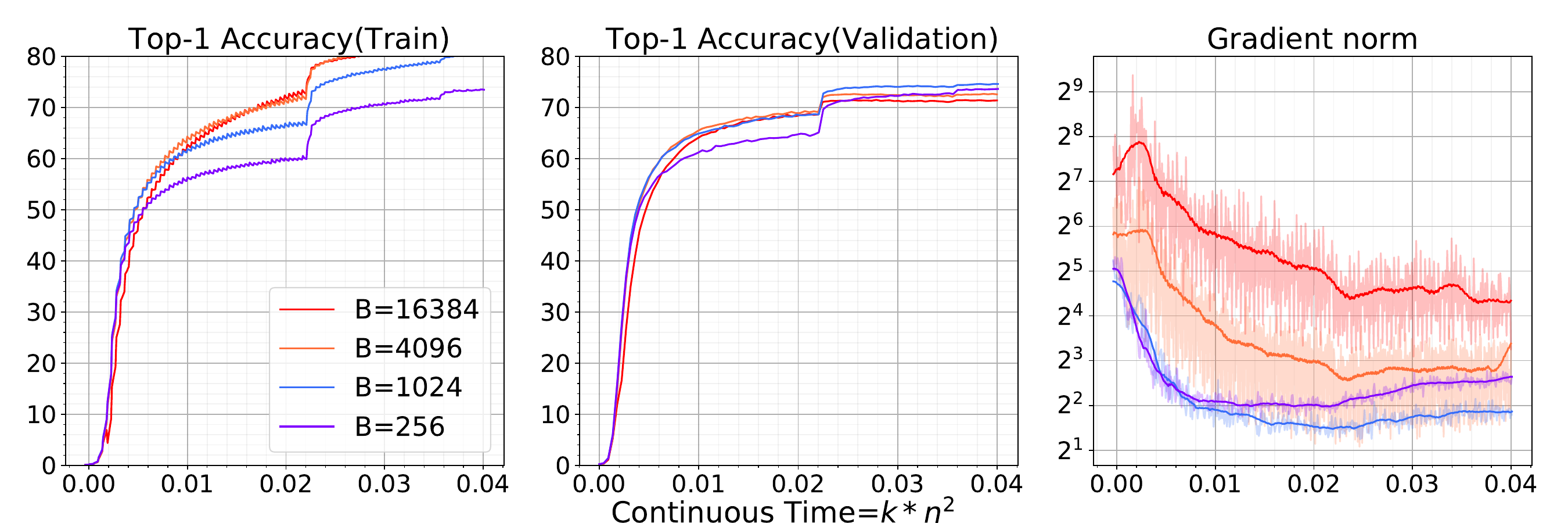}}
\caption{ResNet-50 trained on ImageNet using Adam are close for different batch sizes when the optimization hyperparameters are varied according to the proposed scaling rule for Adam (\Cref{def:adam_scaling}). We use a baseline setting of $\eta=3 \times 10^{-4}$ and $(\beta_1, \beta_2) = (0.999, 0.999)$ for batch size $256$. $\epsilon=10^{-30}\approx 0$ for all experiments. We use a weight decay factor of $10^{-4}$.  We achieve around $74\%$ validation accuracy with batch size $1024$ and the accuracy drops by at most $3\%$ at batch size $16384$.}
\label{fig:sqrtscaling_imagenet_smalleps}
\end{center}
\end{figure}

\begin{figure}[!htbp]
\begin{center}
\centerline{\includegraphics[width=\linewidth]{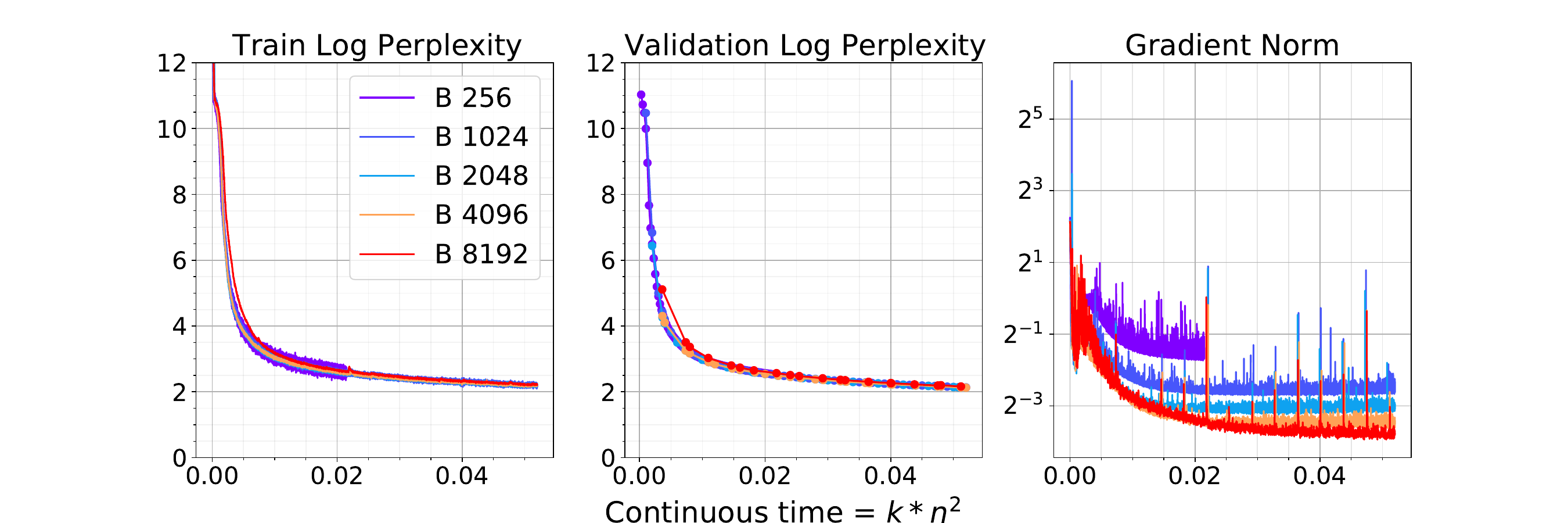}}
\caption{The train and validation log perplexities of $\text{RoBERTa}$-large trained on the
Wiki+Books corpus using Adam (for $48$ hours) are close for moderate batch sizes using the Square Root
Scaling Rule on Adam. $\eta = 10^{-3}$ and $(\beta_1, \beta_2)
= (0.99375, 0.996)$ for batch size $1024$, $\epsilon = 2 \times 10^{-6}$ for
batch size $1024$ and scaled likewise for other batch sizes. We achieve a validation log perplexity of $2.1 \pm 0.1$ for batch size $1024$, $2048$, $4096$ and $8192$. Training with batch size $256$ is computationally inefficient, but follows the same behavior during its 48-hour trajectory.}
\label{fig:sqrtscaling_roberta}
\end{center}
\end{figure}

\begin{table}[]
\centering
\begin{tabular}{c|c|c|c|c|c|c|c|c}
		Pretrain batch size $B$ & CoLA & SST-2 & MRPC & STS-B & QQP & MNLI  & QNLI & RTE \\ \hline
		$1024$ & $0.585$ & $0.92$ & $0.73$ & $0.866$ & $0.873$ & $0.836$  & $0.906$ & $0.682$   \\
		$2048$  & $0.563$ & $0.928$ & $0.803$ & $0.869$ & $0.875$ & $0.826$  & $0.897$ & $0.653$   \\
		$4096$ & $0.581$ & $0.921$ & $0.778$ & $0.869$ & $0.875$ & $0.839$  & $0.892$ & $0.675$   \\
		$8192$ & $0.626$ & $0.929$ & $0.778$ & $0.884$ & $0.877$ & -  & $0.9$ & $0.675$  \\
		\hline \vspace{0.1cm}
\end{tabular}
\caption{Performance of the pretrained $\text{RoBERTa}$ models when finetuned on different downstream tasks in GLUE \cite{wang2018glue}. F1 scores are reported for QQP and MRPC, Spearman correlations are reported for STS-B, Matthews correlations for CoLA, and
accuracy scores are reported for the other tasks. Here, $B$ denotes the batch size used for pretraining. We run an extensive grid search (\Cref{tab:Downstream_hyperparam_academicbert}) to find the best performance of each pretrained model on each of the downstream tasks.}
\label{tab:Downstream_performance_academicbert}
\end{table}

\begin{figure}[!htbp]
\begin{center}
\centerline{\includegraphics[width=\linewidth]{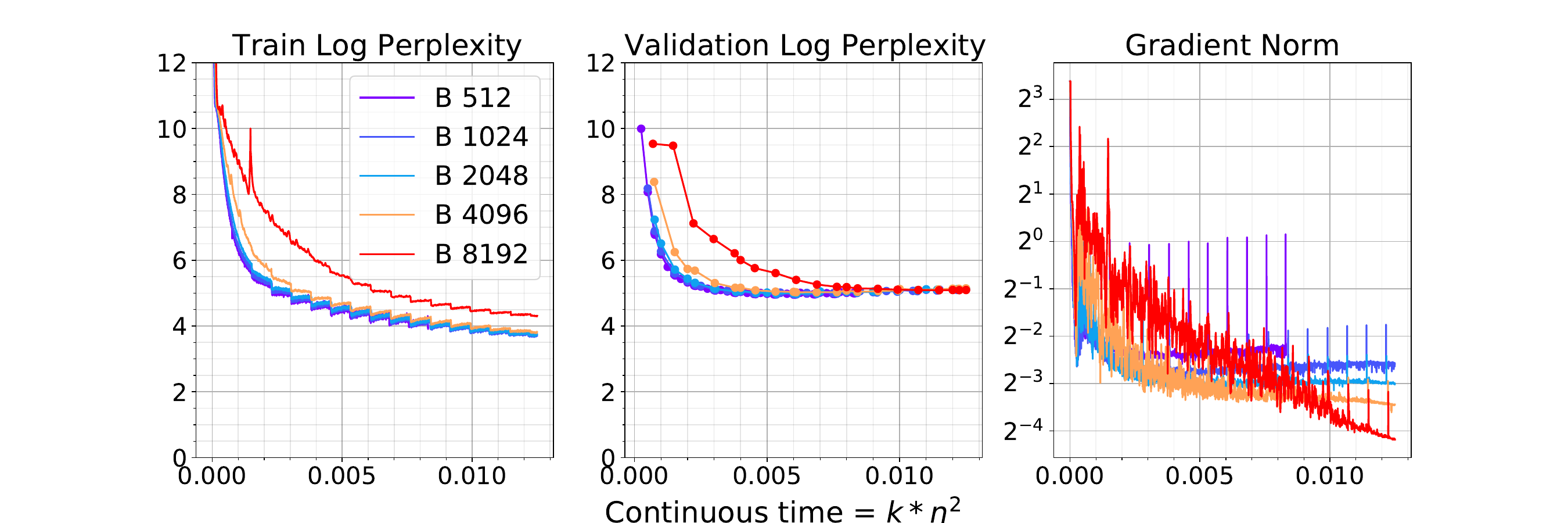}}
\caption{The train and validation perplexities of 12 layer GPT trained on the
Wikitext corpus using Adam (for $48$ hours) are close for moderate batch sizes Square Root
Scaling Rule on Adam on $\text{RoBERTa}$-large. $\eta = 10^{-3}$ and $(\beta_1, \beta_2)
= (0.9875, 0.996)$ for batch size $1024$, $\epsilon = 2 \times 10^{-8}$ for
batch size $1024$ and scaled likewise for other batch sizes.  We achieve a validation log perplexity of $5 \pm 0.1$ for all the batch sizes under consideration. Moreover, we observe an alignment in the behavior of the trajectories across the different batch sizes (except $8192$ in the first half of the training).}
\label{fig:sqrtscaling_gpt}
\end{center}
\end{figure}

\subsection{Ablation study on the proposed scaling rule}\label{sec:app_sqsr_ablation}
\paragraph{CIFAR-10.} We conduct an ablation study on whether all the parameters $\eta, \epsilon, \beta_1, \beta_2$ need to be scaled in our proposed scaling rule. To do so, we compare the performance of a ResNet-50 model trained with batch size $256$ and hypeparameters ($\epsilon = 10^{-8}, \beta_1 = 0.999, \beta_2 = 0.999$) with the performance at a larger batch size, across $5$ runs representing $5$ different scaling rules: (a) Scale $\eta$, keeping others fixed, (b) Scale $\eta, \epsilon$, keeping others fixed, (c) Scale $\eta, \epsilon, \beta_1$, keeping others fixed, (d) Scale $\eta, \epsilon, \beta_2$, keeping others fixed, and (e) Scale $\eta, \epsilon, \beta_1, \beta_2$. Please check the behaviors of the different scaling rules at batch size $2048$, $4096$, $8192$ and $16384$ in \Cref{fig:ablation_sqrtscaling}. We found that (e) consistently beats others in terms of the test functions and the validation accuracies at all batch sizes. The closest scaling rule (c) involved scaling only $\eta, \epsilon, \mbox{ and } \beta_1$ while keeping $\beta_2$ fixed. 

\subsection{Ablation against linear scaling rules} \label{sec:app_lsr_ablation}
\paragraph{CIFAR-10.} We compared the proposed scaling rule against possible linear scaling rules, that scale the hyperparameters linearly with the increase in training batch size. We focused on ResNet-50 training with Adam. The linear scaling rules that we tried were: (a) Scale $\eta$ linearly, keeping $\beta_1, \beta_2, \epsilon$ fixed, (b) Scale $\eta, 1 - \beta_1$ linearly, keeping $\beta_2, \epsilon$ fixed, (c) Scale $\eta, 1 - \beta_2$ linearly, keeping $\beta_1, \epsilon$ fixed, and (d) Scale $\eta, 1 - \beta_1, 1 - \beta_2$ linearly, keeping $\epsilon$ fixed. Fig.~\ref{fig:ablation_LSR} shows the behavior of these scaling rules at batch size $8192$ and $16384$, at different values of $\epsilon$.

The linear scaling rule in (d) seems to perform as well as the proposed Square Root scaling rule (\cref{def:adam_scaling}) in terms of validation accuracies. However, with a closer look on the train accuracy plots, we observe that the Square Root scaling rule tracks the smaller batch training trajectory better than the linear scaling rule. The linear scaling rules seem to catch up, only after the learning rate is decayed. Our hypothesis is that the CIFAR-10 dataset is simple enough for  different scaling rules to work well. 

\paragraph{ImageNet.} We conduct ablation experiments on ResNet-50, trained with Adam on ImageNet, where we follow a linear scaling rule to scale the hyperparameters across batch sizes. Due to computational issues, we didn't conduct extensive experiments, as was done for CIFAR-10. The linear scaling rule for Adam is as follows: the hyperparameters $\eta$, $1 - \beta_1$ and $1 - \beta_2$ are scaled by $\kappa$, when the batch size is scaled by $\kappa$, and $\epsilon$ isn't scaled. As was noted earlier, the definition of continuous time will change to $\eta \times \# \{ \text{gradient steps} \}$. We keep the number of training epochs equal to $90$ as before, follow the same learning rate schedule,  and show the performance of the models in \cref{fig:linearscaling_imagenet_largeeps}. We observe that scaling the hyperparameters to larger batch training with the proposed LSR results in a big drop in validation accuracies.

\begin{figure}
 \centering
     \begin{subfigure}
         \centering\includegraphics[width=\linewidth]{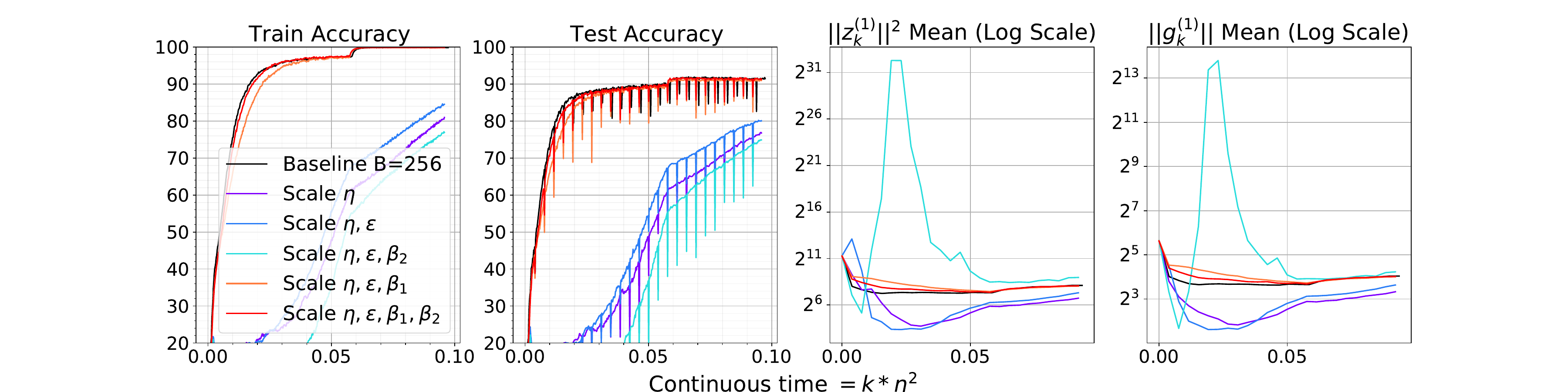}\\(a) $B=2048$
     \end{subfigure}
     \hfill
     \begin{subfigure}
         \centering\includegraphics[width=\linewidth]{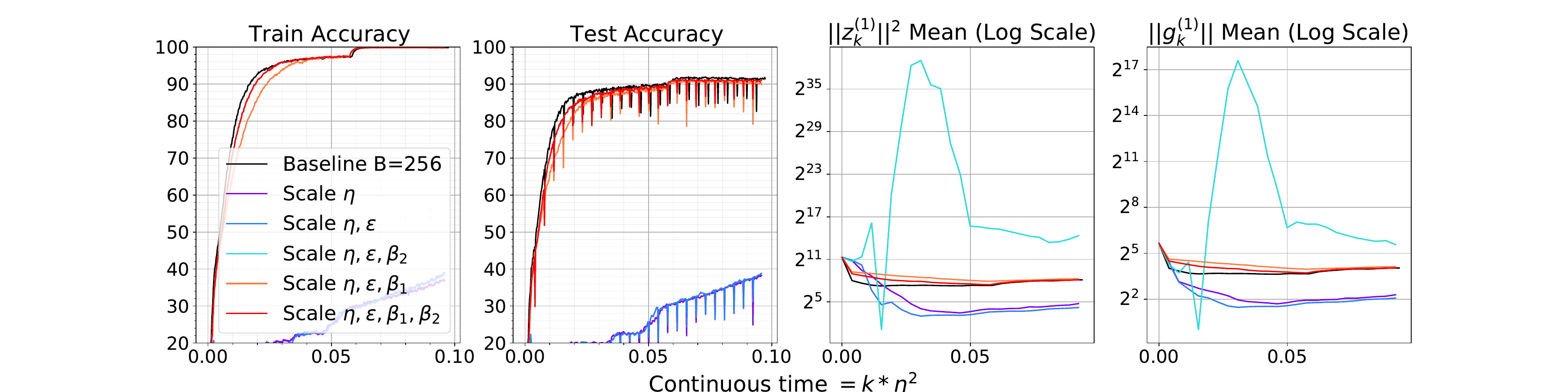}\\(b) $B=4096$
     \end{subfigure}
  \begin{subfigure}
    \centering\includegraphics[width=\linewidth]{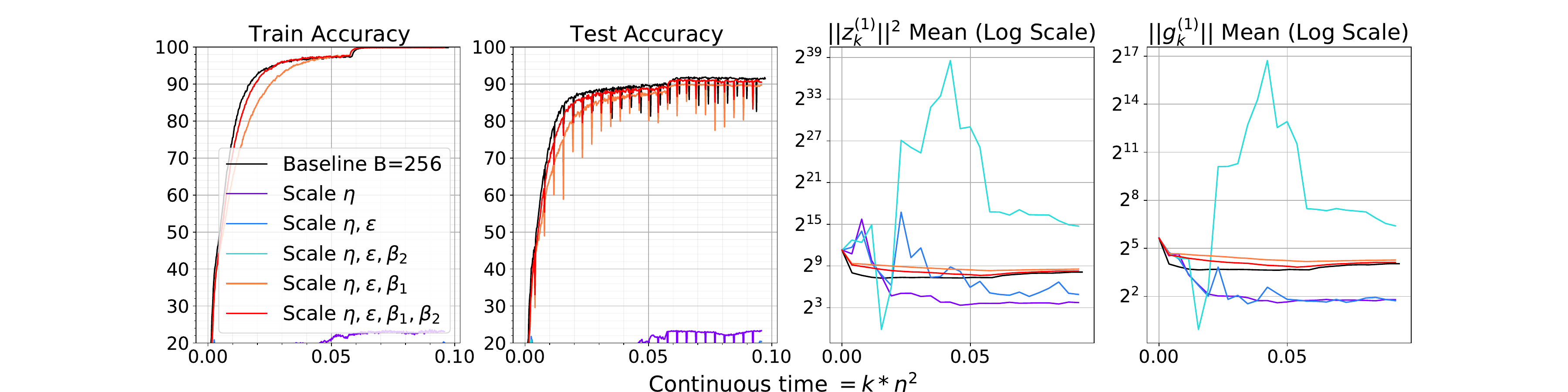}\\(c) $B=8192$
    \end{subfigure}
  \begin{subfigure}
    \centering\includegraphics[width=\linewidth]{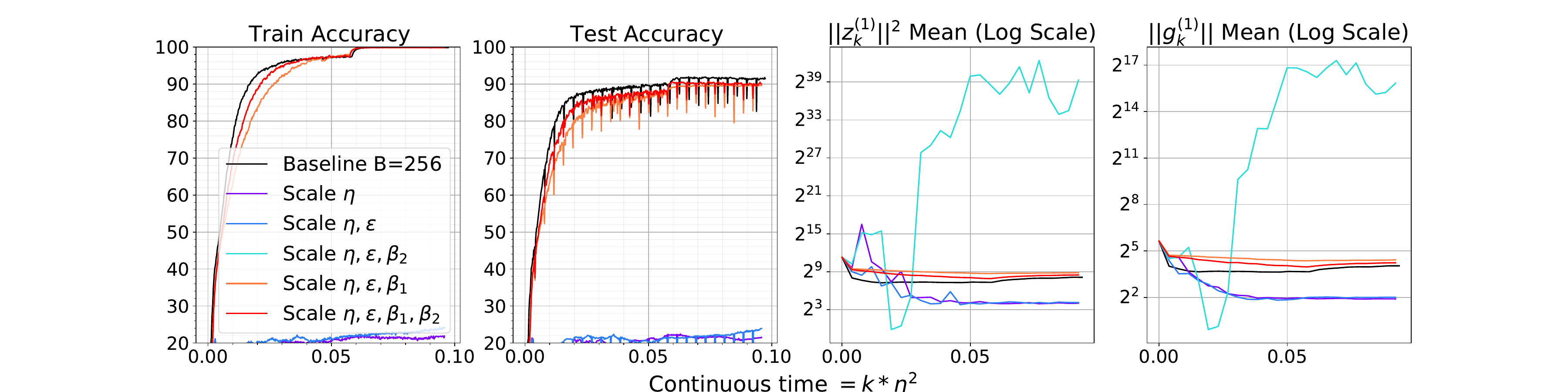}\\(d) $B=16384$
  \end{subfigure}
  \label{fig:ablation_sqrtscaling}
   \caption{Ablation study for the square root scaling rule on Resnet-50 trained with Adam on CIFAR-10. We compare the performance of a model trained with batch size $256$ and hypeparameters ($\epsilon = 10^{-8}, \beta_1 = 0.999, \beta_2 = 0.999$) with the performance at a larger batch size, across $5$ runs representing $5$ variations of the square root scaling rule: (a) Scale $\eta$, keeping others fixed, (b) Scale $\eta, \epsilon$, keeping others fixed, (c) Scale $\eta, \epsilon, \beta_1$, keeping others fixed, (d) Scale $\eta, \epsilon, \beta_2$, keeping others fixed, and (e) Scale $\eta, \epsilon, \beta_1, \beta_2$. We use a weight decay of $10^{-4}$ in all the experiments. We observe that scaling all the hyperparameters consistently gives better performance at higher batch size. Scaling rule (c) is close second.}
  \end{figure}

\begin{figure}[!htbp]
  \centering
  \begin{subfigure}
    \centering\includegraphics[width=\linewidth]{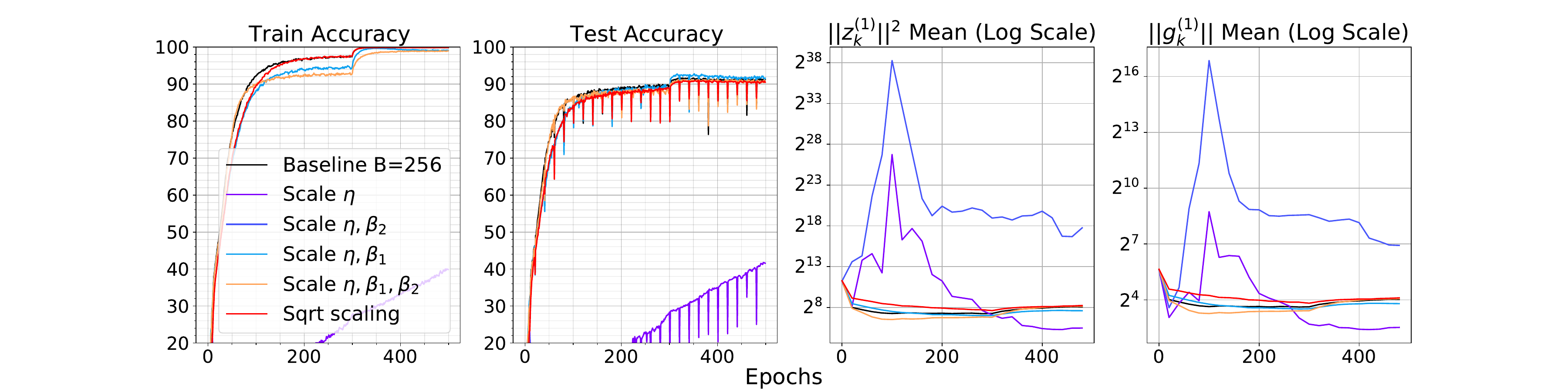} \\(a) $B=8192$, $\epsilon$ of order $10^{-30}$
  \end{subfigure}
\begin{subfigure}
    \centering\includegraphics[width=\linewidth]{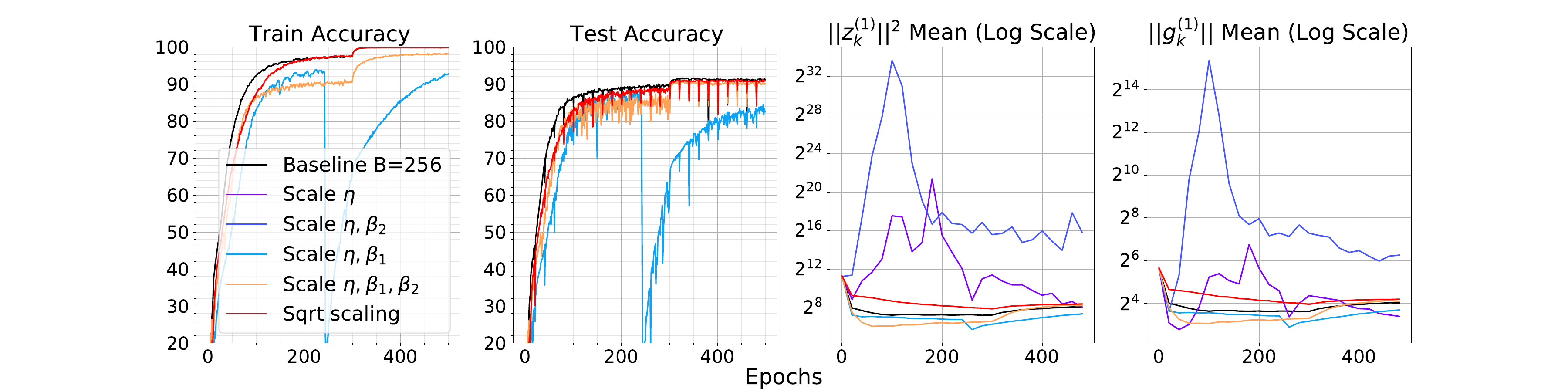}\\(b) $B=16384$, $\epsilon$ of order $10^{-30}$
  \end{subfigure}
  \begin{subfigure}
    \centering\includegraphics[width=\linewidth]{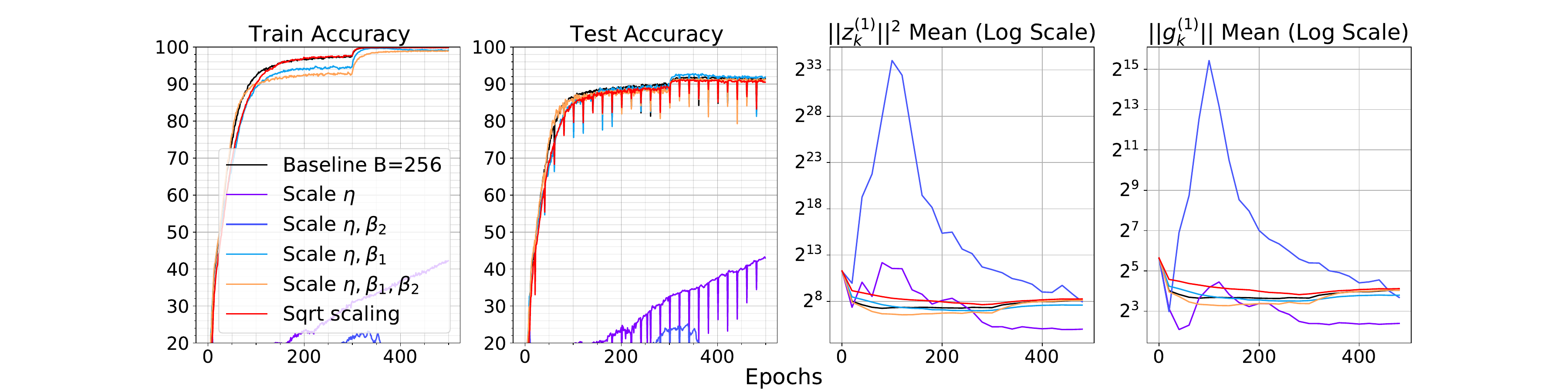}\\(c) $B=8192$, $\epsilon$ of order $10^{-8}$
  \end{subfigure}
\begin{subfigure}
    \centering\includegraphics[width=\linewidth]{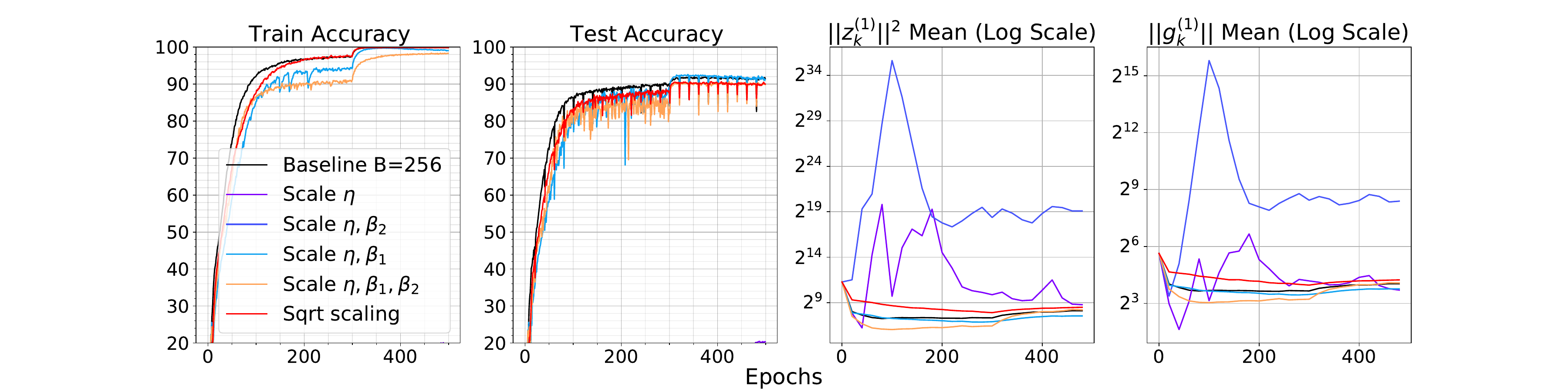}\\(d) $B=16384$, $\epsilon$ of order $10^{-8}$
  \end{subfigure}
  \caption{Ablation study against (possible) linear scaling rules on Resnet-50 trained with Adam on CIFAR-10. We compare the performance of a model trained with batch size $256$ and hypeparameters ($\epsilon = 10^{-8}/10^{-30}, \beta_1 = 0.999, \beta_2 = 0.999$) with the performance at a larger batch size, across $5$ runs representing $5$ possible linear scaling rules: (a) Scale $\eta$ linearly, keeping $\beta_1, \beta_2, \epsilon$ fixed, (b) Scale $\eta, 1 - \beta_1$ linearly, keeping $\beta_2, \epsilon$ fixed, (c) Scale $\eta, 1 - \beta_2$ linearly, keeping $\beta_1, \epsilon$ fixed, and (d) Scale $\eta, 1 - \beta_1, 1 - \beta_2$ linearly, keeping $\epsilon$ fixed. We use a weight decay of $10^{-4}$ in all the experiments. We also compare the behavior of the linear scaling rules against the square root scaling rule. Since the continuous time definition varies across the scaling rules, we plot against the number of epochs trained. A closer look at the training accuracy plots shows that the square root scaling rule tracks the smaller batch training trajectory better than the linear scaling rule.
  }
  \label{fig:ablation_LSR}
  \end{figure}

\begin{figure}[!htbp]
\begin{center}
\centerline{\includegraphics[width=\linewidth]{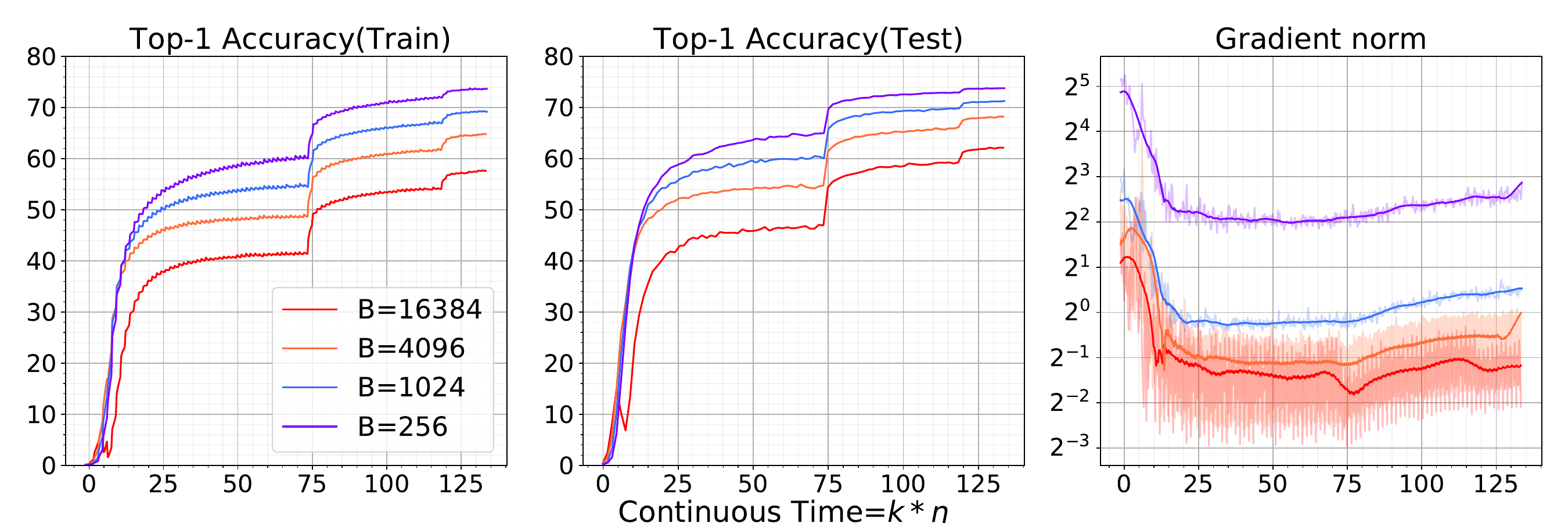}}
\caption{Ablation study against a (possible) linear scaling rule on Resnet-50 trained with Adam on Imagenet. We compare the performance of a model trained with batch size $256$ and hypeparameters ($\eta=3 \times 10^{-4}$, $\epsilon = 10^{-8}$, and $(\beta_1, \beta_2) = (0.999, 0.999)$) with the performance at a larger batch size, when the hyperparameters are scaled as follows: scale $\eta, 1-\beta_1, 1-\beta_2$ by $\kappa$, if the batch size is scaled by $\kappa$, keeping $\epsilon$ fixed. We use a weight decay factor of $10^{-4}$. We clearly observe a decrease in performance at larger batch size, in comparison to a model trained with square root scaling rule (see \Cref{fig:sqrtscaling_imagenet_largeeps}).}
\label{fig:linearscaling_imagenet_largeeps}
\end{center}
\end{figure}
 
\section{Experiment Configuration Details}\label{sec:app_exp_config}

\subsection{A note on learning rate schedule and warm-up}
We have used a learning rate schedule and a warm-up phase in all our experiments. We have to admit that our current theorems do not directly apply to time-varying learning rates or batch sizes. But our experiments demonstrate that our scaling rules continue to hold for learning rate schedulers with a special warm-up, even if they go beyond the scope of our theoretical setting. Technically, the extensions of our theorems to time-varying learning rates or batch sizes are interesting, and we believe they can indeed be shown following the same proof strategy. The corresponding SDE approximations should have hyperparameters changing with time.

\subsection{CIFAR-10}\label{sec:app_cifar10_config}
There are $50000$ images in the training set and $10000$ images in the validation set of CIFAR-10 \cite{cifar10}.

\paragraph{Architecture.} We used the architecture of ResNet-56 from \cite{He_2015_ICCV} without modification. We used the same architecture of VGG-16 with batch normalization from \cite{simonyan2014very}. However, we kept the final layer of the architecture fixed throughout training, to make the model $1$-homogenous and avoid the optimization difficulties of $2$-homogenous networks~(\cite{li2020exp}). 

\paragraph{RMSprop.} To fix a baseline to compare against, we first trained the models with batches of size $256$, sampled with replacement, at peak learning rate $\eta = 10^{-3}$ and $\beta=0.999$. The model was trained for $500 \times \lfloor (50000 / 256) \rfloor = 97500$ gradient steps (or $500$ epochs). We followed an initial warmup for the first $2\%$ of the total gradient steps. The learning rate schedule during the warmup phase is given by $\eta \times 10^{-3} \times (10^{3})^{\# \text{ epochs } / 10}$. We also followed a learning rate decay schedule, the learning rate was decayed by $0.1$ when the model reaches $80\%$ ($400$ epoch) and $90\%$ ($450$ epoch) of the total continuous time respectively. We did experiments at two values of $\epsilon$, small ($=10^{-30}$) and large  ($=10^{-8}$).

We then made multiple runs of the same model with batches of varying sizes in $\{1024, 4096, 16384\}$, with the hyperparameters $\eta$, $\epsilon$ and $\beta$ modified appropriately according to the scaling rule. The number of gradient steps were modified to keep the total amount of continuous time same across all the batches (which amounted to $500$ epochs by the equivalence between continuous time and the number of training epochs). The warmup schedule and the learning rate decay schedule were kept the same. 

\paragraph{Adam.} We first trained the models with batches of size $256$, sampled with replacement, at peak learning rate $\eta = 10^{-3}$ and $(\beta_1, \beta_2) = (0.999, 0.999)$. The total continuous time of training (or number of epochs), the amount of continuous time in the warmup phase, and the learning rate schedule in the warmup phase were same as RMSprop. The only difference was the learning rate schedule after the warmup phase, the learning rate was decayed by $0.1$ when the model reaches $60\%$ of the total continuous time (or 300 epochs).

We then made multiple runs of the same model with batches of varying sizes in $\{1024, 4096, 16384\}$, with the hyperparameters $\eta$, $\epsilon$ and $\beta_1, \beta_2$ modified appropriately according to the scaling rule. The number of gradient steps were modified to keep the total amount of continuous time same across all the batches. The warmup schedule and the learning rate decay schedule were kept the same. 

\subsection{ImageNet}\label{sec:app_imagenet_config}
There are $1281167$ images in the training set and $50000 $ images in the validation set of ImageNet \citep{deng2009imagenet}.

\paragraph{Architecture.} We trained a ResNet-50 \citep{He_2015_ICCV} model without modification.

\paragraph{Adam.} To fix a baseline to compare against, we first trained the models with batches of size $256$, sampled with replacement, at learning rate $\eta = 3 \times 10^{-4}$ and $(\beta_1, \beta_2)=(0.999,0.999)$. The model was trained for a total of $90 \times \lfloor (1281167 / 256) \rfloor = 450360$ gradient steps (or $90$ epochs). We followed an initial warmup for the first $\frac{1}{18}$ fraction of the total continuous time (or the first $5$ epochs). The learning rate schedule during the warmup phase is increased linearly with epoch, i.e. the learning rate is given by $\eta \times 10^{-3} \times (10^{3})^{\# \text{ epochs } / 5}$. We then followed a learning rate decay schedule, where the learning rate was decayed by $0.1$ when the model reaches $\frac{5}{18}$ fraction (or $50$ epoch ) and $\frac{8}{9}$ fraction (or $80$ epoch) of the total continuous time respectively. We use two different values for $\epsilon$, small $\epsilon$ ($=10^{-30} \approx 0$) and a larger $\epsilon$ ($=10^{-8}$).

\subsection{Books and Wikipedia (Academic BERT)}\label{sec:app_books_wiki_config}
 We use a combination of Bookcorpus \citep{zhu2015aligning} plus English Wikipedia,
which totals $16$ GB of uncompressed text.
We split the data uniformly with a ratio $9:1$, to create training and validation datasets for pretraining.

\paragraph{Architecture.} We pretrain a $24$-layer $\text{RoBERTa}$ \cite{liu2019roberta} model. We pretrain on sequences of length $128$.

\paragraph{Pre-training with Adam.} 
We use the code from \cite{wettig2022should}. We follow the optimization recipe from \cite{izsak2021how} for efficient pre-training.
To fix a baseline, we first trained our model with batch size $1024$, with the optimization parameters given in table~\ref{tab:opt_roberta}. In the warmup phase, the learning rate is increased linearly over the interval, i.e. the learning rate at step $k$ in the warmup phase is given by $\frac{k}{k_{\mathrm{warmup}}}  \eta$, where $k_{\mathrm{warmup}}$ denotes the total number of warmup steps and $\eta$ denotes the peak learning rate. Moreover, after the warmup phase, the learning rate is decayed linearly to $0$, i.e. the learning rate at step $k$ after the warmup phase is given by $\frac{k - k_{\mathrm{warmup}}} { k_{\mathrm{max}} -  k_{\mathrm{warmup}} }  \eta$, where $\eta$ denotes the peak learning rate and $k_{\mathrm{max}}$ denotes the maximum number of gradient steps intended for pretraining.

\sadhika{Move this to the appropriate section}

\begin{table}[]
    \centering
    \begin{tabular}{c|c}
    Hyperparameter & Value \\
    \hline
    Dropout & $0.1$ \\
    Attention Dropout & $0.1$ \\
    Warmup Steps & $5520$\\
    Peak Learning Rate & $10^{-3}$ \\
    Batch Size & $1024$ \\
    Weight Decay & $10^{-4}$ \\
    Max Steps & $92000$ \\
    Learning Rate Decay & Linear \\
    Adam $\beta_1$ & $0.99375$ \\
    Adam $\beta_2$ & $0.996$ \\
    Adam $\epsilon$ & $2\times 10^{-6}$ \\
    Gradient Clipping & $0.0$ \\
    Position embeddings & $128$ \\ \hline\vspace{0.1cm} 
    \end{tabular}
    \caption{Optimization hyperparameters of baseline $\text{RoBERTa}$ model during pretraining.}
    \label{tab:opt_roberta}
\end{table}

\textbf{Fine-tuning with Adam}: We also validate the performance of the pretrained models from the previous section on the GLUE\citep{wang2018glue} datasets. For each downstream task, we run an extensive grid search on the hyperparameters for finetuning each pretrained model. We focused on the following hyperparameters for grid search: batch size, the peak learning rate and the total number of training epochs. Please see \cref{tab:Downstream_hyperparam_academicbert} for the hyperparameter grid. The rest of the hyperparameters are fixed for all the runs and are given in \cref{tab:opt_roberta_finetune}. We follow a similar learning rate schedule during the warmup phase as was used for pretraining: the learning rate at step $k$ in the warmup phase is given by $\frac{k}{k_{\mathrm{warmup}}}  \eta$, where $k_{\mathrm{warmup}}$ denotes the total number of warmup steps and $\eta$ denotes the peak learning rate. Moreover, after the warmup phase, the learning rate is decayed linearly to $0$, i.e. the learning rate at step $k$ after the warmup phase is given by $\frac{k - k_{\mathrm{warmup}}} { k_{\mathrm{max}} -  k_{\mathrm{warmup}} }  \eta$, where $\eta$ denotes the peak learning rate and $k_{\mathrm{max}}$ denotes the maximum number of gradient steps intended for finetuning.

\begin{table}[]
    \centering
    \begin{tabular}{c|c}
    Hyperparameter & Value \\
    \hline
    Dropout & $0.1$ \\
    Attention Dropout & $0.1$ \\
    Warmup Steps & $6\%$ of total\\
    Weight Decay & $0.1$ \\
    Learning Rate Decay & Linear \\
    Adam $\beta_1$ & $0.9$ \\
    Adam $\beta_2$ & $0.98$ \\
    Adam $\epsilon$ & $10^{-6}$ \\
    Gradient Clipping & $0.0$ \\ \hline \vspace{0.1cm}
    \end{tabular}
    \caption{Optimization hyperparameters of all pretrained $\text{RoBERTa}$ models during finetuning.}
    \label{tab:opt_roberta_finetune}
\end{table}

\begin{table}[]
    \centering
\begin{tabular}{c|c|c|c}
		Dataset & Finetune batch size & Peak Learning rate & Total training epochs\\ \hline
		CoLA & $\{16, 32\}$ & $\{10^{-5}, 3\times 10^{-5}, 5\times 10^{-5}, 8\times 10^{-5}\}$ & $\{3, 5, 10\}$ \\
		SST-2 & $\{16, 32\}$ & $\{10^{-5}, 3\times 10^{-5}, 5\times 10^{-5}, 8\times 10^{-5}\}$ & $\{3, 5, 10\}$ \\
		MRPC & $\{16, 32\}$ & $\{10^{-5}, 3\times 10^{-5}, 5\times 10^{-5}, 8\times 10^{-5}\}$ & $\{3, 5, 10\}$ \\
		STS-B & $\{16, 32\}$ & $\{10^{-5}, 3\times 10^{-5}, 5\times 10^{-5}, 8\times 10^{-5}\}$ & $\{3, 5, 10\}$ \\
		RTE & $\{16, 32\}$ & $\{10^{-5}, 3\times 10^{-5}, 5\times 10^{-5}, 8\times 10^{-5}\}$ & $\{3, 5, 10\}$\\
		QQP & $\{32\}$ & $\{5\times 10^{-5}, 8\times 10^{-5}\}$ & $\{3, 5\}$ \\
		MNLI & $\{32\}$ & $\{5\times 10^{-5}, 8\times 10^{-5}\}$ & $\{3, 5\}$ \\
		QNLI & $\{32\}$ & $\{5\times 10^{-5}, 8\times 10^{-5}\}$ & $\{3, 5\}$  \\
		\hline \vspace{0.1cm}
\end{tabular}
\caption{Hyperparameter grid for pretrained  $\text{$\text{RoBERTa}$}_{\text{LARGE}}$ on the downstream tasks.}
\label{tab:Downstream_hyperparam_academicbert}
\end{table}

\subsection{WikiText-103 (GPT)}
WikiText-103 \citep{merity2017pointer} is a dataset with 103 million tokens
extracted from Wikipedia. We split the data uniformly with a ratio $9:1$, to
create training and validation datasets for pretraining. We use Adam
\citep{kingma2014adam} optimization algorithm.

\paragraph{Architecture.} We pretrain a $12$-layer $\text{GPT}$ \citep{brown2020language} model without modification. The model has $12$ layers with hidden dimension $768$, feedforward network dimension $3072$, $12$ attention heads in each attention layer, and attention head size $64$. We pretrain on sequences of length $128$ (unless stated otherwise).

\paragraph{Adam.} We use the code from \cite{wettig2022should}. To fix a baseline, we first trained our model with batch size $1024$, with the optimization parameters given in \Cref{tab:opt_gpt}. In the warmup phase, the learning rate is increased linearly over the interval, i.e. the learning rate at step $k$ in the warmup phase is given by $\frac{k}{k_{\mathrm{warmup}}}  \eta$, where $k_{\mathrm{warmup}}$ denotes the total number of warmup steps and $\eta$ denotes the peak learning rate. Moreover, after the warmup phase, the learning rate is decayed linearly to $0$, i.e. the learning rate at step $k$ after the warmup phase is given by $\frac{k - k_{\mathrm{warmup}}} { k_{\mathrm{max}} -  k_{\mathrm{warmup}} }  \eta$, where $\eta$ denotes the peak learning rate and $k_{\mathrm{max}}$ denotes the maximum number of gradient steps intended for pretraining.

\begin{table}[]
    \centering
    \begin{tabular}{c|c}
    Hyperparameter & Value \\
    \hline
    Dropout & $0.1$ \\
    Attention Dropout & $0.1$ \\
    Warmup Steps & $1000$\\
    Peak Learning Rate & $10^{-3}$ \\
    Batch Size & $1024$ \\
    Weight Decay & $10^{-4}$ \\
    Max Steps & $12500$ \\
    Learning Rate Decay & Linear \\
    Adam $\beta_1$ & $0.9875$ \\
    Adam $\beta_2$ & $0.996$ \\
    Adam $\epsilon$ & $2\times 10^{-8}$ \\
    Gradient Clipping & $0.0$ \\
    Position embeddings & $128$ \\
    \hline \vspace{0.1cm}
    \end{tabular}
    \caption{Optimization hyperparameters of baseline $\text{GPT}$ model during pretraining.}
    \label{tab:opt_gpt}
\end{table}


\end{document}